\documentclass[10pt,journal,compsoc]{IEEEtran}

\usepackage[utf8]{inputenc} % allow utf-8 input
\usepackage[T1]{fontenc}    % use 8-bit T1 fonts
\usepackage{hyperref}       % hyperlinks
\usepackage{url}            % simple URL typesetting
\usepackage{booktabs}       % professional-quality tables
\usepackage{amsfonts}       % blackboard math symbols
\usepackage{nicefrac}       % compact symbols for 1/2, etc.
\usepackage{microtype}      % microtypography
\usepackage{microtype}
\usepackage{graphicx}
\usepackage{subfig}
\usepackage{booktabs} % for professional tables
\usepackage{multirow}
\usepackage{times}
\usepackage{epsfig}
\usepackage{graphicx}
\usepackage{amsmath}
\usepackage{amssymb}
\usepackage{wrapfig}
\usepackage{bm}
 
\usepackage[ruled,linesnumbered]{algorithm2e}
 \usepackage{bbm}
\usepackage{float}
\usepackage[utf8]{inputenc}
\usepackage[english]{babel}
 
\usepackage{amsthm}

\usepackage[utf8]{inputenc}
\usepackage[english]{babel}
 \usepackage{MnSymbol}
\newtheorem{theorem}{Theorem}

\newtheorem{proposition}{Proposition}
\newtheorem{assumption}{Assumption}

\newtheorem{definition}{Definition}

\usepackage{hyperref}
\usepackage{color}

\usepackage{ragged2e}

\begin{document}
\title{Bayesian Active Learning by Disagreements: \\ A Geometric Perspective}
% A/On Geometric Interpretation of Deep Bayesian  Active Learning

%  GBALD: A Two-stage Deep Bayesian Active learning Framework
% Core-set Construction as an Bayesian Initialization:  Estimating Model Uncertainty for Deep Bayesian Active learning

% The \author macro works with any number of authors. There are two commands
% used to separate the names and addresses of multiple authors: \And and \AND.
%
% Using \And between authors leaves it to LaTeX to determine where to break the
% lines. Using \AND forces a line break at that point. So, if LaTeX puts 3 of 4
% authors names on the first line, and the last on the second line, try using
% \AND instead of \And before the third author name.

\author{Xiaofeng Cao and 
       Ivor W.  Tsang% <-this % stops a space
\IEEEcompsocitemizethanks{\IEEEcompsocthanksitem \emph{X. Cao and I. W. Tsang are with the  
Australian Artificial Intelligence Institute,  University of Technology Sydney, NSW 2008, Australia. 
E-mail: ivor.tsang@uts.edu.au. This work was supported  
 by Australian Research Council  under Grant   DP180100106 and DP200101328.}}\protect
\thanks{Manuscript received xx xx, xxxx; revised xx xx, xxxx. }}


\markboth{Journal of \LaTeX\ Class Files,~Vol.~14, No.~8, August~2015}%
{Shell \MakeLowercase{\textit{et al.}}: Bare Demo of IEEEtran.cls for Computer Society Journals}

\IEEEtitleabstractindextext{%
\begin{abstract}\justifying
We present  geometric Bayesian active learning by disagreements (GBALD), a framework that performs   BALD on its core-set construction interacting with model uncertainty estimation.
Technically, GBALD   constructs  core-set   on ellipsoid, not  typical sphere, preventing low-representative elements from spherical boundaries. The improvements  are twofold:
1) relieve   uninformative  prior  and 2) reduce  redundant  estimations. 
Theoretically,    geodesic search with ellipsoid  can   derive   tighter  lower  bound  on error and  easier to achieve zero error than with sphere. 
Experiments show that GBALD has slight perturbations to noisy and repeated samples, and  outperforms  BALD, BatchBALD and  other existing deep active learning  approaches. 
 \end{abstract}
\begin{IEEEkeywords}
Geometric Bayesian, deep active learning,  core-set, model uncertainty, ellipsoid.
\end{IEEEkeywords}}
 
  \maketitle

\section{Introduction}
%Lack of training labels  restricts the performance of  deep  neural networks (DNNs), though prices of GPU resources were falling fast.  

Deep  neural networks (DNNs)   lack the ability of learning from limited (insufficient) labels, which degenerates  its    generalizations to    new tasks. 
Recently, leveraging the abundance of unlabeled data has become a potential solution to relieve this bottleneck whereby the expert knowledge is involved to perform annotations.
In such setting, the deep learning researchers introduced the \textbf{active learning (AL)} \cite{gal2017deep}, which solicit experts' annotations from  those informative or representative unlabeled data, by maximizing the model uncertainty  \cite{ashukha2019pitfalls,lakshminarayanan2017simple} of the current learning model. During this AL process, the learning model tries to achieve a desired accuracy performance using the  minimal data labeling.  Recent shift of    model uncertainty in many  fields  shows that   deep Bayesian AL \cite{pinsler2019bayesian,kirsch2019batchbald} 
  contributes the   Bayesian neural networks training \cite{blundell2015weight}, 
 Monte-Carlo (MC) dropout  \cite{gal2016dropout}, and Bayesian core-set construction \cite{DBLP:conf/iclr/SenerS18}, etc.

\textbf{Bayesian AL}   \cite{golovin2010near,jedoui2019deep} presents  an expressive probabilistic  interpretation on the model uncertainty estimation \cite{gal2016dropout}.  Theoretically, for a simple regression model  such as linear, logistic, and probit, AL can derive their closed-forms on updating one sparse subset, which   maximally
reduces  the uncertainty  of   posteriors over   regression  parameters \cite{pinsler2019bayesian}. However, for a DNN model, optimizing massive  training parameters is not easily tractable. It is thus that the Bayesian approximation provides  alternatives, including  the importance sampling \cite{doucet2000sequential} and the Frank-Wolfe optimization \cite{vavasis1992approximation}.  With  importance sampling, a typical Bayesian AL approach   can be  expressed as to maximizing  the information gain in terms of  predictive entropy over the model, and it is called  Bayesian active learning by disagreements (BALD) \cite{houlsby2011bayesian}.

 \textbf{BALD} has two  interpretations:  model uncertainty estimation and  core-set construction. To estimate the uncertainty of a model, a greedy strategy is usually applied. The criterion is to select those data that maximize  the parameter disagreements between the current training model and its following updates as  \cite{gal2017deep}.   However, naively interacting with BALD using an uninformative prior \cite{strachan2003bayesian,price2002uninformative} leads to unstable biased acquisitions \cite{gao2019consistency}, for example,  insufficient or unbalanced prior labels. Under this setting,   an uninformative prior can be constructed to reflect a balanced state among   different Bayesian outcomes, if there is  no  available information. Moreover, the similarity or consistency of those acquisitions to their previous,  brings some redundant information  to the model, and  may decelerate the subsequent training. 

  \begin{figure} 
\centering
\includegraphics[scale=0.25]{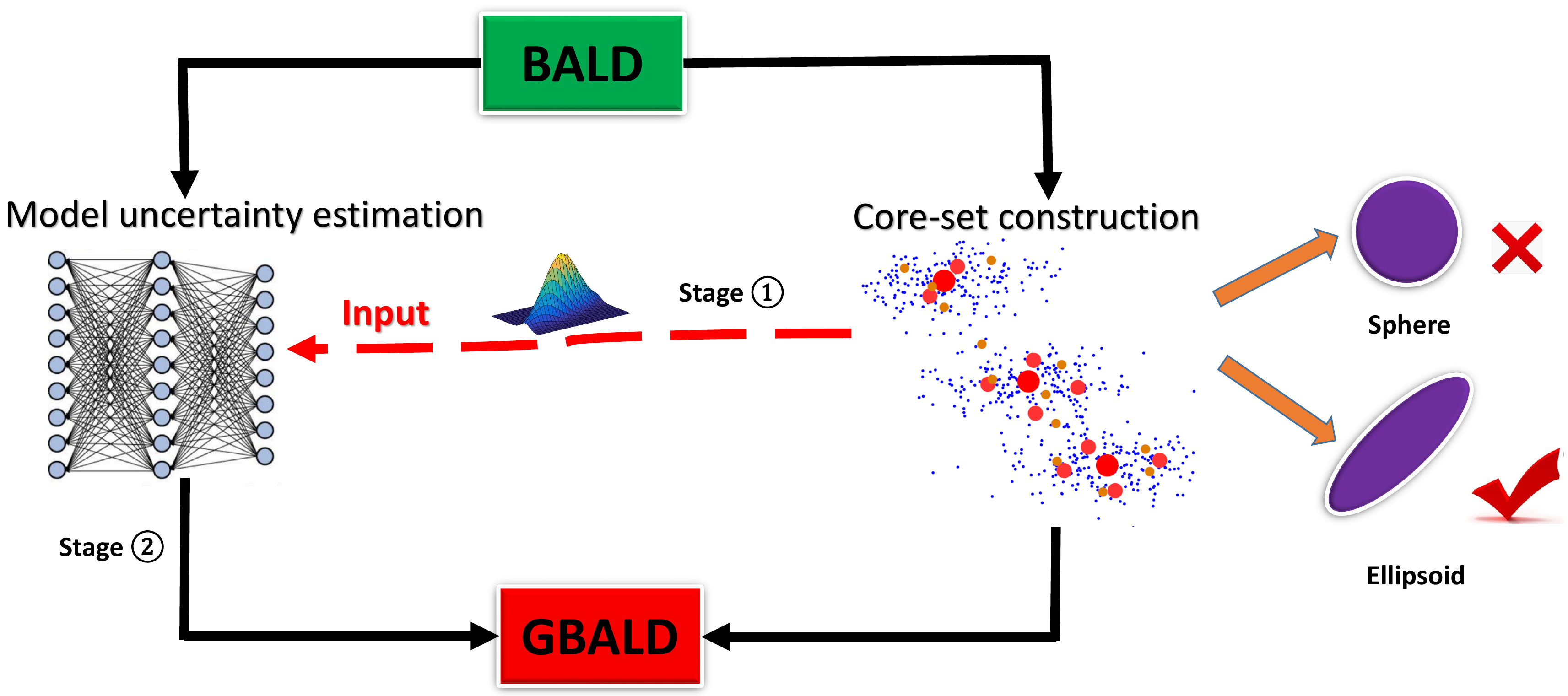}
\caption{{Illustration of  two-stage GBALD framework, which integrates model uncertainty estimation and  core-set construction  into a uniform framework.  Stage $\textcircled{1}$:  core-set construction is with  ellipsoid, not typical sphere, representing the original distribution to initialize  the input features of DNNs. Stage $\textcircled{2}$: model uncertainty estimation   with those initial acquisitions  then  explores   highly  informative and  representative samples for DNN.}
 } 
%\vspace{-5pt}
\label{Illustration}
\end{figure}

\textbf{Core-set construction}  \cite{campbell2018bayesian} avoids the greedy interaction to  learning model via  capturing   the characteristics of   data distributions. 
By approximating   complete data posterior over   model parameters, optimization of BALD can be deemed as a  core-set construction process on a sphere \cite{kirsch2019batchbald}, which seamlessly solicits a compact subset to draw  the input  data distribution, and efficiently  mitigates the sensitivity to   uninformative prior and   redundant information.

%. It solicits a subset of training data to update the data representation, which % Then, directly inducing a   representative approximation as a collection  of the training set 
%can efficiently  mitigate the sensitivity to  uninformative prior and redundant information.

From a  geometric perspective,    updates of  core-set construction are usually optimized with   spherical geodesic    as \cite{nie2013early,wang2019incorporating}. Once the core-set is obtained,   deep AL algorithm immediately seeks  annotations from  knowledgeable experts  and   starts the training. However,  
the   data located at    the  boundary  regions of  distributions, usually with a uniform  manner,
 could not be  highly-representative   elements of core-set. Therefore,  constructing   core-set on   sphere may not be the optimal choice for deep AL.

This paper presents a novel AL framework, namely \textbf{Geometric BALD (GBALD)}, over the geometric interpretation of BALD that,  interpreting  BALD with the core-set construction on  ellipsoid,   initializes   effective   representations  to estimate the model uncertainty.
The goal is to seek for 
 significant accuracy improvements  against an uninformative  prior and the redundant information.   Figure~\ref{Illustration} describes   this two-stage framework. In the first stage, 
  geometric core-set construction on  ellipsoid \cite{perrone2019learning}   initializes a set of effective acquisitions to start  a DNN model regardless of an uninformative prior. Taking  the core-set as the inputs,  the next stage ranks the batch acquisitions of model uncertainty according to their geometric representativeness,  and then solicits some highly-representative examples from the batch. With    representation constraints,
  the ranked acquisitions  reduce the probability of sampling those nearby samples of the previous  acquisitions,  preventing  redundant information.  
In order to explore   these improvements, following the typical approximately linear perceptron analysis \cite{sugiyama2006active},  
 our generalization analysis   shows that,  the lower bound of generalization errors  of geodesic search  with  ellipsoid,  is proven to be tighter than that of  geodesic search   with   sphere.  Achieving a nearly zero   error  by geodesic search  with  ellipsoid, is also  proven to have a higher probability than that of sphere. 

\par Contributions of this paper can be summarized from the geometric, algorithmic, and theoretical perspectives. 
\begin{itemize}
\item Geometrically, our key technical innovation is to construct  core-set on ellipsoid, not typical sphere, preventing low-representative elements from boundary distribution.

\item In term of the algorithm design, our work proposes a two-stage framework from a Bayesian perspective that sequentially introduces the core-set representation and  model uncertainty estimation, strengthening their performance “independently”. Moreover, different to the typical BALD optimizations, we present geometric solvers to construct a core-set and estimate model uncertainty using it, which result in a different perspective for Bayesian AL. 

\item Theoretically, to guarantee those improvements, our generalization analysis proves that, compared to the typical Bayesian spherical interpretation, the geodesic search with ellipsoid can derive a tighter lower error bound and achieve a higher probability to obtain a nearly zero error.  

\end{itemize}

The rest of this paper is organized as follows. In Section~2,
we first review the related work. Secondly, we elaborate  BALD   and   GBALD   in Sections~3 and 4, respectively.   
 Experimental results are presented in Section~5. Finally, we conclude
this paper in Section~6.

\section{Related work}
 \textbf{AL.} The early probability  support vector machine (SVM) proposed the concept of  AL \cite{cohn1994improving,schohn2000less} that    acquires the  data with   minimum  margin to effectively update the support vectors. Many AL acquisition algorithms   were then proposed to  relieve
 the training bottleneck of SVM, which results in  unsatisfied predictions, for example,  uncertainty sampling \cite{lewis1994sequential},  margin sampling \cite{scheffer2001active}, MC estimation of error reduction \cite{roy2001toward}, transductive experimental design \cite{yu2006active}, etc.  Given a learning   model without sufficient training labels,  those effective acquisitions reduce  the expensive  cost of   human annotations in many scenarios, e.g.   multiple correct outputs \cite{jedoui2019deep},  cost-sensitive classification \cite{krishnamurthy2019active}, adversarial training \cite{sinha2019variational}, etc. In theory,  the researchers studied the 
label complexity bound \cite{hanneke2007bound} (label demand before achieving a desired error threshold) and its noisy performance \cite{javdani2014near} of an AL   algorithm.
 
 \textbf{Model uncertainty.} In deep learning setting, AL was introduced to improve the training of   DNNs   by annotating a batch  of unlabeled data, where the data which  maximize the model uncertainty  \cite{lakshminarayanan2017simple} are the primary acquisitions. For example,  in   ensemble deep learning \cite{ashukha2019pitfalls}, 
out-of-domain uncertainty estimation selects those data, which do not follow the same distribution as the input training data;   in-domain uncertainty draws the data from the original input distribution, producing reliable probability estimates.  Gal \emph{et al.} \cite{gal2016dropout}   use MC dropout to estimate the predictive uncertainty via approximating a Bayesian convolutional neural network.   Lakshminarayanan \emph{et  al.}
\cite{lakshminarayanan2017simple} evaluate  the   uncertainty of   unlabeled data using a proper scoring rule, deriving  the sampling criteria of AL to fed the  DNNs.  

  \textbf{Bayesian AL.} Taking a Bayesian perspective \cite{golovin2010near}, AL can be deemed as  minimizing the Bayesian posterior risk with multiple label acquisitions over the input unlabeled data.  A potential informative approach is  to reduce the uncertainty about  model parameters using Shannon’s entropy \cite{tang2002active}.
 This can be interpreted as seeking the acquisitions for which the Bayesian parameters under the
posterior disagree about the outcome the most, so  this  acquisition algorithm is referred to as 
Bayesian active learning by disagreement (BALD)  \cite{houlsby2011bayesian}.    In applications, BALD was introduced into 
  natural language processing 
 \cite{siddhant2018deep},   text classification \cite{burkhardt2018semisupervised},   decision making \cite{javdani2014near}, data augmentation \cite{tran2019bayesian}, etc.

  \textbf{Deep AL.} Recently,    Gal \emph{et al.} \cite{gal2017deep}   proposed to cooperate  BALD with   DNNs to improve its   acquisition performance.   The unlabeled data which maximize  the model uncertainty of DNNs provide positive feedback. However, it needs to repeatedly update the model until the acquisition budget is exhausted. To improve the acquisition efficiency,  batch sampling with BALD  is applied as    \cite{kirsch2019batchbald, pinsler2019bayesian}. In BatchBALD,  Kirsch \emph{et al.} \cite{kirsch2019batchbald}   developed a tractable approximation to the mutual information of one   batch of unlabeled data and the  current  model parameters. However, those uncertainty evaluations of Bayesian AL whether in single or batch acquisitions all take a  greedy strategy, which leads to 
 computationally infeasible,  or excursive  parameter estimations.  Pinsler \emph{et al.}   \cite{pinsler2019bayesian}  thus approximated the   posterior  over the model parameters by a sparse subset, i.e. a core-set. Applying the Frank-Wolfe optimization \cite{vavasis1992approximation}, the batch acquisitions of a large-scale dataset can be efficiently derived, thereby interpreting closed-form solutions for the  core-set construction on  a linear or probit regression function.  As a consequence, the non-deep models obtained the theoretical guarantees from this optimization solver due to their tractable parameters.
However, for deep  AL, 
being short of interactions to DNNs is not able to maximally drive their model performance.  
 
 \iffalse
%{In BatchBALD,  Kirsch \emph{et al.}  \cite{kirsch2019batchbald}   developed a tractable approximation to the mutual information of one   batch of unlabeled data and  current  model parameters. However, those uncertainty evaluations of Bayesian AL whether in single or batch acquisitions all take   greedy strategies, which lead to 
% computationally infeasible,  or excursive  parameter estimations. Pinsler \emph{et al.}   \cite{pinsler2019bayesian}  thus approximated the   posterior  over the model parameters by a sparse subset, i.e. core-set construction. Applying Frank-Wolfe optimization \cite{vavasis1992approximation}, batch acquisitions of large-scale dataset can be efficiently derived, thereby interpreting closed-form solutions for core-set construction on   linear and probit regression functions.  As a consequence, non-deep models obtained theoretical guarantees from this optimization solver due to tractable model parameters. For deep Bayesian AL, 
 %lacking of   interaction to DNNs may not maximally drive their model performance.} %In applications, BALD was introduced into 
 % natural language processing 
%\cite{siddhant2018deep},   text classification \cite{burkhardt2018semisupervised},   decision making \cite{javdani2014near}, data augmentation \cite{tran2019bayesian}, etc.
\fi 

\section{BALD}

BALD has two different interpretations:   model uncertainty estimation and   core-set construction. We simply introduce them in this section.

\subsection{Model uncertainty estimation}

We consider a discriminative model $p(y|x,\theta)$  parameterized by   $\theta$ that maps $x\in\mathcal{X}$  into an output distribution over a  set  of   $y\in \mathcal{Y}$. Given an  initial labeled (training) set $\mathcal{D}_0 \in \mathcal{X}\times \mathcal{Y}$, the Bayesian inference over this parameterized model is to estimate the posterior  $p(\theta|\mathcal{D}_0)$, i.e. estimate $\theta$ by repeatedly updating $\mathcal{D}_0$. AL  adopts this setting from a Bayesian perspective. 

\par With AL,  the learner can choose  a set of   unlabeled data from $\mathcal{D}_u=\{x_i\}_{j=1}^{N} \in \mathcal{X}$ via maximizing  the uncertainty  of the model parameters.   Houlsby \emph{et al.} \cite{houlsby2011bayesian} proposed  a greedy strategy termed BALD to  update  $\mathcal{D}_0$ by estimating  a desired data $x^*$ that maximizes the decrease in expected posterior entropy:
\begin{equation}
\begin{split}
x^*=\operatorname*{arg\ max}_{x\in  \mathcal{D}_u }  {\rm\bm{H}} [\theta|\mathcal{D}_0]-\mathbb{E}_{y\sim p(y|x,\mathcal{D}_0)}   \Big[{\rm\bm{H}}[\theta|x,y,\mathcal{D}_0]\Big],
\end{split}
\end{equation}
where the labeled and unlabeled sets are updated by $\mathcal{D}_0 =\mathcal{D}_0 \cup \{x^*, y^* \},  \mathcal{D}_u =  \mathcal{D}_u  \backslash  x^*$, and $y^*$ denotes the output    of $x^*$. In deep AL, $y^*$ can be annotated as a label from experts and $\theta$  yields a DNN model.

\subsection{Core-set construction}
 Let $p(\theta|\mathcal{D}_0)$ be updated by its log posterior  ${\rm log} p(\theta|\mathcal{D}_0, x^*)$,    $y^*\in \{y_i\}_{i=1}^N$, { assume the outputs  are conditional independent of the inputs, i.e.  $p(y^*|x^*,D_0)=\int_\theta p(y^*|x^*, \theta)  p(\theta|D_0)  {\rm d} \theta $},  then we have the \emph{complete data log posterior following} \cite{pinsler2019bayesian}:

 \begin{equation}
\begin{split}
&\mathbb{E}_{y^*}[{\rm log} p( \theta|\mathcal{D}_0, x^*, y^*) ]\\
&=\mathbb{E}_{y^*}[ {\rm log} p(\theta|\mathcal{D}_0) + {\rm log}  p(y^*|x^*, \theta)- {\rm log}p(y^*|x^*, \mathcal{D}_0) ] \\
              &=        {\rm log} p(\theta|\mathcal{D}_0) + \mathbb{E}_{y^*}\Big[{\rm log}  p(y^*|x^*, \theta)+{\rm\bm{H}}[y^*|x^*,\mathcal{D}_0]\Big]\\
               &= {\rm log} p(\theta|\mathcal{D}_0) +\sum_{i=1}^{N} \Bigg(\mathbb{E}_{y_i} \Big[{\rm log}  p(y_i|x_i, \theta)+{\rm\bm{H}}[y_i|x_i,\mathcal{D}_0] \Big]   \Bigg). 
\end{split}
\end{equation}

{The key idea of  the core-set construction is to approximate the log posterior of Eq.~(2) by a subset of $D_u' \subseteq  D_u$ such that:
$\mathbb{E}_{\mathcal{Y}_u}[{\rm log} p( \theta|\mathcal{D}_0, \mathcal{D}_u, \mathcal{Y}_u) ]\approx \mathbb{E}_{\mathcal{Y}'_u}[{\rm log} p( \theta|\mathcal{D}_0, \mathcal{D}_u', \mathcal{Y}_u') ], $
where $\mathcal{Y}_u$ and $\mathcal{Y}'_u$  denote  the predictive labels of  $\mathcal{D}_u$ and $\mathcal{D}'_u$ respectively by the Bayesian discriminative model,  that is, $p(\mathcal{Y}_u|\mathcal{D}_u,D_0)=\int_\theta p(\mathcal{Y}_u|\mathcal{D}_u, \theta)  p(\theta|D_0)  {\rm d}  \theta $,  and $p(\mathcal{Y}_u'|\mathcal{D}_u',D_0)=\int_\theta p(\mathcal{Y}_u'|\mathcal{D}_u', \theta)  p(\theta|D_0)  {\rm d}  \theta $.
Here $D_u'$ can be indicated by a core-set  \cite{pinsler2019bayesian} that highly represents $\mathcal{D}_u$. The optimization tricks such as the Frank-Wolfe optimization \cite{vavasis1992approximation} then can be adopted to solve this problem.}

\textbf{Motivations.} Eqs.~(1) and (2) provide the Bayesian rules of BALD over   model uncertainty and   core-set construction  respectively,    which further attract the attention of   deep learning researchers.  However, the two interpretations of  BALD are  limited by: 1) the redundant information  and 2) an uninformative   prior, where one major reason which causes these two issues is the poor initialization on the prior, i.e. $p(\mathcal{D}_0|\theta)$. For example, an unbalanced label initialization on $\mathcal{D}_0$ usually leads to an uninformative prior, which further conducts the acquisitions  of AL to select those unlabeled data from one or some fixed classes;   highly-biased results \cite{gao2019consistency} with  redundant information  are inevitable.  Therefore, these two limitations affect each other. 

  \section{GBALD }

\begin{figure}
\subfloat[Sphere geodesic]{
\label{fig:improved_subfig_b}
\begin{minipage}[t]{0.25\textwidth}
\centering
\includegraphics[width=1.2in,height=1.12in]{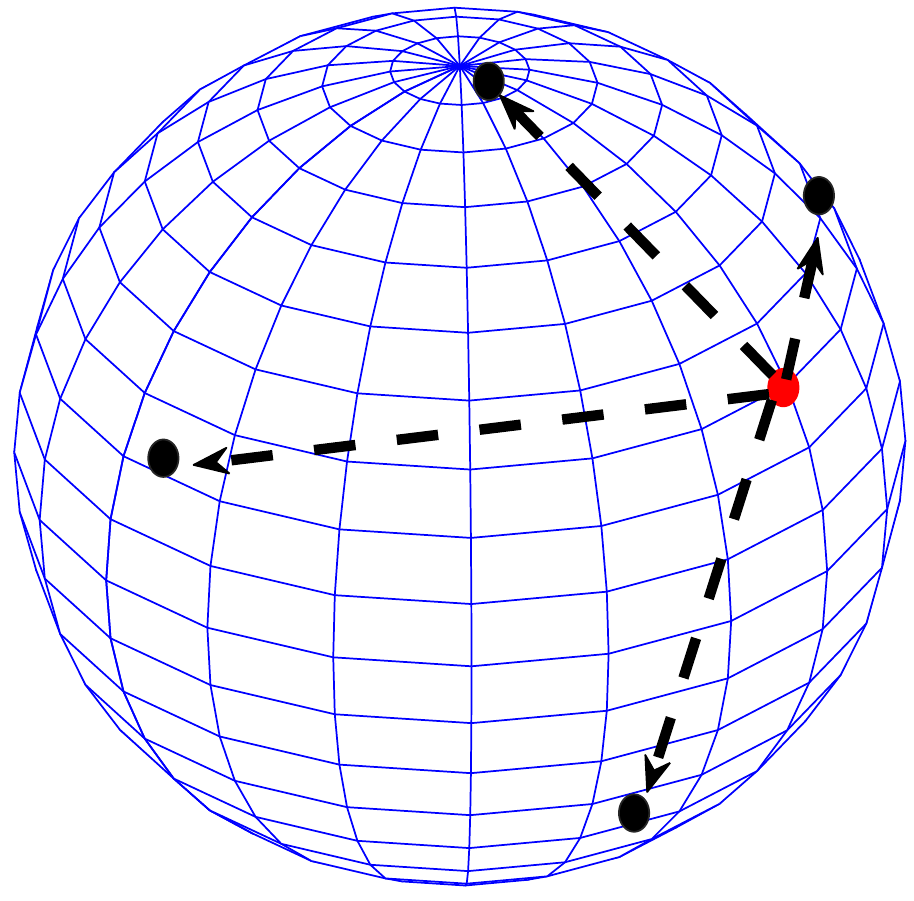}
\end{minipage}
}
\subfloat[Ellipsoid geodesic \ \ \ \ \ \ ]{
\label{fig:improved_subfig_b}
\begin{minipage}[t]{0.25\textwidth}
\centering
\includegraphics[width=0.8in,height=1.08in]{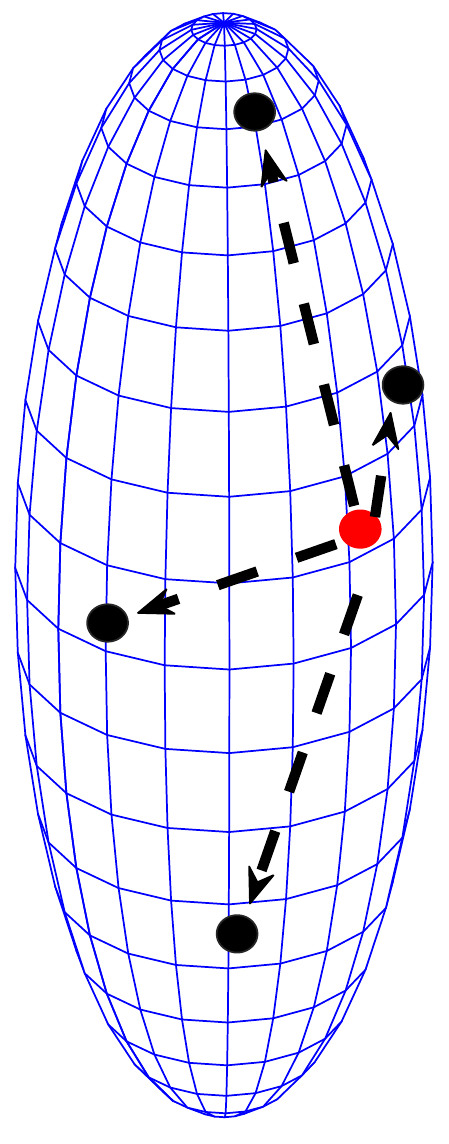}
\end{minipage}
}
\caption{Optimizing BALD with sphere and  ellipsoid geodesics. Ellipsoid geodesic rescales the sphere geodesic to prevent  the updates of  core-set towards  the  boundary regions of the sphere where 
  characteristics of spherical    distribution cannot be properly  captured. Note that the black points denote the feasible  updates of the red points and the dash lines denote the geodesics.}   
\end{figure}

 GBALD consists of two components: 1) initial acquisitions based on   core-set construction and   2)   model uncertainty estimation with those initial acquisitions.

\subsection{Geometric  interpretation of core-set}

Modeling the complete data posterior  over   parameter distribution
can relieve the  two limitations of BALD, which has been stated at the end of Section~3.2. Typically, optimizing the acquisitions of  Bayesian  AL  is  equivalent   to   approximating   a core-set  centered with the  spherical embeddings  \cite{DBLP:conf/iclr/SenerS18}. Let $w_i$  be the sampling weight of $x_i$, $\|w_i\|_0\leq N$, the core-set  construction is to optimize:
 \begin{equation}
\begin{split}
\operatorname*{min}_{ w}  &   \ \ \Bigg \|\underbrace{\sum_{i=1}^N   \mathbb{E}_{y_i}\Big[{\rm log}  p(y_i|x_i, \theta)+{\rm\bm{H}}[y_i|x_i,\mathcal{D}_0]\Big]}_{\mathcal{L}}\\
 &  -\underbrace{\sum_{i=1}^N w_i \mathbb{E}_{y_i}\Big[{\rm log}  p(y_i|x_i, \theta)+{\rm\bm{H}}[y_i|x_i,\mathcal{D}_0]\Big]}_{\mathcal{L}(w)}  \Bigg \|^2,\\
  \end{split}
\end{equation}
where $\mathcal{L}$ and $\mathcal{L}(w)$ denote 
  the full   and   expected (weighted) log-likelihoods, respectively \cite{campbell2018bayesian, campbell2019automated}. Specifically,  $\sum_{i=1}^N{\rm\bm{H}}[y_i|x_i,\mathcal{D}_0]=-\sum_{y_i} p(y_i|x_i,\mathcal{D}_0){\rm log}(p(y_i|x_i,\mathcal{D}_0) $, where  $p(y_i|x_i,\mathcal{D}_0)=\int_\theta p(y_i|x_i,\theta) p(\theta|\mathcal{D}_0) {\rm d} \theta$.  Note $\|\cdot \|$ denotes the $\ell^2$ norm.

The approximation of Eq.~(3) implicitly requires  that the  complete data log posterior of Eq.~(2) w.r.t. $\mathcal{L}$  must be close to an expected   posterior w.r.t. $\mathcal{L}(w)$ such that approximating a sparse subset for the original inputs by sphere geodesic search is feasible (see Figure~2(a)). Generally, solving this optimization is intractable due to the cardinality constraint   \cite{pinsler2019bayesian}.  Campbell \emph{et al.} \cite{campbell2019automated} proposed to relax the constraint in    Frank–Wolfe  optimization, in which  mapping  $\mathcal{X}$ is usually  performed in a Hilbert space (HS) with a bounded inner product operation.  In this solution, the sphere embedded in  the HS replaces the  cardinality constraint with a polynomial constraint. However, the initialization on $\mathcal{D}_0$ affects the iterative approximation to $\mathcal{D}_u$  at the beginning of the  geodesic search. Moreover, the posterior of  $p(\theta|\mathcal{D}_0) $ is  uninformative, if the initialized $\mathcal{D}_0$ is   empty or not correct. Therefore, the typical Bayesian core-set construction of BALD cannot ideally fit an uninformative 
prior. The another geometric interpretation of core-set construction, such as $k$-centers \cite{DBLP:conf/iclr/SenerS18}, is not restricted  to this setting. We thus follow the construction of $k$-centers to find the core-set. 

\textbf{$\bm{k}$-centers.}   Sener  \emph{et al.}  \cite{DBLP:conf/iclr/SenerS18} proposed a core-set representation approach for   deep AL based on $k$-centers. This approach can be adopted in the core-set construction of BALD without the help of a discriminative (training) model. Therefore, the uninformative prior has no further influence to the core-set. Typically, the $k$-centers approach uses  a  greedy strategy to search the data  $\widetilde x$ whose nearest distance to the elements of $\mathcal{D}_0$ is the maximal:
\begin{equation}
 \widetilde x=\operatorname*{ arg \ max}_{ x_i\in  \mathcal{D}_u }    \min_{c_i\in  \mathcal{D}_0 } \|x_i-c_i\|,
\end{equation}
then  $\mathcal{D}_0$  is updated by   $\mathcal{D}_0 \cup \{ \widetilde x, \widetilde y\}$, $\mathcal{D}_u$  is updated by  $\mathcal{D}_u  \backslash  \widetilde x $, where $\widetilde y$ denotes the output    of $\widetilde x$. This max-min operation usually performs $k$ times to construct the centers.

From a  geometric perspective, the ${k}$-centers   can be deemed as   the core-set construction via the spherical geodesic search as \cite{badoiu2002approximate,har2004coresets}. Specifically, the max-min optimization guides $\mathcal{D}_0$ to be updated into one  data  which  draws the longest geodesic from $x_i, \forall i$ across the sphere center. The iterative update on $\widetilde x$ is then along its unique diameter through the sphere center. 
However, this greedy optimization has a large probability that leads the core-set to fall into the boundary regions of the sphere, which is not able to capture the characteristics of the distributions. 
 
\subsection{Initial acquisitions based on core-set construction}
 We present a novel greedy search which rescales the geodesic of a sphere into an ellipsoid following Eq.~(4), in which the iterative update on the geodesic  search is rescaled  (see Figure~2(b)).  We follow the importance sampling strategy to begin the search. 

\textbf{Initial prior on geometry.} Initializing  $p(\mathcal{D}_0|\theta)$ is performed with 
 a group of internal spheres centered with $D_j, \forall j,$ subjected to $D_j \in \mathcal{D}_0$, in which  the geodesic  between $\mathcal{D}_0$ and the unlabeled data 
is over those spheres. Since  $\mathcal{D}_0$ is known, the specification of  
 $\theta$ then plays a key role on  initializing $p(\mathcal{D}_0|\theta) $. Given a radius $R_0$ for any observed internal sphere, $p(y_i|x_i,\theta)$ is firstly   defined by
\begin{equation}
p(y_i|x_i,\theta) =\left\{
\begin{aligned}
&\  \ \ \  \ \   \  \   \  1, & \exists j, \|x_i-D_j \|\leq  R_0,\\
 &{\rm max}\Bigg \{\frac{R_0}{\|x_i-D_j \|} \Bigg\}, & \forall j,  \|x_i-D_i \|>  R_0, \\
\end{aligned}
\right.
\end{equation}
thereby $\theta$ yields the parameter $R_0$.   When the data is enclosed with a   ball, the  probability of Eq.~(5) is 1. The data near the ball, is given a  probability of ${\rm max}\Big \{\frac{R_0}{\|x_i-D_j\|} \Big\}$ constrained by ${\rm min}  \|x_i-D_j \|, \forall j$, i.e. the probability  is assigned by the nearest ball to $x_i$, which is centered with  $D_j$. From Eq.~(3), the information entropy of $y_i\sim \{y_1,y_2,...,y_N\}$ over     $x_i\sim \{x_1,x_2,...,x_N\}$   can be expressed as the integral regarding $p(y_i|x_i,\theta)$:
\begin{equation}
\begin{split}&  
\sum_{i=1}^N{\rm H}(y_i|x_i,\mathcal{D}_0)\\
&=-\sum_{i=1}^N  \int_\theta p(y_i|x_i,\theta)p(\theta|D_0) d \theta{\rm log} \Big(\int_\theta p(y_i|x_i,\theta)p(\theta|D_0)\Big) d \theta,
\end{split}
\end{equation}
 which can be approximated by $ -\sum_{i=1}^N    p(y_i|x_i,\theta)   {\rm log} \Big(p(y_i|x_i,\theta)  \Big)$ following the details of Eq.~(3). In short, this  indicates an approximation to the entropy  over the entire outputs on $\mathcal{D}_u$ that  assumes the prior  $p(D_0|\theta)$ w.r.t.   $p(y_i|x_i,\theta)$  is already known from  Eq.~(5).

\textbf{Max-min optimization.}  Recalling the max-min optimization trick of $k$-centers in the core-set construction of \cite{DBLP:conf/iclr/SenerS18}, the minimizer of Eq.~(3) then can be divided into two parts:  $\operatorname*{ min}_{x^*} \mathcal{L} $   and $\operatorname*{ max}_{w}  \mathcal{L}(w) $, where $\mathcal{D}_0$ is updated by acquiring $x^*$. However, the updates of $\mathcal{D}_0$ decide     the minimizer of $\mathcal{L}$  with regard to  the internal  spheres 
centered with  ${D}_i, \forall i$. Therefore, minimizing $\mathcal{L}$ should be constrained  by  an unbiased full likelihood over $\mathcal{X}$ to alleviate the potential biases from the initialization of $\mathcal{D}_0$. Let  $\mathcal{L}_0$ denote the unbiased full likelihood over $\mathcal{X}$ that particularly stipulates $\mathcal{D}_0$ as the $k$-means centers written as $\mathcal{U}$ of $\mathcal{X}$ which jointly draw the input distribution. We define $\mathcal{L}_0= |{\sum_{i=1}^N  \mathbb{E}_{y_i}[{\rm log}  p(y_i|x_i, \theta)+{\rm\bm{H}}[y_i|x_i,\mathcal{U}]]}|$ to regulate $\mathcal{L}$, that is  
 \begin{equation}
\begin{split}
&\operatorname*{ min}_{x^*} \|\mathcal{L}_0- \mathcal{L}\|^2,\ \ {\rm s.t.}\  \mathcal{D}_0 =\mathcal{D}_0    \cup \{x^*, y^* \},  \mathcal{D}_u =  \mathcal{D}_u  \backslash  x^*. \\
   \end{split}
 \end{equation}
 The other sub optimizer is  $\operatorname*{ max}_{w}  \mathcal{L}(w) $. We present a greedy strategy following Eq.~(1):
\begin{equation}
\begin{split}
&\operatorname*{  max}_{1\leq i \leq N}\  \operatorname*{ min }_{w_i}  \   \sum_{i=1}^N w_i \mathbb{E}_{y_i}[{\rm log}  p(y_i|x_i, \theta)+{\rm\bm{H}}[y_i|x_i,\mathcal{D}_0]]\\
&= \sum_{i=1}^N  w_i{\rm log}  p(y_i|x_i, \theta)-\sum_{i=1}^N w_i p(y_i|x_i, \theta)  {\rm log}  p(y_i|x_i, \theta), 
   \end{split}
 \end{equation}
which  can be further written as: $ \sum_{i=1}^N w_i{\rm log}  p(y_i|x_i, \theta)(1-{\rm log}  p(y_i|x_i, \theta)).$
{ Let $w_i=1, \forall i$ for unbiased estimation of the likelihood  $\mathcal{L}(w)$, } Eq.~(8)  can be simplified as
 \begin{equation}
\begin{split}
\operatorname*{  max}_{x_i\in  \mathcal{D}_u}\ \operatorname*{ min }_{D_j\in \mathcal{D}_0}     {\rm log} p(y_i|x_i, \theta),
   \end{split}
 \end{equation}
 where $p(y_i|x_i, \theta)$ follows Eq.~(5).
Combining Eqs.~(7) and (9), the optimization of Eq.~(3) is then transformed  as
 \begin{equation}
\begin{split}
x^*=\operatorname*{ arg \ max}_{x_j\in  \mathcal{D}_u}\ \operatorname*{ min}_{D_j\in \mathcal{D}_0}  \Bigg \{\|\mathcal{L}_0- \mathcal{L}\|^2 +  {\rm log} p(y_j|x_j, \theta) \Bigg \},
   \end{split}
 \end{equation}
 where $\mathcal{D}_0$ is updated by acquiring $x^*$, i.e. $\mathcal{D}_0= \mathcal{D}_0 \cup \{x^*, y^*\}$.

{\textbf{Geodesic line.} For a metric geometry $M$, a geodesic  line  is a curve $\gamma$ which projects its interval $I$ to $M$: $I\rightarrow M$,  maintaining everywhere locally a distance minimizer \cite{lou2020differentiating}. Given a constant $\nu>0$ such that for any $a, b\in I$ there exists a geodesic  distance $d(\gamma(a),  \gamma(b)):=\int_  a^b \sqrt{{g}_{\gamma(t)}(\gamma'(t),\gamma'(t))} d t $, where $\gamma'(t)$ denotes the  geodesic curvature, and ${g}$ denotes the metric tensor over $M$.  Here, we define $\gamma'(t)=0$, then ${g}_{\gamma(t)}(0,0)=1$ such that $d(\gamma(a),  \gamma(b))$ can be  generalized as a   segment of a straight line: $d(\gamma(a),  \gamma(b))=\|a-b\|$.} 
 
{\textbf{Ellipsoid geodesic distance.} For any observation points $p, q \in M$, if the spherical geodesic distance is defined as  $d(\gamma(p),  \gamma(q))=\|p-q\|$. The affine projection obtains its ellipsoid  interpretation: $d(\gamma(p),  \gamma(q))=\|\eta(p-q)\|$, where $\eta$ denotes  the affine factor subjected to $0<\eta< 1$.   }

{\textbf{Optimizing with  ellipsoid  geodesic search.}
The max-min optimization of Eq.~(10) is performed  on an ellipsoid  geometry  to prevent    the updates of the core-set towards  the  boundary regions,} where the  ellipsoid geodesic line  scales the original  update on the sphere. Assume    $x_i$ is the previous acquisition and  $x^*$ is the next desired acquisition, the ellipsoid geodesic rescales the position of $x^*$ as  $x_e^*=x_i+\eta (x^*-x_i)$. Then, we update this position of  $x_e^*$ to its nearest neighbor $x_j$ in the unlabeled data pool, i.e. $\operatorname*{arg \ min}_{x_j\in  \mathcal{D}_u}  \|x_j- x_e^*\|$, also can be written as
 \begin{equation}
\begin{split}
\operatorname*{arg \  min}_{x_j\in  \mathcal{D}_u} \Big\|x_j- [ x_i+\eta (x^*-x_i)]\Big\|.
  \end{split}
\end{equation}
To study the advantage of ellipsoid  geodesic search, Section~6 presents our   generalization analysis.  
 
\subsection{Model uncertainty estimation with core-set}
 GBALD starts the   model uncertainty estimation with those initial core-set acquisitions, in which it introduces a ranking  scheme to derive both informative and representative acquisitions.

\textbf{Single acquisition.} We follow \cite{gal2017deep} and  use MC dropout to perform  Bayesian  inference on the  neural network model.  It then leads to ranking the informative acquisitions with batch sequences is with high efficiency. We first present the ranking criterion by rewriting Eq.~(1) as  the batch returns:
\begin{equation}
\begin{split}
&\{x^*_1,x^*_2,...,x^*_b\}=\operatorname*{arg\ max}_{\{\hat x_1,\hat x_2,...,\hat x_b\} \subseteq \mathcal{D}_u }  \\
    &\Bigg\{{\rm\bm{H}} [\theta|\mathcal{D}_0]-\mathbb{E}_{\hat y_{1:b}\sim p(\hat y_{1:b}|\hat x_{1:b},\mathcal{D}_0)}   \Big[{\rm\bm{H}}[\theta|\hat x_{1:b},\hat y_{1:b},\mathcal{D}_0] \Big]\Bigg\},\\
\end{split}
\end{equation}
where $\hat x_{1:b}=\{\hat x_1,\hat x_2,...,\hat x_b\}$,  $\hat y_{1:b}=\{\hat y_1,\hat y_2,...,\hat y_b\}$,   $\hat y_i$ denotes the output of $\hat x_i$.
The informative acquisition    $x^*_t$ is then selected  from the ranked batch acquisitions $\hat x_{1:b}$ due to the highest (most) representation for the unlabeled data: 
\begin{equation}
\begin{split}
 x^*_t =\operatorname*{arg\ max}_{ x^*_i\in \{x^*_1,x^*_2,...,x^*_b\}}  \Bigg \{ \max_{D_j\in \mathcal{D}_0} \ p(y_i |x^*_i,\theta):=\frac{R_0}{\|x^*_i-D_j \|} \Bigg\},
\end{split}
\end{equation} 
where $t$ denotes the index of the final acquisition, subjected to $1\leq t \leq b$. This also adopts the max-min  optimization of  $k$-centers in Eq.~(4), i.e.   $ x^*_t =\operatorname*{arg\ max}_{ x^*_i\in \{x^*_1,x^*_2,...,x^*_b\}}     \min_{D_j\in \mathcal{D}_0} \   {\|x^*_i-D_j \|}.$

\textbf{Batch acquisitions.}
The  greedy strategy of Eq.~(13) can be written as  a batch of acquisitions by controlling its output  as a batch set, i.e. 
\begin{equation}
\begin{split}
\{ x^*_{t_1},...,x^*_{t_{b'}} \}=\operatorname*{arg\ max}_{ x^*_{t_1:t_{b'}} \subseteq \{x^*_1,x^*_2,...,x^*_b\}}  p(y^*_{t_1:t_{b'}} |x^*_{t_1:t_{b'}},\theta),
\end{split}
\end{equation}
where $ x^*_{t_1:t_{b'}}=\{ x^*_{t_1},...,x^*_{t_{b'}} \}$, $ y^*_{t_1:t_{b'}}=\{ y^*_{t_1},...,y^*_{t_{b'}} \}$,     $y^*_{t_i}$ denotes the output of $x^*_{t_i}$, $1\leq i \leq b'$, and $1\leq b'\leq b$. This setting can be used to accelerate the acquisitions of AL in a large  dataset. 

\section{Two-stage GBALD Algorithm}
 The   GBALD algorithm has two stages: 1) construct a core-set on ellipsoid (Lines 3 to 13), and 2) estimate   model uncertainty with a deep learning model  (Lines 14 to 21).  
 
Algorithmically, core-set construction is derived from the max-min optimization of Eq.~(10), then updated with ellipsoid geodesic w.r.t. Eq.~(11), where $\theta$ yields a geometric probability model w.r.t. Eq.~(5).  Importing the core-set into $\mathcal{D}_0$ derives the deep learning model to return $b$ informative acquisitions one time, where   $\theta$ yields a deep learning model. Ranking those samples, we select $b'$ samples with the highest representations as the batch outputs w.r.t. Eq.~(14). The iterations of batch acquisitions stop until its budget is exhaust. The final update on $\mathcal{D}_0$  is our acquisition set of AL. 
 
% {\bf Details of the hyperparameter settings are presented at Appendix C.6.}

\begin{algorithm*}[!h] 
  \caption{Two-stage GBALD Algorithm}
   \textbf{Input:} Data set $\mathcal{X}$, core-set size  $N_{\mathcal{M}}$, batch returns $b$, batch output $b'$, iteration budget $\mathcal{A}$.  \\
   \textbf{Initialization:} $\alpha \leftarrow 0$, core-set $\mathcal{M}\leftarrow \emptyset$.\\ 
   \textbf{Stage \textcircled{1} begins:}  \\
Initialize $\theta$ to  yield a geometric probability model w.r.t. Eq.~(5).\\
Perform $k$-means to initialize  $\mathcal{U}$ to  $\mathcal{D}_0$.\\
Core-set  construction begins by  acquiring $x_i^*$, \\
\For {$i \leftarrow 1,2,...,N_{\mathcal{M}}$}{
$x_i^* \leftarrow  \operatorname*{ arg \ max}_{x_i\in  \mathcal{D}_u}\ \operatorname*{ min}_{D_i\in \mathcal{D}_0}  \Bigg \{\Big\|\mathcal{L}_0- \mathcal{L}\Big\|^2 +  {\rm log} p(y_i|x_i, \theta) \Bigg \},$ where $\mathcal{L}_0\leftarrow \Big|{\sum_{i=1}^N   \mathbb{E}_{y_i}[{\rm log}  p(y_i|x_i, \theta)+{\rm\bm{H}}[y_i|x_i,\mathcal{U}]]}\Big|$.\\
Ellipsoid  geodesic line  scales  $x_i^*$: 
$x_i^* \leftarrow \operatorname*{arg \  min}_{x_j\in  \mathcal{D}_u} \Big\|x_j- [ x_i+\eta (x^*-x_i)]\Big\|.$ \\
Update $x_i^*$ into core-set $\mathcal{M}$:  $\mathcal{M} \leftarrow x_i^*\cup \mathcal{M}$. \\
Update  $N  \leftarrow N-1$.\\
}
Import core-set to update $\mathcal{D}_0$: $\mathcal{D}_0\leftarrow \mathcal{M}\cup \mathcal{U}'$, where  $\mathcal{U}'$ updates each element of $\mathcal{U}$ into their nearest samples in $\mathcal{X}$.\\
  \textbf{Stage \textcircled{2} begins:}   \\
Initialize $\theta$ to  yield a deep learning model.\\
\While{$\alpha<\mathcal{A}$}{
Return $b$  informative deep learning acquisitions in one budget: $\{x^*_1,x^*_2,...,x^*_b\} \leftarrow\operatorname*{arg\ max}_{x\in  \mathcal{D}_u }  {\rm\bm{H}} [\theta|\mathcal{D}_0]-\mathbb{E}_{y\sim p(y|x,\mathcal{D}_0)}   \Big[{\rm\bm{H}}[\theta|x,y,\mathcal{D}_0] \Big]$.\\
Rank  $b'$ informative   acquisitions with  the highest  geometric representativeness: $\{ x^*_{t_1},...,x^*_{t_{b'}} \} \leftarrow\operatorname*{arg\ max}_{ x^*_i\in \{x^*_1,x^*_2,...,x^*_b\}}  p(y_i |x^*_i,\theta)$.\\
Update $\{ x^*_{t_1},...,x^*_{t_{b'}} \} $ into $\mathcal{D}_0$: $\mathcal{D}_0\leftarrow \mathcal{D}_0\cup \{ x^*_{t_1},...,x^*_{t_{b'}} \}$.\\
$\alpha\leftarrow \alpha+1$.\\
}
 \textbf{Output:}  final update on $\mathcal{D}_0$.
\end{algorithm*}

\section{Generalization errors of geodesic search with sphere  and ellipsoid}
Optimizing with ellipsoid geodesic  linearly rescales the spherical search, which draws core-set  on a tighter geometric object. The inherent motivation is that,   geodesic search with ellipsoid can prevent   the redundant updates of  core-set, avoiding those elements from spherical boundaries.  Following the approximately  perceptron analysis of \cite{sugiyama2006active}, this section presents generalization error analysis from geometry, which provides  feasible   guarantees for geodesic search with  ellipsoid.  
The proofs   are presented in Appendix.

 \subsection{Assumptions of generalization analysis}
 Let ${\rm Pr}[err(h,k)=0]_{\rm Sphere}$ and ${\rm Pr}[err(h,k)=0]_{\rm Ellipsoid}$ be the probabilities of achieving a zero  error by geodesic search with sphere and ellipsoid, respectively, we study their inequality relationship.  The assumptions are inspired from $\gamma$-tube manifold, which characterizes the probability mass of decision boundaries.

 Given $S_A$  be the sphere  that tightly covers class $A$  where $S_A$ is with a center $c_a$ and radius  $R_a$, the assumption on sphere is as follows.
\begin{assumption}\label{Assumption_sphere}
Ben-David et al. \cite{ben2008relating}  proposed that the $\gamma$-tube manifold    \cite{li2020finding}  can characterize the probability mass of an optimal version space-based hypothesis. From geometry, we here assume that the probability mass of  achieving a zero   error  
 by geodesic search  with sphere,  is roughly defined  as the volume ratio of the  $\gamma$-tube and   sphere, that is,   ${\rm Pr} [err(h,k)=0]_{\rm Sphere}:=\frac{\rm {Vol(Tube)}}{{\rm Vol(Sphere)}}$.  Given $\frac{\pi}{\varphi}={\rm arcsin} \frac{R_a-d_a}{R_a}$,   the   assumption  is formalized as 
\[{\rm Pr} [err(h,k)=0]_{\rm Sphere}=1-\frac{t_k^3}{R_a^3},\]
where   $t_k= \frac{R_a^2}{3}+\sqrt[3]{-\frac{\mu_k}{2 \pi}+\sqrt{\frac{\mu_k^2}{4 \pi^2}-\frac{ \pi^3  R_a^3}{27\pi^3}      }       }+\sqrt[3]{-\frac{\mu_k}{2 \pi}-\sqrt{\frac{\mu_k^2}{4 \pi^2}-\frac{ \pi^3  R_a^3}{27\pi^3}      }       }$, and  $\mu_k= (\frac{2k-4}{3k}-\frac{1}{\varphi}cos \frac{\pi}{\varphi}  )\pi R_a^3- \frac{4\pi R_b^3}{3k}$.
\end{assumption}

Given   class  $A$   is tightly covered by   ellipsoid   $E_a$, let  $R_{a_1}$ be the polar radius of   $E_a$, and $R_{a_2}$, $R_{a_3}$ be the equatorial radii of $E_a$,  the assumption on  ellipsoid  is as follows.

\begin{assumption}\label{Assumption_Ellipsoid}
Following Assumption~\ref{Assumption_sphere},   the probability  mass of achieving a zero  error  by geodesic search with ellipsoid, can be assumed  as the volume ratio of the  $\gamma$-tube and    ellipsoid, that is ${\rm Pr} [err(h,k)=0]_{\rm Ellipsoid}:=\frac{\rm {Vol(Tube)}}{{\rm Vol(Ellipsoid)}}$.  Then, the  assumption   is formalized as 
\[{\rm Pr} [err(h,k)=0]_{\rm Ellipsoid}=1-\frac{\lambda_{k_1}\lambda_{k_2}\lambda_{k_3} }{R_{a_1}R_{a_2}R_{a_3} },\]
where $\lambda_{k_i}= \frac{R_{a_i}^2}{3}+\sqrt[3]{-\frac{\sigma_{k_i}}{2 \pi}+\sqrt{\frac{\sigma_{k_i}^2}{4 \pi^2}-\frac{ \pi^3  R_{a_i}^3}{27\pi^3}      }       }+\sqrt[3]{-\frac{\sigma_{k_i}}{2 \pi}-\sqrt{\frac{\sigma_{k_i}^2}{4 \pi^2}-\frac{ \pi^3  R_{a_i}^3}{27\pi^3}      }       }$, and $\sigma_{k_i}= (\frac{2k-4}{3k}-\frac{\pi R_{a_i}  }{2\varphi}  )\pi R_{a_1}R_{a_2}R_{a_3}- \frac{4\pi R_{b_1}R_{b_2}R_{b_3}}{3k}$, $i=1,2,3$. 
\end{assumption}

The specifications of  of Assumptions~\ref{Assumption_sphere} and \ref{Assumption_Ellipsoid} are presented in Appendix A.2 and A.3, respectively. We next present our low-dimensional generalization analysis.

\subsection{Low-dimensional generalizations}
 Our generalization analysis begins from low-dimensional (3-D) sphere/ellipsoid  to high-dimensional   hypersphere/hyperellipsoid, where the high-dimensional settings can be extended into infinite dimensions.

 \subsubsection{Our settings}
\textbf{Geodesic search with sphere.}  With Assumptions~1 and 2, given a perceptron function $h:={w_1 x_1+w_2x_2+w_3}$, the  task is  to classify the two  classes $A$ and $B$ embedded  in a 3-D space. Let $S_A$ and $S_B$ be the spheres that tightly cover $A$ and $B$, respectively, where $S_A$ is with a center $c_a$ and radius  $R_a$,  and $S_B$ is with a center $c_b$ and radius $R_b$. 
  Under this setting,  our generalization  analysis is presented as follows.

\begin{theorem} \label{Geodesic_search_with_sphere}
With Assumptions~1 and 2, given a  perceptron  function $h={ w_1 x_1+w_2x_2+w_3}$ that classifies $A$ and $B$,  and a sampling budget $k$. By drawing core-set  on $S_A$ and $S_B$, the minimum distances to the boundaries of that core-set elements  of $S_A$ and $S_B$, are defined as  $d_a$ and  $d_b$, respectively.  Let  $err(h,k)$ be the classification error rate with respect to $h$ and $k$,  given $\frac{\pi}{\varphi}={\rm arcsin} \frac{R_a-d_a}{R_a}$,  we then have an inequality of error:
${\rm min} \Bigg \{ \frac{4R_a^3-(2R_a+t_k)(R_a-t_k)^2}{4R_a^3+4R_b^3},   \frac{4R_b^3-(2R_b+t_k')(R_b-t_k')^2}{4R_b^3+4R_a^3}\Bigg\}$\[< err(h,k)<\frac{1}{k}, \]
where $t_k= \frac{R_a^2}{3}+\sqrt[3]{-\frac{\mu_k}{2 \pi}+\sqrt{\frac{\mu_k^2}{4 \pi^2}-\frac{ \pi^3  R_a^3}{27\pi^3}      }       }+\sqrt[3]{-\frac{\mu_k}{2 \pi}-\sqrt{\frac{\mu_k^2}{4 \pi^2}-\frac{ \pi^3  R_a^3}{27\pi^3}      }       }$,  $\mu_k= (\frac{2k-4}{3k}-\frac{1}{\varphi}cos \frac{\pi}{\varphi}  )\pi R_a^3- \frac{4\pi R_b^3}{3k}$, $t_k'= \frac{R_b^2}{3}+\sqrt[3]{-\frac{\mu_k'}{2 \pi}+\sqrt{\frac{\mu_k'^2}{4 \pi^2}-\frac{ \pi^3  R_b^3}{27\pi^3}      }       }+\sqrt[3]{-\frac{\mu_k'}{2 \pi}-\sqrt{\frac{\mu_k'^2}{4 \pi^2}-\frac{ \pi^3  R_b^3}{27\pi^3}      }       } $,  and  $\mu_k'=  (\frac{2k-4}{3k}-\frac{1}{\varphi}cos \frac{\pi}{\varphi}  )\pi R_b^3- \frac{4\pi R_a^3}{3k}$.
\end{theorem}

\textbf{Geodesic search with ellipsoid.}  With Assumptions~1 and 2, given   class  $A$ and $B$   are tightly covered by   ellipsoid   $E_a$ and $E_b$ in a 3-D space. Let  $R_{a_1}$ be the polar radius of   $E_a$, and 
$R_{a_2}, R_{a_3}$ be the equatorial radii of $E_a$,  $R_{b_1}$ be the polar radius of   $E_b$, 
and $R_{b_2}, R_{b_3}$ be the equatorial radii of $E_b$,  the generalization analysis is ready to present following these settings.

\begin{theorem}\label{Geodesic_search_with_ellipsoid}
With Assumptions~1 and 2, given a  perceptron  function $h={w_1 x_1+w_2x_2+w_3}$ that classifies $A$ and $B$,  and a sampling budget $k$.  By  drawing core-set on $E_a$ and $E_b$, the minimum distances to the boundaries of that core-set elements  of $S_A$ and $S_B$, are defined as  $d_a$ and  $d_b$, respectively.  Let  $err(h,k)$ be the classification error rate with respect to $h$ and $k$, given $\frac{\pi}{\varphi}={\rm arcsin} \frac{R_a-d_a}{R_a}$,  we then have an inequality of error:
${\rm min} \Bigg \{ \frac{4\prod_i \!R_{a_i}-(2R_{a_1}+\lambda_k)(R_{a_1}-\lambda_k)^2}{4\prod_i R_{a_i}+4\prod_i R_{b_i}}, \!\!  \frac{4\prod_i \!R_{b_i}-(2R_{b_1}+\lambda_k')(R_{b_1}-\lambda_k')^2}{4\prod_i R_{b_i}+4\prod_i R_{a_i}}\Bigg\}$
\[< err(h,k)<\frac{1}{k},\] 
where $i=1,2,3$,  $\lambda_k= \frac{R_{a_1}^2}{3}+\sqrt[3]{-\frac{\sigma_k}{2 \pi}+\sqrt{\frac{\sigma_k^2}{4 \pi^2}-\frac{ \pi^3  R_{a_1}^3}{27\pi^3}      }       }+\sqrt[3]{-\frac{\sigma_k}{2 \pi}-\sqrt{\frac{\sigma_k^2}{4 \pi^2}-\frac{ \pi^3  R_{a_1}^3}{27\pi^3}      }       }$,  $\sigma_k= (\frac{2k-4}{3k}-\frac{\pi R_{a_1}  }{2\varphi}  )\pi \prod_i R_{a_i}- \frac{4\pi \prod_i R_{b_i}}{3k}$,  
  $\lambda_k'= \frac{R_{b_1}^2}{3}+\sqrt[3]{-\frac{\sigma_k}{2 \pi}+\sqrt{\frac{\sigma_k'^2}{4 \pi^2}-\frac{ \pi^3  R_{b_1}^3}{27\pi^3}      }       }+\sqrt[3]{-\frac{\sigma_k}{2 \pi}-\sqrt{\frac{\sigma_k^2}{4 \pi^2}-\frac{ \pi^3  R_{b_1}^3}{27\pi^3}      }       }$, and   $\sigma_k'= (\frac{2k-4}{3k}-\frac{\pi R_{b_1}  }{2\varphi}  )\pi\prod_i R_{b_i}- \frac{4\pi\prod_i R_{a_i}}{3k}$.
\end{theorem}

\subsubsection{Our insights}
\textbf{Insight $\textcircled{1}$: tighter lower error bound.}
Let ${\rm lower}[err(h,k)]_{\rm Sphere}$ and ${\rm lower}[err(h,k)]_{\rm Ellipsoid}$ be the lower bounds of the generalization errors by geodesic search with sphere and ellipsoid, respectively. 
With $R_{a_1}<R_a$, compare Theorems~\ref{Geodesic_search_with_sphere} and \ref{Geodesic_search_with_ellipsoid}, we have the following proposition.
\begin{proposition}
Given a  perceptron  function $h={w_1 x_1+w_2x_2+w_3}$ that classifies $A$ and $B$,  and a sampling budget $k$. By  drawing core-set on $S_a$ and $S_b$,  let  $err(h,k)$ be the classification error rate with respect to $h$ and $k$, with Theorems~\ref{Geodesic_search_with_sphere} and \ref{Geodesic_search_with_ellipsoid}, the lower bounds of geodesic search with sphere and ellipsoid satisfy: ${\rm lower}[err(h,k)]_{\rm Ellipsoid}<{\rm lower}[err(h,k)]_{\rm Sphere}$.
\end{proposition}

 \textbf{Insight $\textcircled{2}$: higher probability of achieving a  zero error.}
Let ${\rm Pr}[err(h,k)=0]_{\rm Sphere}$ and  ${\rm Pr}[err(h,k)=0]_{\rm Ellipsoid}$ be the probabilities of achieving a zero  error  of geodesic search with sphere and ellipsoid, respectively. Their relationship is presented in Proposition~2.
\begin{proposition}
Given a  perceptron  function $h={w_1 x_1+w_2x_2+w_3}$ that classifies $A$ and $B$,  and a sampling budget $k$. By  drawing core-set on $E_a$ and $E_b$,  let  $err(h,k)$ be the classification error rate with respect to $h$ and $k$, with Assumptions~\ref{Assumption_sphere} and \ref{Assumption_Ellipsoid}, the probabilities of geodesic search with sphere and ellipsoid satisfy: ${\rm Pr}[err(h,k)=0]_{\rm Ellipsoid}>{\rm Pr}[err(h,k)=0]_{\rm Sphere}$.
\end{proposition}

Overall,   geodesic search with ellipsoid  is more effective than with   sphere, due to 1)  tighter lower error bound, and 2)   higher probability to achieve a zero error. 
\subsection{High-dimensional generalizations}

With the above insights,  we next present a connection between 3-D sphere/ellipsoid and ${\bm d}$-dimensional hyperesphere/hyperellipsoid, where ${\bm d}>3$. The major technique is to prove that the volume of the 3-D sphere and ellipsoid are   lower dimensional generalization of the ${\bm d}$-dimensional hypheresphere and hyperellipsoid, respectively. References can refer to $n$-sphere \cite{barnea1999hyperspherical}, \cite{blumenson1960derivation}, and volume prototypes of  hyperellipsoids \cite{kaymak2002fuzzy}. 
With volume generalization analysis,  all proofs from Theorems~1 to 2 and Propositions~1 and 2  can hold in  ${\bm d}$-dimensional geometry.

In the following, Theorems~3 and 4  then present a high-dimensional generalization for the above theoretical results, in terms of the volume functions of sphere and ellipsoid. 

\begin{theorem} Let ${\rm Vol}_{\bm d}(S_a)$ or ${\rm Vol}_{\bm d}(r_a)$ be the volume of   ${\bm d}$-dimensional hypersphere $S_a$ with a radius $r_a$, given $\vartheta \in [0,\pi]$,
by performing integral operation on any $({\bm d \text{-} 1})$-dimensional hypersphere, there exists ${\rm Vol}_{\bm d}$  can be approximated as ${\rm  Vol}_{\bm m}(S_a) =\int_{0}^{\pi/2} 2  {\rm Vol}_{\bm m\text{-}1}(r_acos(\vartheta)) r_a(cos(\vartheta))  d\vartheta$, 
and we define this operation as  $ {\rm  Vol}_{\bm m}(S_a) \bowtie {\rm Vol}_{\bm m\text{-}1}(S_a)$. 
 Then, we know   $   {\rm  Vol}_{\bm m\text{-}1}(S_a) \bowtie {\rm Vol}_{\bm m\text{-}2}(S_a)$, $ {\rm  Vol}_{\bm m\text{-}2}(S_a) \bowtie {\rm Vol}_{\bm m\text{-}3}(S_a)$, ..., $   {\rm  Vol}_{\bm 4}(S_a) \bowtie {\rm Vol}_{\bm 3}(S_a)$. With this progressive relationship, 
we can say ${\rm Vol}_{\bm 3}(S_a)$ is a low-dimensional generalization of ${\rm Vol}_{\bm d}(S_a)$.
\end{theorem}

 The proof skills of Theorem~3 can refer to  a mathematical perspective \footnote{\url{https://www.sjsu.edu/faculty/watkins/ndim.htm}}. Appendix A.8 also presents a machine learning proof skill. 
Moreover, the proof   of Theorem~3  can be adopted in the generalization of 3-D ellipsoid to ${\bm d}$-D hyperellipsoid.
\begin{theorem} Let ${\rm Vol}_{\bm d}(E_a)$ be the volume of a ${\bm d}$-dimensional  hyperellipsoid, given $\vartheta \in [0,\pi]$,
by performing integral operation on any $({\bm d\text{-}1}$)-dimensional hyperellipsoid, there exists    $   {\rm  Vol}_{\bm m-1}(S_a) \bowtie {\rm Vol}_{\bm m\text{-}2}(S_a)$, $ {\rm  Vol}_{\bm m\text{-}2}(S_a) \bowtie {\rm Vol}_{\bm m\text{-}3}(S_a)$, ..., $   {\rm  Vol}_{\bm 4}(S_a) \bowtie {\rm Vol}_{\bm 3}(S_a)$. With this progressive relationship, 
we can say 
${\rm Vol}_{\bm 3}(E_a)$ is a low-dimensional generalization of ${\rm Vol}_{\bm d}(E_a)$.
\end{theorem}

The volume of a hyperellipsoid can also be generalized by replacing the operation $\prod_i R_a$ by the polar radius  $\prod_i R_{a_i}$,  for $0<i<{\bm d}+1$,  in the formula for the volume of a hypersphere.   Proofs also can refer to a mathematical skill\footnote{\url{https://www.sjsu.edu/faculty/watkins/ellipsoid.htm}}.

\section{Experiments}
 In experiments, we start  by showing how BALD degenerates  its performance  with an uninformative prior and the redundant information,  and show that how our proposed GBALD   relieves theses limitations. 
\par Our  experiments discuss three questions: 1) is   GBALD using core-set of Eq.~(11)  competitive  with an uninformative prior? 2) can  GBALD using the ranking of Eq.~(14)  improve the informative acquisitions of model uncertainty? and 3) can GBALD outperform  the  state-of-the-art acquisition approaches?
Following the experiment settings of \cite{gal2017deep,kirsch2019batchbald}, we   use  MC dropout  to implement the Bayesian approximation of DNNs.
Three benchmark datasets are selected: MNIST, SVHN, and CIFAR10.

\subsection{Baselines}
To evaluate the performance of GBALD,  several typical baselines from the latest   deep AL literature  are selected.
\begin{itemize}

\item Bayesian active learning by disagreement (BALD) \cite{houlsby2011bayesian}.  It has been introduced in Section~3.

\item  Maximize variation ratio (Var) \cite{gal2017deep}. The algorithm chooses the unlabeled data that maximizes its variation ratio of the probability:
\begin{equation}
x^*=\operatorname*{ arg \ max}_{x\in  \mathcal{D}_u}  \Big\{  1-\max_{y\in\mathcal{Y}} \ {\rm Pr}(y|,x,\mathcal{D}_0)) \Big\}.
\end{equation}

\item  Maximize  entropy (Entropy) \cite{gal2017deep}. The algorithm chooses the unlabeled data that maximizes the predictive entropy:
  \begin{equation}
x^*=\operatorname*{ arg \ max}_{x\in  \mathcal{D}_u}  \Big\{   - \sum_{y\in \mathcal{Y}}{\rm Pr}(y|x,\mathcal{D}_0)){\rm log}\Big({\rm Pr}(y|x,\mathcal{D}_0)\Big) \Big\}.
\end{equation}

\item  $k$-modoids \cite{park2009simple}. A classical unsupervised algorithm that represents the input distribution by $k$ clustering centers:
\begin{equation}
\{x_1^*,x_2^*,...,x_k^*\}=\operatorname*{ arg \ min}_{z_1,z_2,...,z_k}  \Big\{ \sum_{i=1}^{k}\sum_{ z_i\in \mathcal{X}^k  }\|x_i-z_i\|        \Big\},
\end{equation}
where  $\mathcal{X}^k$ denotes the $k$-th subcluster centered with $z_i$, and $z_i\in \mathcal{X},\forall i$.

\item Greedy $k$-centers ($k$-centers) \cite{DBLP:conf/iclr/SenerS18}.  A geometric core-set interpretation on sphere.  See Eq.~(4).
\item BatchBALD \cite{kirsch2019batchbald}. A batch extension of BALD which incorporates the diversity, not maximal entropy as BALD,  to rank the acquisitions:
\begin{equation}
\begin{split}
&\{ x^*_{t_1},...,x^*_{t_b} \} \\
&=\operatorname*{arg\ max}_{   x_{t_1},...,x_{t_b}}   {\rm H}(y_{t_1},...,y_{t_b})   - {\rm E}_{p(\theta|\mathcal{D}_0)}[{\rm H}(y_{t_1},...,y_{t_b}| \theta) ],\\
\end{split}
\end{equation} 
where ${\rm H}(y_{t_1},...,y_{t_b})$ denote the expected entropy over all possible labels from $y_{t_1}$ to $y_{t_b}$ such that  ${\rm H}(y_{t_1},....,y_{t_b})={\rm E}_p(y_{t_1},..,y_{t_b})[- {\rm log} p(y_{t_1},..,y_{t_b}]$, and ${\rm E}_{p(\theta|\mathcal{D}_0)}[{\rm H}(y_{t_1},...,y_{t_b}| \theta) ]$ is estimated by  MC sampling \cite{roy2001toward} \cite{osborne2012active} a subset    from $\mathcal{X}$  which approximates the parameter distributions of $\theta$.
\end{itemize}

\textbf{Parameters of GBALD.} The parameter settings of Eq.~(5) are   $R_0=2.0e+3$ and $\eta=$0.9. Accuracy of each acquired  dataset  of the experiments are averaged over 3 runs.

\subsection{Uninformative priors} 
 
As discussed in the introduction, BALD is sensitive to an uninformative prior, i.e. $p(\mathcal{D}_0|\theta)$. We thus initialize $\mathcal{D}_0$ from a fixed class of the training data of  datasets to observe its acquisition performance. In this way,  $p(\mathcal{D}_0|\theta)$ is uninformative due to extreme label category. 

Figure~3 presents the prediction accuracies of BALD with an acquisition budget of 130 over the training set of MNIST. Based on the uninformative setting of 
$p(\mathcal{D}_0|\theta)$,   we randomly select 20 samples from   digit `0' and `1' to initialize $\mathcal{D}_0$, respectively. The   classification model of AL follows a   convolutional neural network (CNN)   with  one block of [convolution,
dropout, max-pooling, relu], with 32, 3x3 convolution filters,  5x5 max pooling, and 0.5  dropout rate.  In the AL loops, 
 we use 2,000 MC dropout   samples      from the unlabeled data pool to fit the training of the network  as with \cite{kirsch2019batchbald}. 
 
\begin{figure*} 
 
\subfloat[Digit `0']{
\label{fig:improved_subfig_b}
\begin{minipage}[t]{0.49\textwidth}
\centering
\includegraphics[width=2.72in,height=2.22in]{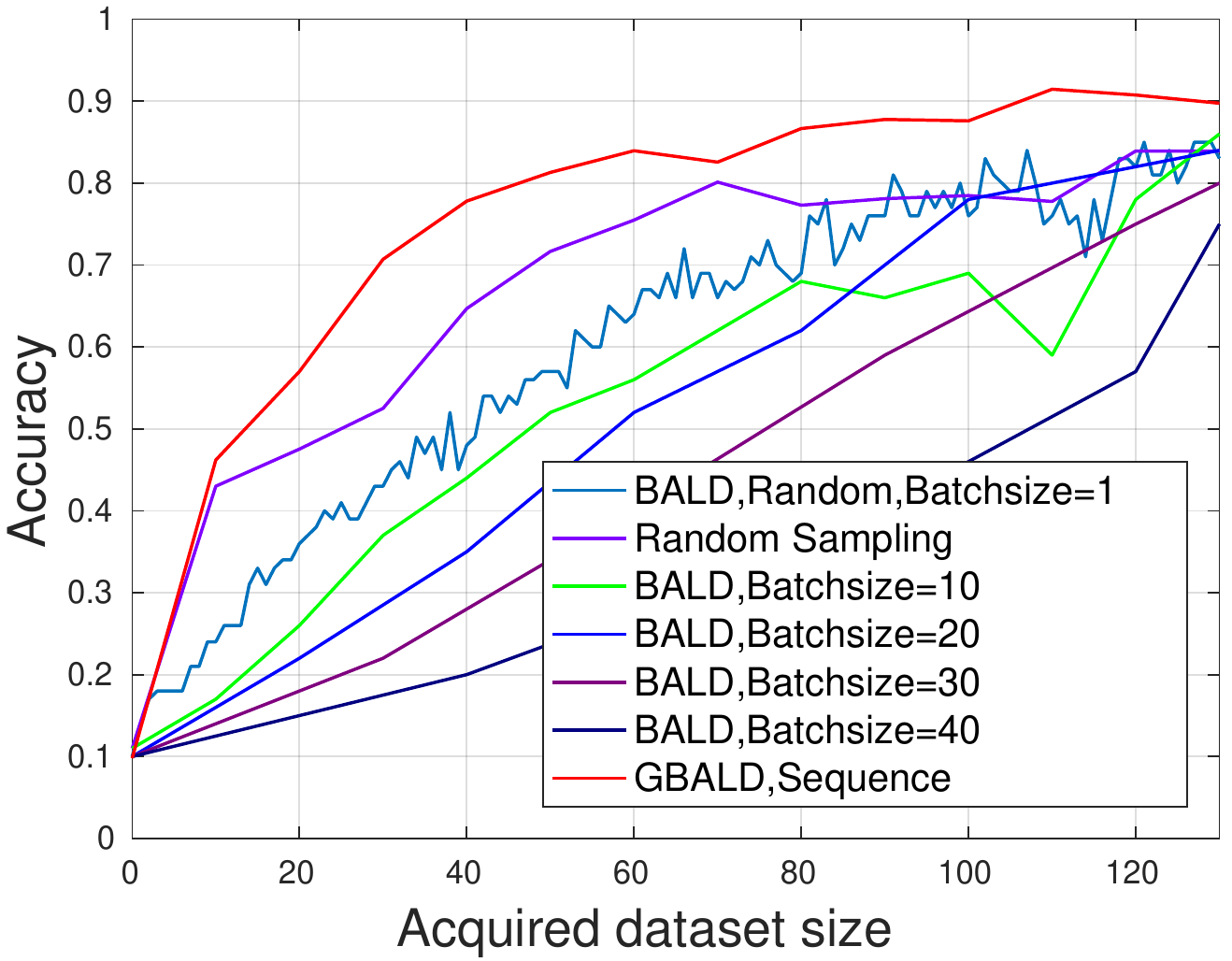}
\end{minipage}
}
\subfloat[Digit `1']{
\label{fig:improved_subfig_b}
\begin{minipage}[t]{0.49\textwidth}
\centering
\includegraphics[width=2.72in,height=2.22in]{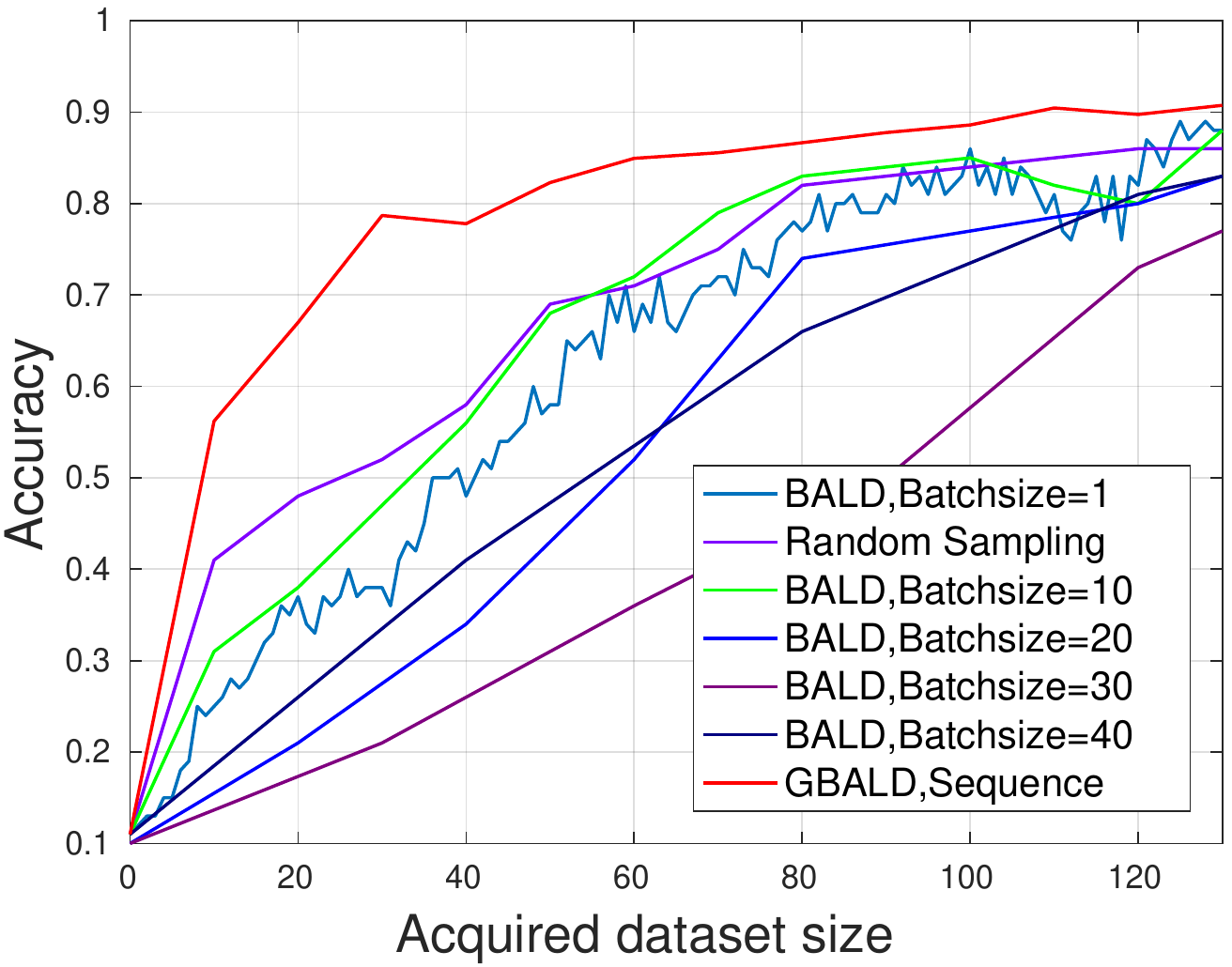}
\end{minipage}
}
\caption{Acquisitions with uninformative  priors from digit  `0' and `1'.
 }  
 
\end{figure*}

As the figure shown,   BALD  can slowly accelerate the training model due to the biased initial acquisitions, which cannot uniformly cover all the label categories. Moreover, the uninformative prior guides BALD to unstable acquisition results. Specifically, in Figure~3(b), BALD with Bathsize\ = 10 shows better performance than that of Batchsize\ =1; while BALD in Figure~3(a) keeps stable performance. This is because the initial labeled data does not cover all classes and   BALD with Batchsize\ =1 may further be misled to select those samples from  one or a few fixed classes at the first acquisitions. However,  Batchsize >1 may  result in a random acquisition process  that possibly  covers more diverse labels at its first acquisitions. Another excursive result of    BALD  is that  the increasing batch size cannot degenerate its acquisition performance   in  Figure~3(b). For example,  Batchsize\ =10 $\succ$ Batchsize\ =1\ $\succ$ Batchsize\ =20,40\ $\succ$ Batchsize\ =30, where `$\succ$' denotes `better' performance;     Batchsize\ = 20 achieves similar results as with Batchsize\ =40. \textbf{This undermines
 the 
acquisition policy of BALD: its performance would be degenerated when  the batch size increases, and sometimes  worse than random sampling. This also is the reason why we utilize a core-set to start BALD in our framework. }

\par Different to BALD,   core-set construction of GBALD using Eq.~(11) provides a \textbf{complete label matching against all classes}. Therefore, it outperforms BALD with the  batch sizes of 1, 10, 20, 30, and 40. As the shown  learning curves in Figure~3,  GBALD with a batch size of  1 and sequence 
size of 10 (i.e. breakpoints of acquired size are  10,\ 20,\ ...,\ 130)   achieves significantly higher accuracies than BALD using different batch sizes since BALD misguides the network updating using a poor prior.

 \begin{figure*} 
\subfloat[Digit `0']{
\label{fig:improved_subfig_b}
\begin{minipage}[t]{0.49\textwidth}
\centering
\includegraphics[width=2.72in,height=2.22in]{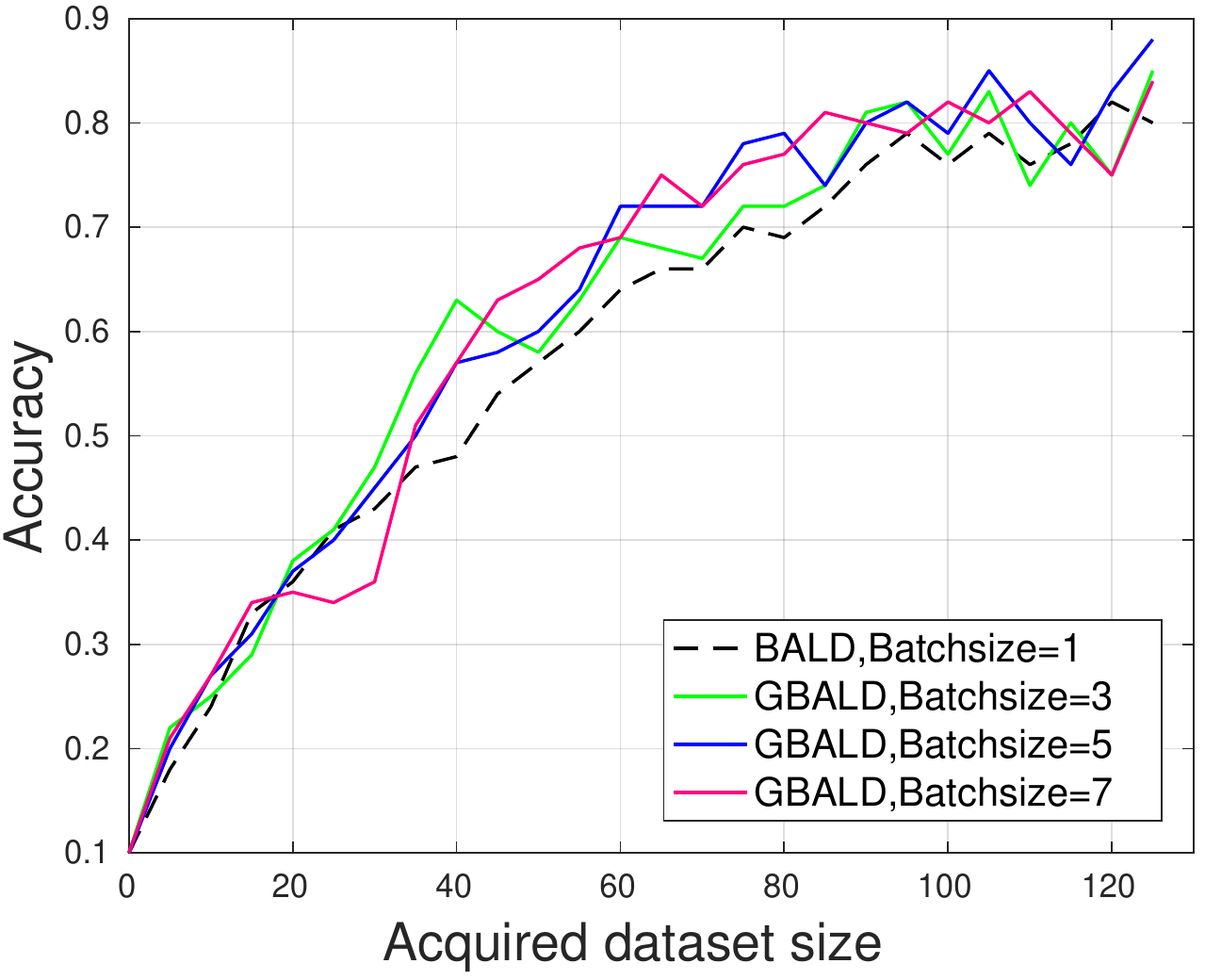}
\end{minipage}
}
\subfloat[Digit `1']{
\label{fig:improved_subfig_b}
\begin{minipage}[t]{0.49\textwidth}
\centering
\includegraphics[width=2.72in,height=2.22in]{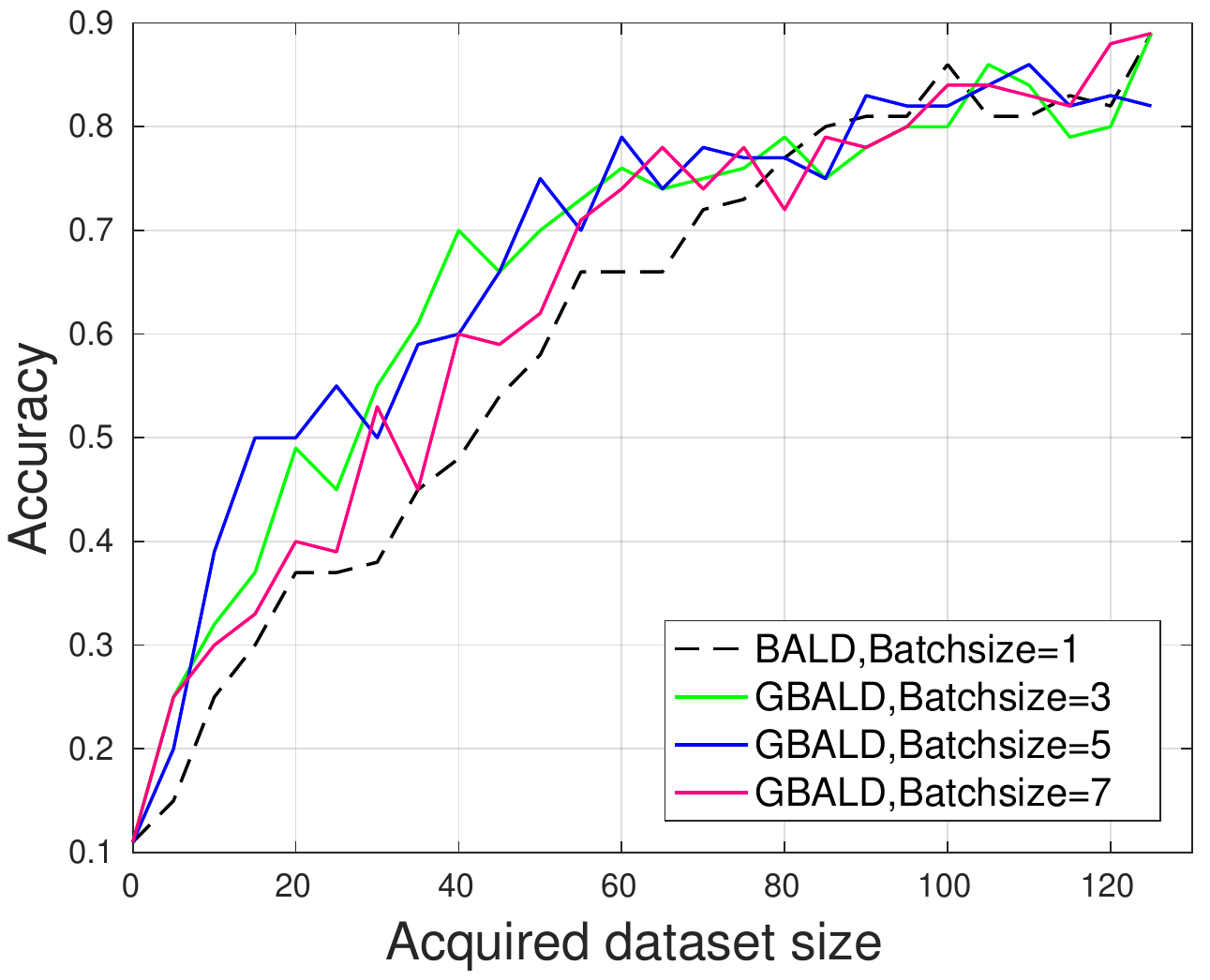}
\end{minipage}
}
\caption{GBALD outperforms BALD using ranked informative acquisitions which cooperate with representation constraints.
 } 
 
\end{figure*}

\subsection{Improved informative acquisitions}
For BALD,   repeated or similar acquisitions can easily  delay  the acceleration or improvement of the model training. Following the experiment  settings of Section~7.1, we compare the best performance of BALD with a batch size of 1 and GBALD with different batch size parameters.   Following Eq.~(14), we set $b=\{3, 5, 7\}$ and $b'$=1, respectively, that means,  we output the most representative data from  a batch of highly-informative acquisitions. Different settings on $b$ and $b'$ are used to observe the parameter perturbations of GBALD.
 
Training  by the same  parameterized CNN model as in Section~7.2, Figure~4  presents the acquisition performance of parameterized BALD and GBALD. As the learning curves shown, 
BALD cannot accelerate the model as fast as GBALD due to the repeated information over the acquisitions. For GBALD, it ranks the batch acquisitions of the highly-informative samples and selects the most representative ones. By employing  this special ranking strategy, GBALD can  \textbf{reduce the probability of sampling those nearby data of the previous acquisitions}. It is thus GBALD significantly outperforms BALD, even if we progressively increase the ranked batch size $b$.   
 
\subsection{Active acquisitions}
 
GBALD using Eqs.~(11) and (14) has been demonstrated to achieve  successful improvements over BALD. We thus combine  these two components into a uniform framework. Figure~5 presents the AL accuracies using different acquisition algorithms on the three image datasets.  The selected baselines  follow \cite{gal2017deep} including 1) maximizing   the variation ratios (Var),  2)  BALD, 3) maximizing  the entropy (Entropy), 4) $k$-medoids, and  one greedy  5) $k$-centers approach \cite{DBLP:conf/iclr/SenerS18}. The network architecture is a three-layer multi-layer perceptron (MLP) with three blocks of [convolution,
dropout, max-pooling, relu], with 32, 64, and 128 3x3 convolution filters,  5x5 max pooling, and 0.5  dropout rate.  In the AL loops, 
  the MC dropout still randomly samples  2,000  data    from the unlabeled data pool to approximate the training of the network architecture following \cite{kirsch2019batchbald}. The initial labeled data of MNIST, SVHN and CIFAR-10 are 20, 1000, 1000 random samples from their full training sets, respectively.

 The batch size of the compared baselines is 100, where GBALD ranks 300 acquisitions to select 100 data for the training, i.e. $b=300,b'=100$.   As the  learning curves shown in Figure~5, 1) $k$-centers algorithm performs more poorly than the other compared baselines because the representation optimization with the sphere geodesic usually falls into the selection of the boundary data; 2) Var, Entropy, and BALD algorithms cannot accelerate  the network model rapidly due to  those highly-skewed acquisitions towards few fixed classes  at its first acquisitions (start states); 
3) $k$-medoids approach does not interact with the neural network model while directly  imports  the clustering centers into its training set and the results are not strong; 4) the accuracies of the acquisitions of GBALD achieve better performance at the beginning than the Var, Entropy, and BALD approaches which fed the training set of the network model via acquisition loops. In short, \textbf{the network is improved faster  after drawing the distribution characteristics} of the input dataset with sufficient labels. GBALD thus consists of the representative and informative acquisitions in its uniform framework. The advantages of  these two acquisition paradigms  are integrated to 
present higher accuracies than any  single paradigm. 

 \begin{figure*} 
\subfloat[MNIST]{
\label{fig:improved_subfig_b}
\begin{minipage}[t]{0.32\textwidth}
\centering
\includegraphics[width=2.38in,height=1.78in]{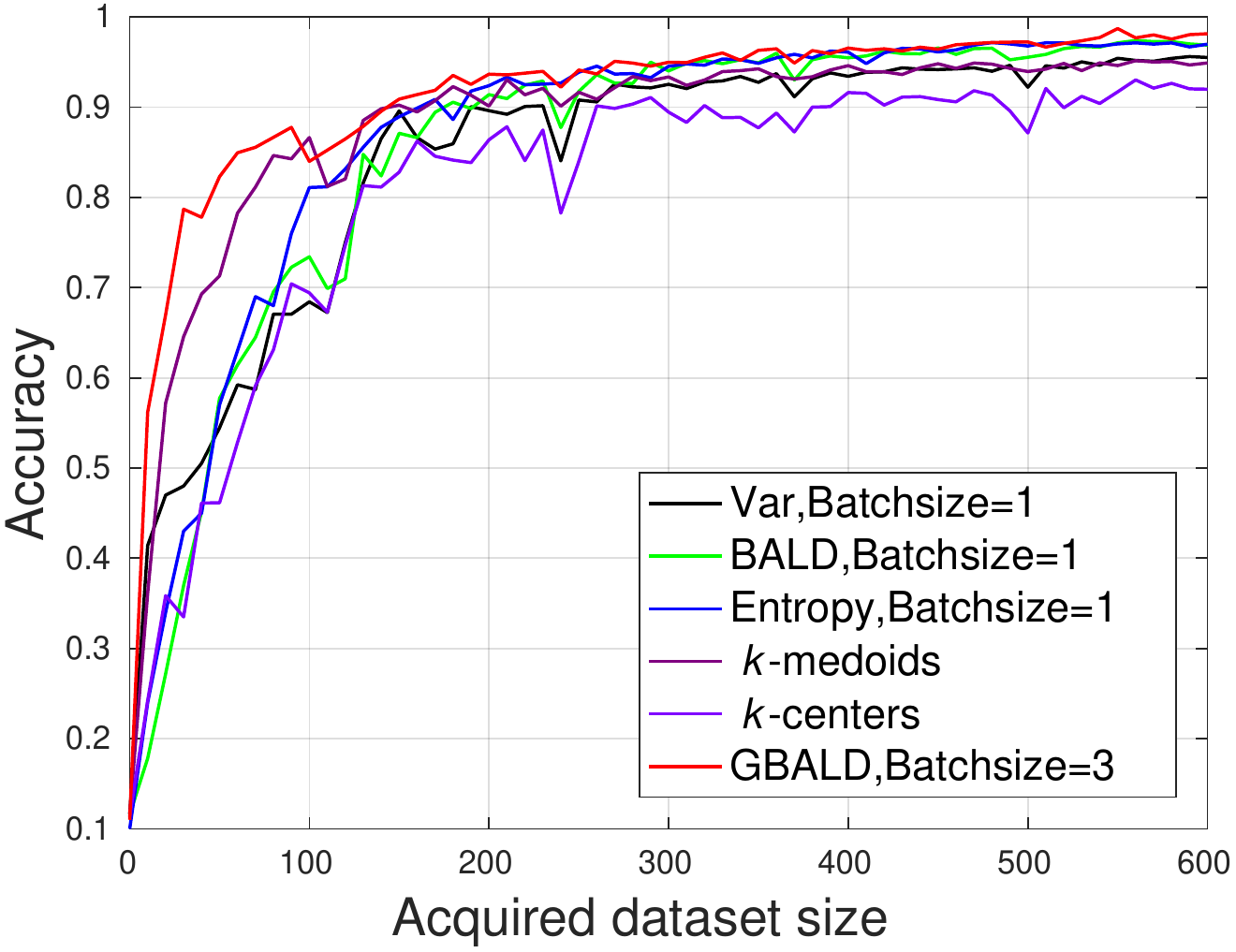}
\end{minipage}
}
\subfloat[SVHN]{
\label{fig:improved_subfig_b}
\begin{minipage}[t]{0.32\textwidth}
\centering
\includegraphics[width=2.38in,height=1.78in]{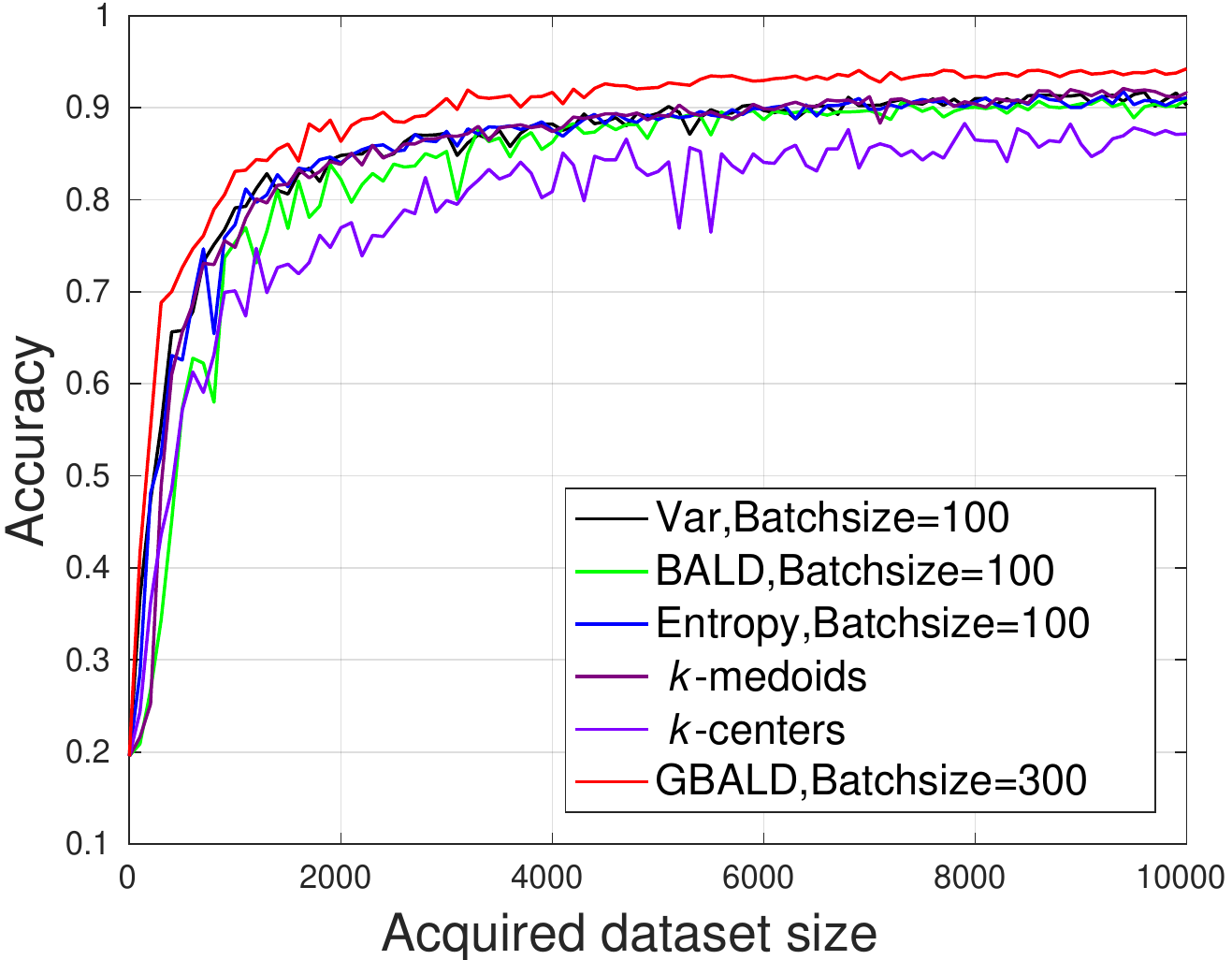}
\end{minipage}
}
\subfloat[CIFAR10]{
\label{fig:improved_subfig_b}
\begin{minipage}[t]{0.32\textwidth}
\centering
\includegraphics[width=2.38in,height=1.78in]{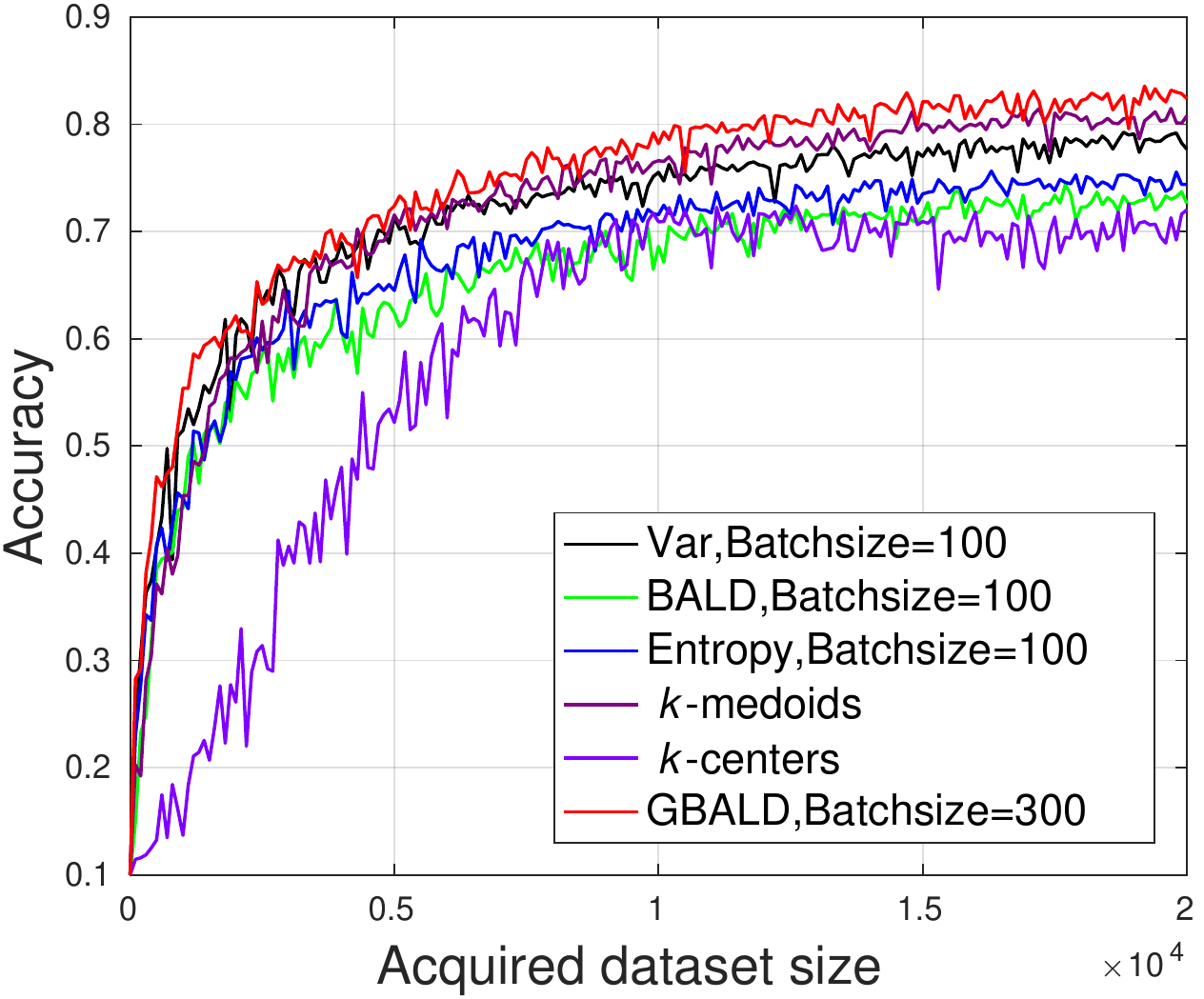}
\end{minipage}
} 
\caption{ Active  acquisitions on MNIST, SVHN, and CIFAR10 datasets. 
 } 
\end{figure*}

\begin{table*}
 \caption{Mean$\pm$std of the test accuracies of the breakpoints of the learning curves on MNIST, SVHN, and CIFAR-10. }
 \setlength{\tabcolsep}{5.5pt}{
\begin{center}
\scalebox{1.2}{
 \begin{tabular}{c| c c c  c c c } 
\hline
\multirow{2}{*}{Datasets}  &  &  \multicolumn{4}{c}{Algorithms}  \\
        &    Var  &   BALD  &      Entropy &$k$-medoids  &$k$-centers &GBALD    \\

\hline
   MNIST     &0.8419$\pm$ 0.1721   &0.8645$\pm$0.1909  & 0.8498$\pm$0.2098   &0.8785$\pm$0.1433&0.8052$\pm$0.1838 &\textbf{0.9106$\pm$0.1296}    \\

  SVHN   &0.8535$\pm$0.1098& 0.8510$\pm$0.1160 &0.8294$\pm$0.1415 &  0.8498$\pm$0.1294 & 0.7909$\pm$0.1235&\textbf{0.8885$\pm$0.1054}  \\

  CIFAR-10     &0.7122$\pm$0.1034&0.6760$\pm$0.1023&0.6536$\pm$0.1038&   0.71837$\pm$0.1245&  0.5890$\pm$0.1758 &\textbf{0.7440$\pm$0.1087}     \\
\hline
\end{tabular}}
\end{center}}
\end{table*}

 \begin{table}
 \caption{Number of acquisitions  on MNIST, SVHN and CIFAR10 until 70\%, 80\%, and 90\% accuracies are reached. }

\setlength{\tabcolsep}{3.0pt}
\begin{center}
\scalebox{1.0}{
 \begin{tabular}{c| c c c  } 
\hline
\multirow{2}{*}{Algorithms}  &  & \multicolumn{1}{c}{Accuracies}   \\
                                         &   70\% &  80\%   &      90\%   \\

\hline
   Var       & 140/1,700/5,700   &150/2,200/>20,000  &210/>10,000/>6,100      \\
   
    BALD   &110/1,700 /8,800&120 /2,300/>20,000&190/7,100 / >20,000 \\

Entropy     &110/1,900/11,200&150/2,400/>20,000&200/8,600/>20,000  \\

$k$-modoids    &70/1,700/5,900&90/2,200/16,000&\textbf{170}/6,200 />20,000   \\

$k$-centers   &110/2,000/10,100&150/3,800/>20,000&280/>10,000/>20,000  \\

GBALD  &\textbf{50/1,400/4,800}&\textbf{70/1,900/12,200}&\textbf{170/3,900}/>20,000     \\
\hline
\end{tabular}}
\end{center}

\end{table}

\par Table~1 reports the mean$\pm$std values of the test accuracies of the  breakpoints of the  learning curves in Figure~5, where the breakpoints of MNIST are $\{0,10,20,30,...,600\}$, the breakpoints of SVHN are $\{0,100,200,...,10000\}$,  and  the   breakpoints of CIFAR10 are   $\{0,100,200,...,20000\}$.  We then calculate their average accuracies and std values over these acquisition points. As the shown in Table~1, all std values around 0.1, yielding a norm value. Usually, an average accuracy on the same acquisition size with different random seeds of DNNs, will result a small std value.  Our mean accuracy spans across the whole learning curve.

The results show that 1) GBALD achieves the highest average accuracies; 2)$k$-medoids is ranked the second amongst the compared baselines; 3) $k$-centers has ranked the worst accuracies amongst these approaches; 4) the others, which iteratively update the training model are ranked at the middle including BALD, Var and Entropy algorithms.  
Table~2 shows the acquisition numbers of achieving the accuracies of 70\%, 80\%, and 90\% on the three datasets.
The three numbers of each cell are  the acquisition numbers over MNIST, SVHN, and CIFAR10, respectively.
The results show that GBALD can use fewer acquisitions to achieve a desired accuracy than the other algorithms.

\subsection{Active acquisitions with repeated samples}
Repeatedly collecting samples in the establishment of a database is very common. Those repeated samples may be  continuously evaluated as the primary acquisitions of AL due to the lack of one or more categories of class labels.   Meanwhile,  this situation may lead the evaluation of the model uncertainty to fall into  repeated acquisitions. To respond  this collecting situation, we compare the acquisition performance of BALD, Var, and GBALD using 5,000 and 10,000 repeated samples from the first 5,000 and 10,000 unlabeled data of SVHN, respectively. In addition, the unsupervised algorithms which do not interact with the network architecture, such as $k$-medoids and $k$-centers, have been shown that they cannot accelerate the training in terms of the experiment  results of Section~7.3. Thus, we are no longer studying their performance. The network architecture still follows the settings of   Section~7.3.

\par The acquisition results over the repeated SVHN datasets are presented in Figure~7. The batch sizes of the compared baselines are 100, where GBALD ranks 300 acquisitions to select 100 data for the training, i.e. $b=300,b'=100$.   The mean$\pm$std values of these baselines of the breakpoints (i.e. $\{0,100,200,...,10000\}$)  are reported in Table~3. The results demonstrate that GBALD shows slighter perturbations on the repeated samples than Var and BALD because it draws the core-set from the input distribution  as the initial acquisitions,  leading a small  probability to sample from one or more fixed class.
In GBALD, the informative acquisitions constrained with geometric representations further scatter the acquisitions spread in different classes. However, the Var and BALD algorithms have no particular schemes against the repeated acquisitions. The maximizer on the model uncertainty may be repeatedly produced by those repeated samples.
In additional, the unsupervised algorithms such as $k$-medoids and $k$-centers don not have these limitations, but cannot accelerate the training since there has no interactions with the network architecture. 

 \begin{figure*}
\subfloat[Var]{
\label{fig:improved_subfig_b}
\begin{minipage}[t]{0.32\textwidth}
\centering
\includegraphics[width=2.38in,height=1.78in]{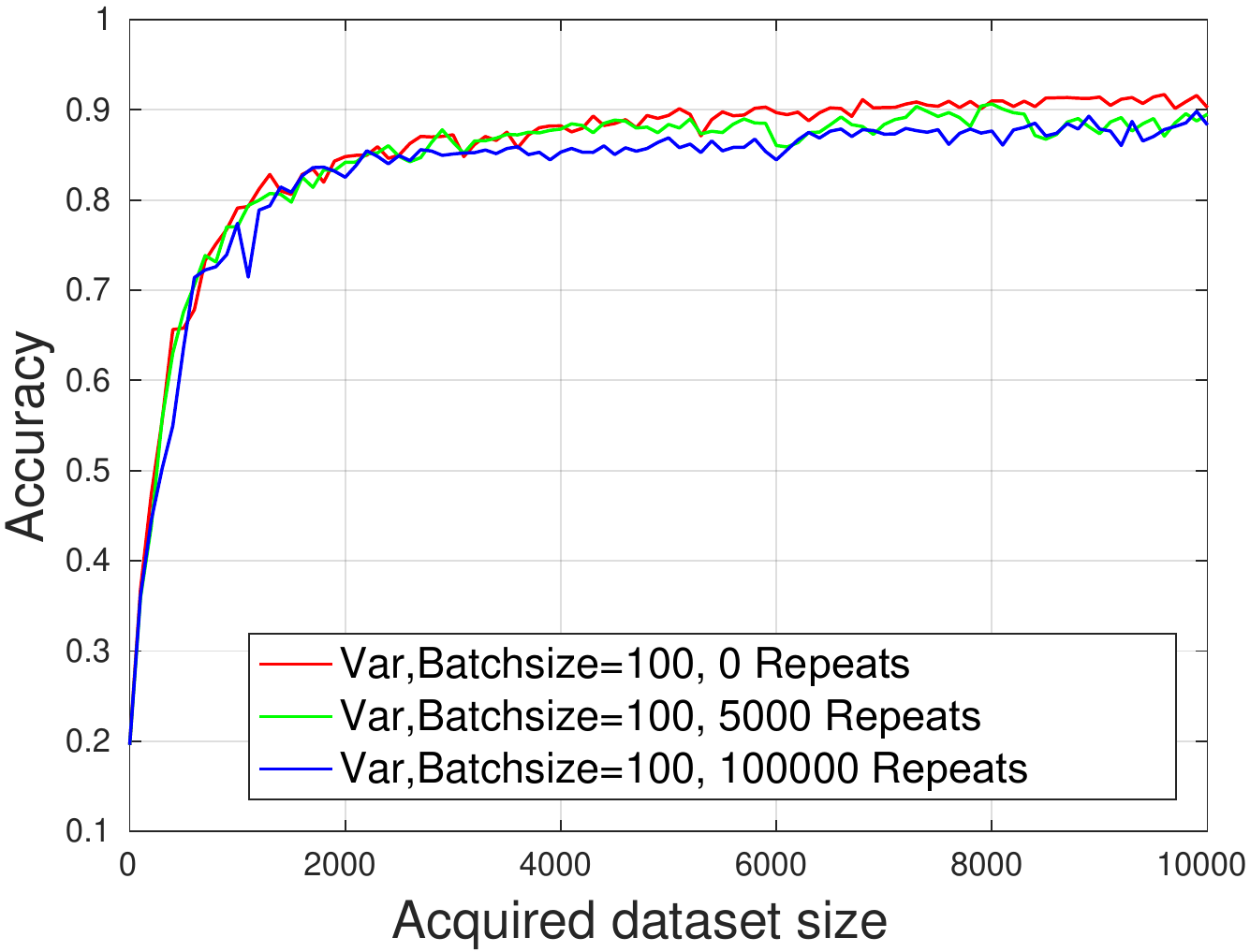}
\end{minipage}
}
\subfloat[BALD]{
\label{fig:improved_subfig_b}
\begin{minipage}[t]{0.32\textwidth}
\centering
\includegraphics[width=2.38in,height=1.78in]{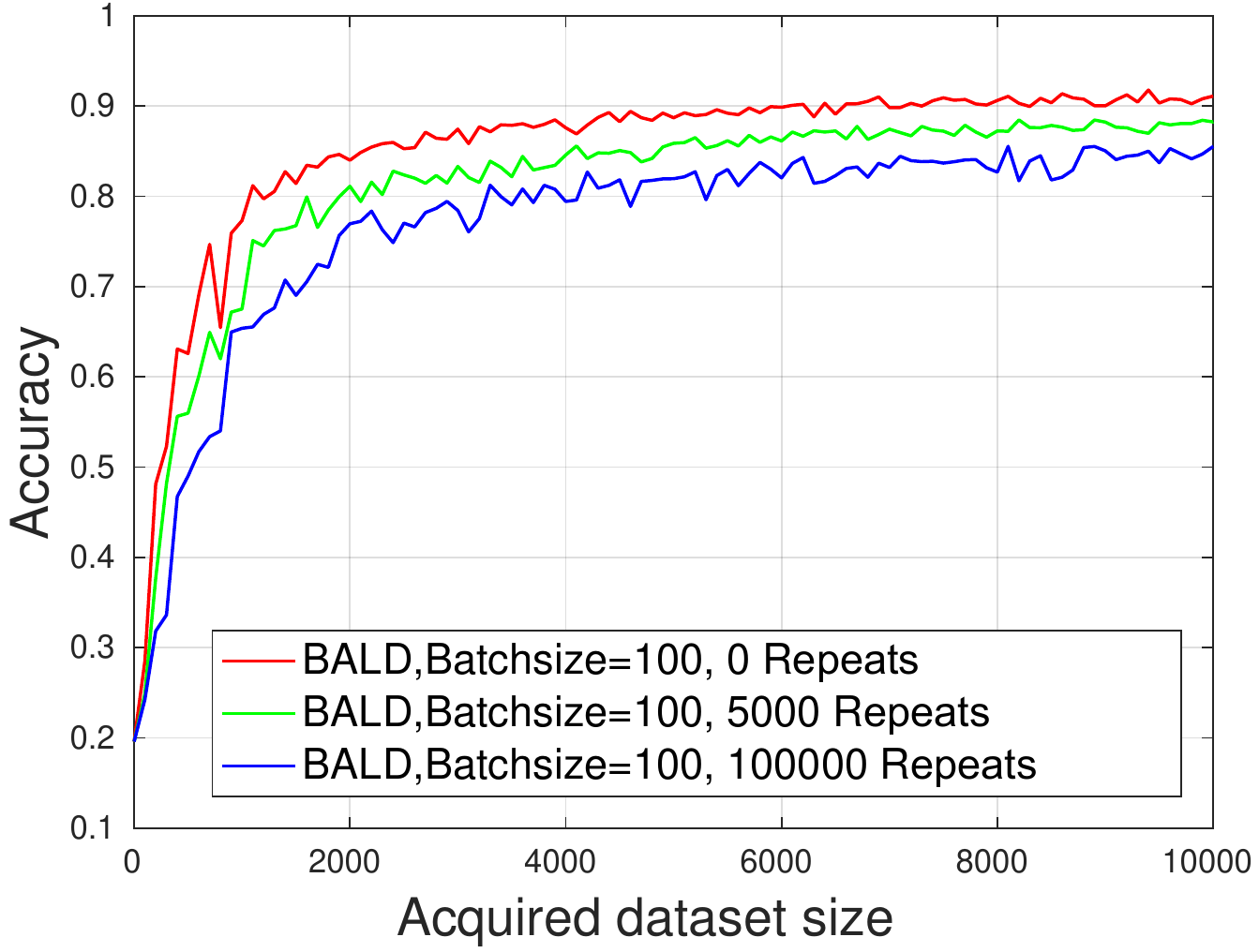}
\end{minipage}
}
\subfloat[GBALD ]{
\label{fig:improved_subfig_b}
\begin{minipage}[t]{0.32\textwidth}
\centering
\includegraphics[width=2.38in,height=1.78in]{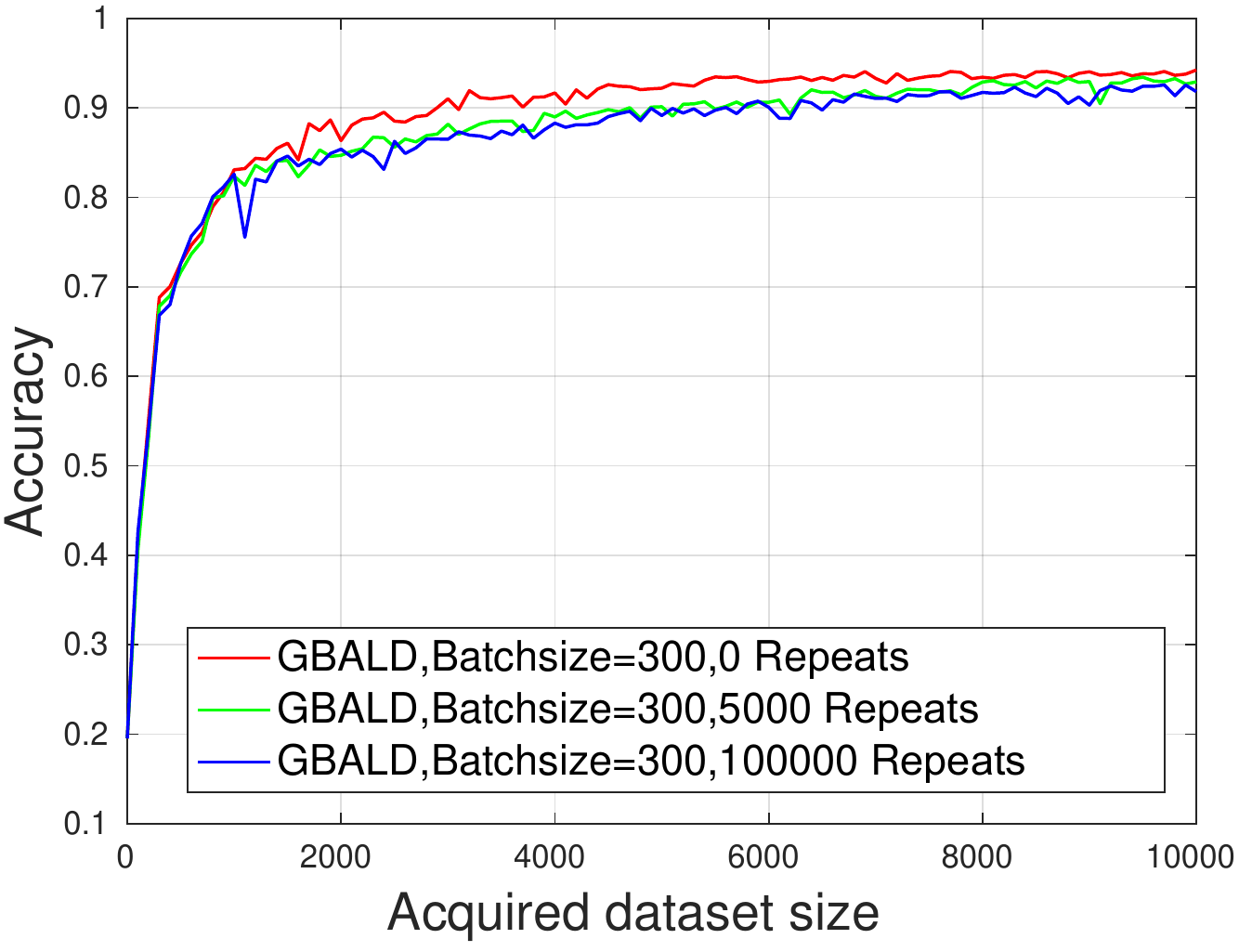}
\end{minipage}
}
\caption{Active  acquisitions on  SVHN with 5,000 and 10,000 repeated samples.
 } 
\end{figure*}

\begin{table}
 \caption{Mean$\pm$std of active  acquisitions on  SVHN with 5,000 and 10,000 repeated samples. }
\setlength{\tabcolsep}{8.0pt}
\begin{center}
\scalebox{1.0}{
 \begin{tabular}{c| c c c  } 
\hline
\multirow{2}{*}{Algorithms}  &  & \multicolumn{1}{c}{Accuracies}   \\
                                         &   0 repeats    & 5,000 repeats     &  10,000 repeats   \\

\hline
 Var          & 0.8535$\pm$0.1098   &0.8478$\pm$0.1074  &0.8281$\pm$0.1082     \\
   
BALD    &0.8510$\pm$0.1160&0.8119$\pm$0.1216 &0.7689$\pm$0.1288 \\

GBALD  &\textbf{0.8885$\pm$0.1054}&\textbf{0.8694$\pm$0.1032}&\textbf{0.8630$\pm$0.1002}     \\
\hline
\end{tabular}}
\end{center}
\end{table}

\subsection{Active  acquisitions with noisy samples}
 Noisy labels \cite{golovin2010near,han2018co} are inevitable due to human errors in data annotation.  Training on noisy labels, the neural network  model will degenerate its inherent properties.  To assess the perturbations of the above acquisition algorithms against noisy labels, we organize the following experiment scenarios: we select  the first 5,000 and 10,000 samples respectively   from the unlabeled data pool of the MNIST dataset 
 and  reset their labels   by shifting  $\{$`0',`1',...,`8'$\}$ to $\{$`1',`2',...,`9'$\}$, respectively.   The network architecture follows the MLP of Section~7.3. The selected baselines are Var and BALD.

 \par Figure~7 presents the acquisition results of those baseline with noisy labels.  The batch sizes of the compared baselines are 100, where GBALD ranks 300 acquisitions to select 100 data for the training, i.e. $b=300,b'=100$.  Table~4 presents the mean$\pm$std values of the breakpoints (i.e. $\{0,100,200,...,10000 \}$) over learning curves of Figure~7. The results further show that GBALD has  smaller   noisy perturbations than the  other baselines. 
 For Var and BALD,    model uncertainty leads    high probabilities to sample those noisy data due to their   greatly updating on the   model.

  \begin{figure*} 
\subfloat[Var ]{
\label{fig:improved_subfig_b}
\begin{minipage}[t]{0.32\textwidth}
\centering
\includegraphics[width=2.38in,height=1.78in]{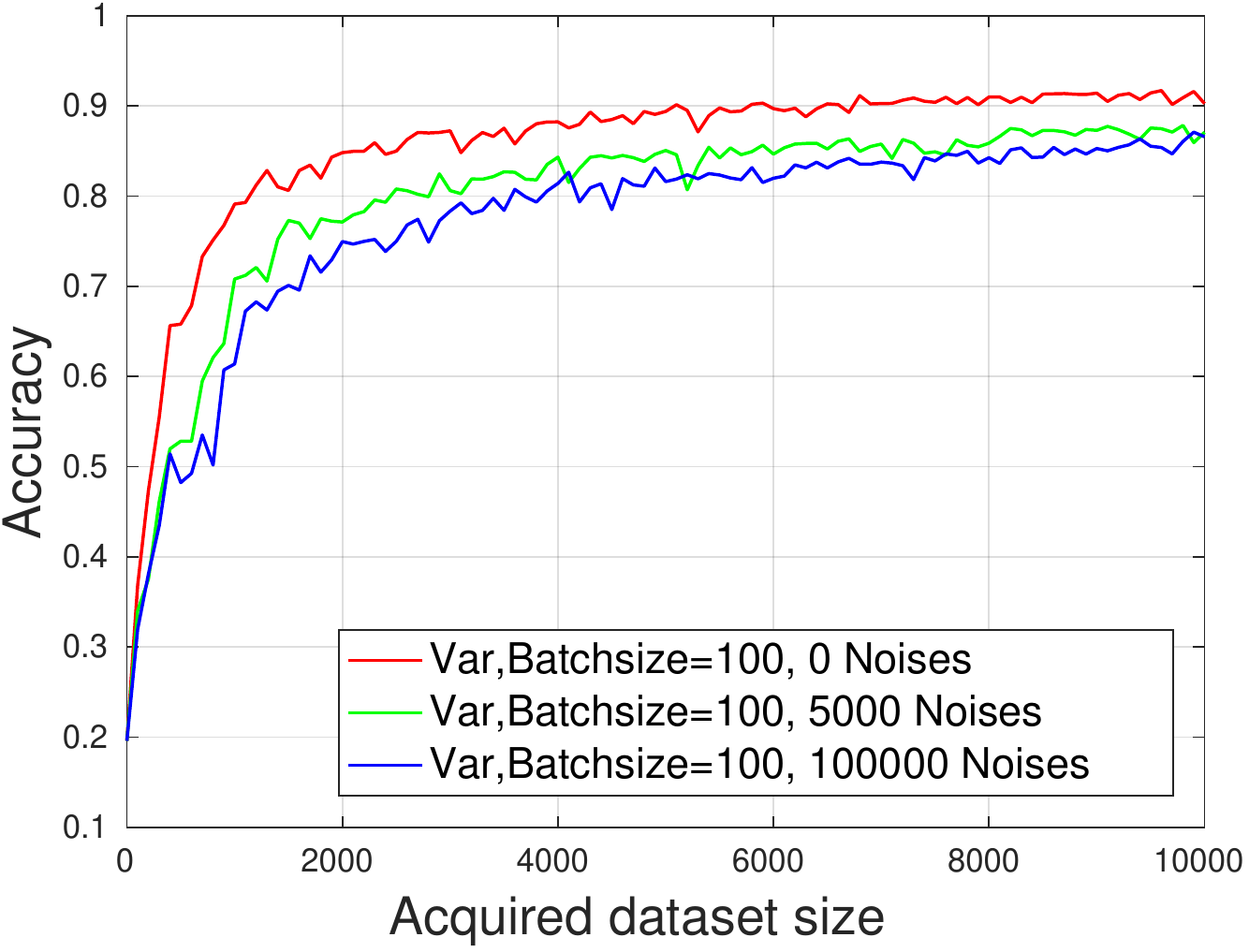}
\end{minipage}
}
\subfloat[BALD]{
\label{fig:improved_subfig_b}
\begin{minipage}[t]{0.32\textwidth}
\centering
\includegraphics[width=2.38in,height=1.78in]{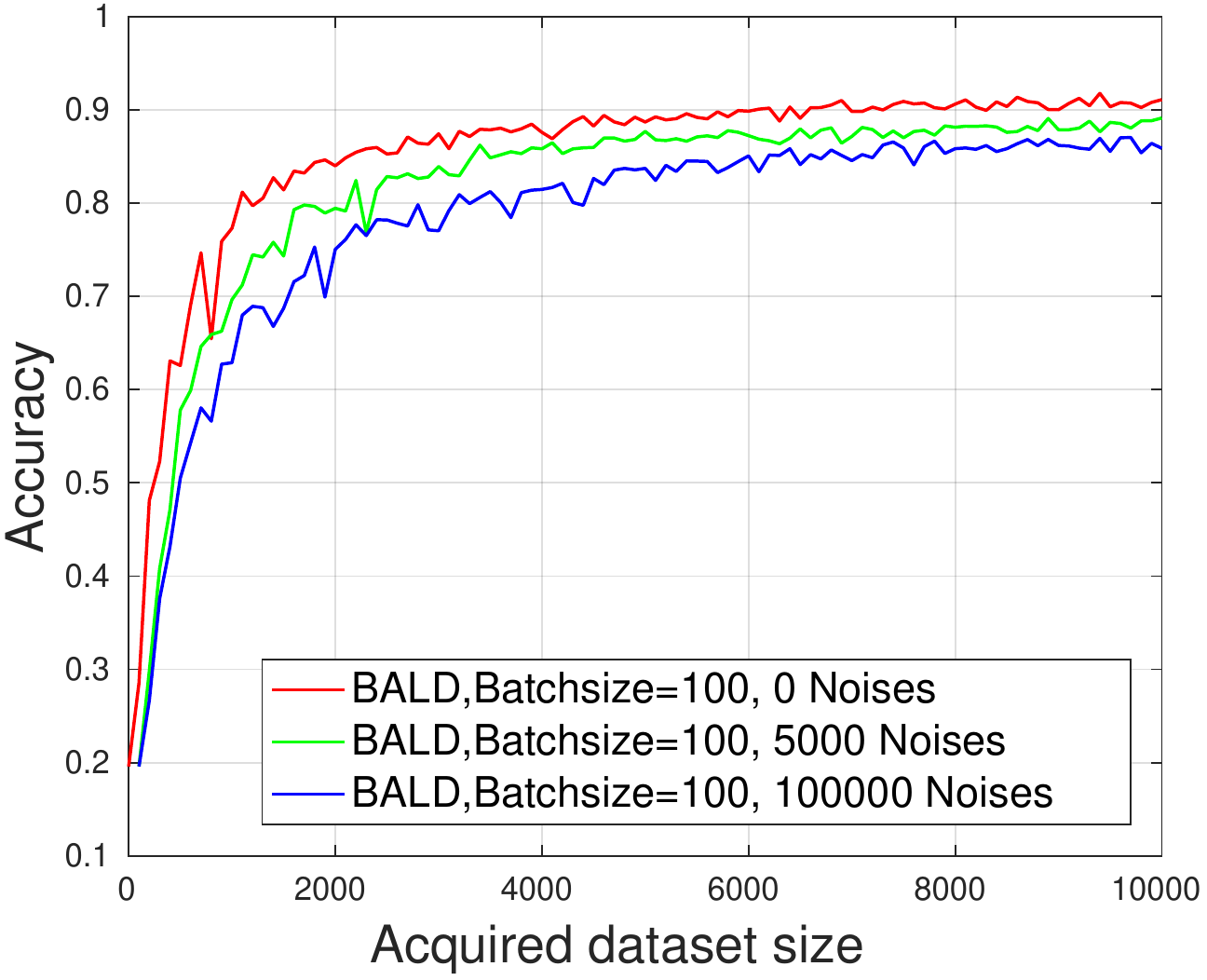}
\end{minipage}
}
\subfloat[GBALD]{
\label{fig:improved_subfig_b}
\begin{minipage}[t]{0.32\textwidth}
\centering
\includegraphics[width=2.38in,height=1.78in]{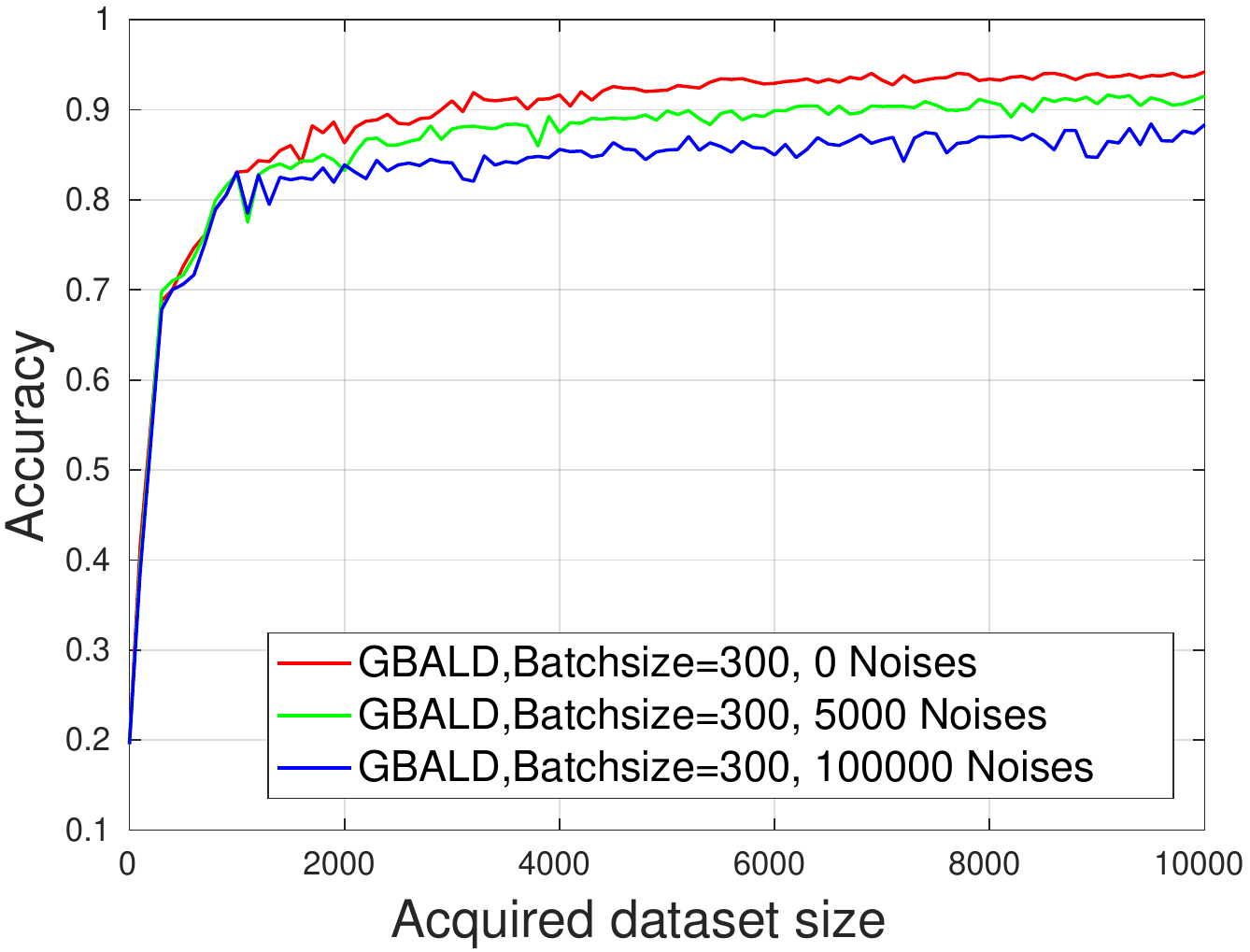}
\end{minipage}
}
\caption{Active noisy acquisitions on SVHN with  5,000 and 10,000 noisy labels.
 } 
\end{figure*}

\begin{table} 
 \caption{Mean$\pm$std of active noisy  acquisitions on  SVHN with 5,000 and 10,000 noises. }
\setlength{\tabcolsep}{8pt}
\begin{center}
\scalebox{1.0}{
 \begin{tabular}{c| c c c  } 
\hline
\multirow{2}{*}{Algorithms}  &  & \multicolumn{1}{c}{Accuracies}   \\
                                         &   0 noises & 5,000  noises   &  10,000 noises    \\
  \hline

 Var          & 0.8535$\pm$0.1098   &0.7980$\pm$0.1203 &0.7702$\pm$0.1238   \\
   
BALD    &0.8510$\pm$0.1160&0.8205$\pm$0.1185 &0.7849$\pm$0.1239 \\

GBALD  &\textbf{0.8885$\pm$0.1054}&\textbf{0.8622$\pm$0.0991}&\textbf{0.8301$\pm$0.0916}     \\
\hline
\end{tabular}}
\end{center}
\end{table}

\subsection{GBALD vs.  BatchBALD}
Batch deep AL was recently  proposed to accelerate the training of a DNN model. In recent literature, BatchBALD \cite{kirsch2019batchbald} extended  BALD with a batch acquisition setting to converge the network 
using fewer iteration loops. Different to   BALD,  BathBALD introduces the diversity    to avoid the  repeated or similar     acquisitions.

How to set the batch size of the acquisitions attracted our eyes before  starting the experiments. It involves with whether our experiment settings are fair and reasonable.  From a theoretical view, the larger the batch size, the worse the batch acquisitions will be.  Experimental results of \cite{kirsch2019batchbald} also demonstrated this phenomenon.  We thus set different batch sizes to run BatchBALD. Figure~8 presents the comparison results of BALD, BatchBALD, and our proposed GBALD as with the experiment settings of Section~7.3. As the shown in this figure,  BatchBALD degenerates the test accuracies if we progressively increase the bath sizes, where BatchBALD with a batch size of 10  keeps similar learning curves as BALD. 
It shows that   BatchBALD actually can accelerate BALD with a similar acquisition result if the batch size is not large.  This means, if the batch size is  between 2 to 10, BatchBALD will degenerate into BALD and maintains highly-consistent results. 

\begin{figure}
\centering
\includegraphics[scale=0.62]{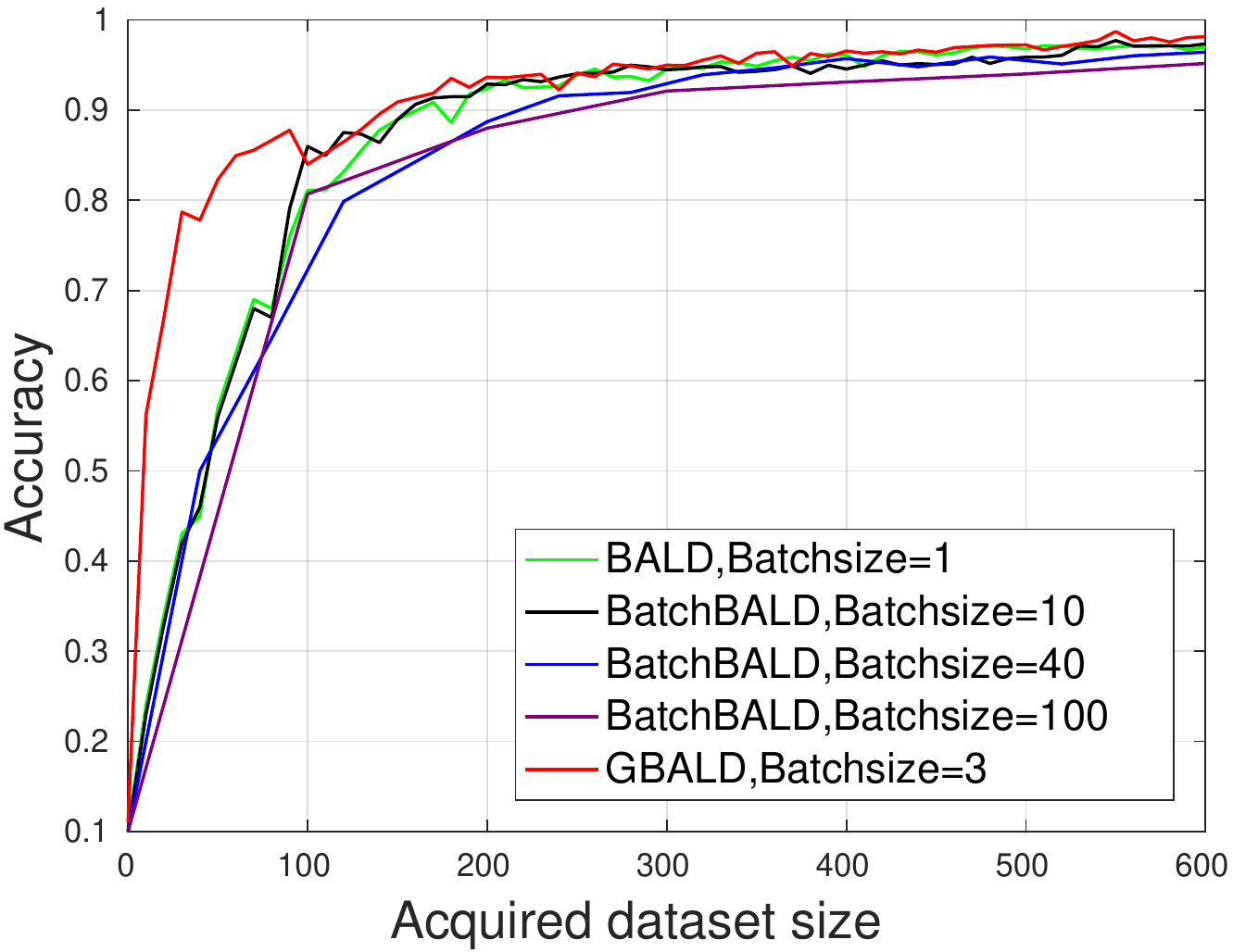}
\caption{Comparisons of BALD, BatchBALD, and GBALD of active   acquisitions on MNIST with bath settings.}
\end{figure}

\par Also because of this, BatchBALD also  has the same sensitivity  to an uninformative   prior. For our GBALD, the core-set  solicits  sufficient data which properly  matches the input distribution (w.r.t. acquired data set size $\leq$ 100), providing expressive input features to start the DNN model (w.r.t. acquired data set size $>$ 100).  Table~5 then presents the  mean$\pm$std of the breakpoints ($\{0,10,20,...,600 \}$) of active  acquisitions on  MNIST with batch settings.  The statistical results show that GBALD has much higher mean accuracy than BatchBALD with different bath sizes. Therefore, evaluating the model uncertainty of DNN using those highly-representative core-set samples can improve the performance of  the neural network.

\begin{table}
 \caption{{Mean$\pm$std of  BALD, BatchBALD, and GBALD of active    acquisitions on  MNIST with batch settings.}}
\setlength{\tabcolsep}{20.5pt}
\begin{center}
\scalebox{1.0}{
 \begin{tabular}{c| c c  } 
\hline
Algorithms  &  {Batch sizes}  & Accuracies \\
\hline
BALD  & 1&   0.8654$\pm$0.0354 \\
BatchBALD  & 10& 0.8645$\pm$0.0365 \\
BatchBALD  & 40& 0.8273$\pm$0.0545\\
BatchBALD  & 100&  0.7902$\pm$0.0951\\
GBALD  & 3& 0.9106$\pm$0.1296\\
\hline
\end{tabular}}
\end{center}
\end{table}

\subsection{Acceleration of accuracy}
Accelerations of accuracy  i.e. the first-orders of the   breakpoints of the learning curve, describe the efficiency of the active acquisition loops. Different to the accuracy curves, the acceleration curve  reflects how active acquisitions help the convergence of  the interacting DNN model.

We thus  firstly present the acceleration curves of different baselines  on MNIST, SVHN, and CIFAR10 datasets as with the experiments of Section~7.3.  The acceleration curves of active acquisitions   are drawn in Figure~9. Observing those acceleration curves of different algorithms clearly finds that,  GBALD always keeps \textbf{higher accelerations of accuracy} than the other baselines against the three benchmark datasets. This revels the reason of why GBALD can derive more informative and representative data to maximally update the DNN model.

The acceleration curves  of active acquisitions with repeated samples  are presented in Figure~10.  As the shown in this  figure, GBALD presents \textbf{slighter perturbations} to the number of repeated samples than that of Var and BALD due to its effective ranking scheme on optimizing model uncertainty    of DNNs.  
The acceleration curves of active  noisy acquisitions   are drawn in Figure~11. 
Compared to Figure~7, it presents  more intuitive descriptions for the noisy perturbations to different baselines. With horizontal comparisons to the  acceleration curves of Var and BALD, our proposed GBALD has smaller noisy perturbations due to 1) the powerful core-set  which properly captures the input distribution,  and 2) both the highly representative and informative   acquisitions of the model uncertainty.

 \begin{figure*}[!htbp] 
\subfloat[MNIST]{
\label{fig:improved_subfig_b}
\begin{minipage}[t]{0.32\textwidth}
\centering
\includegraphics[width=2.38in,height=1.78in]{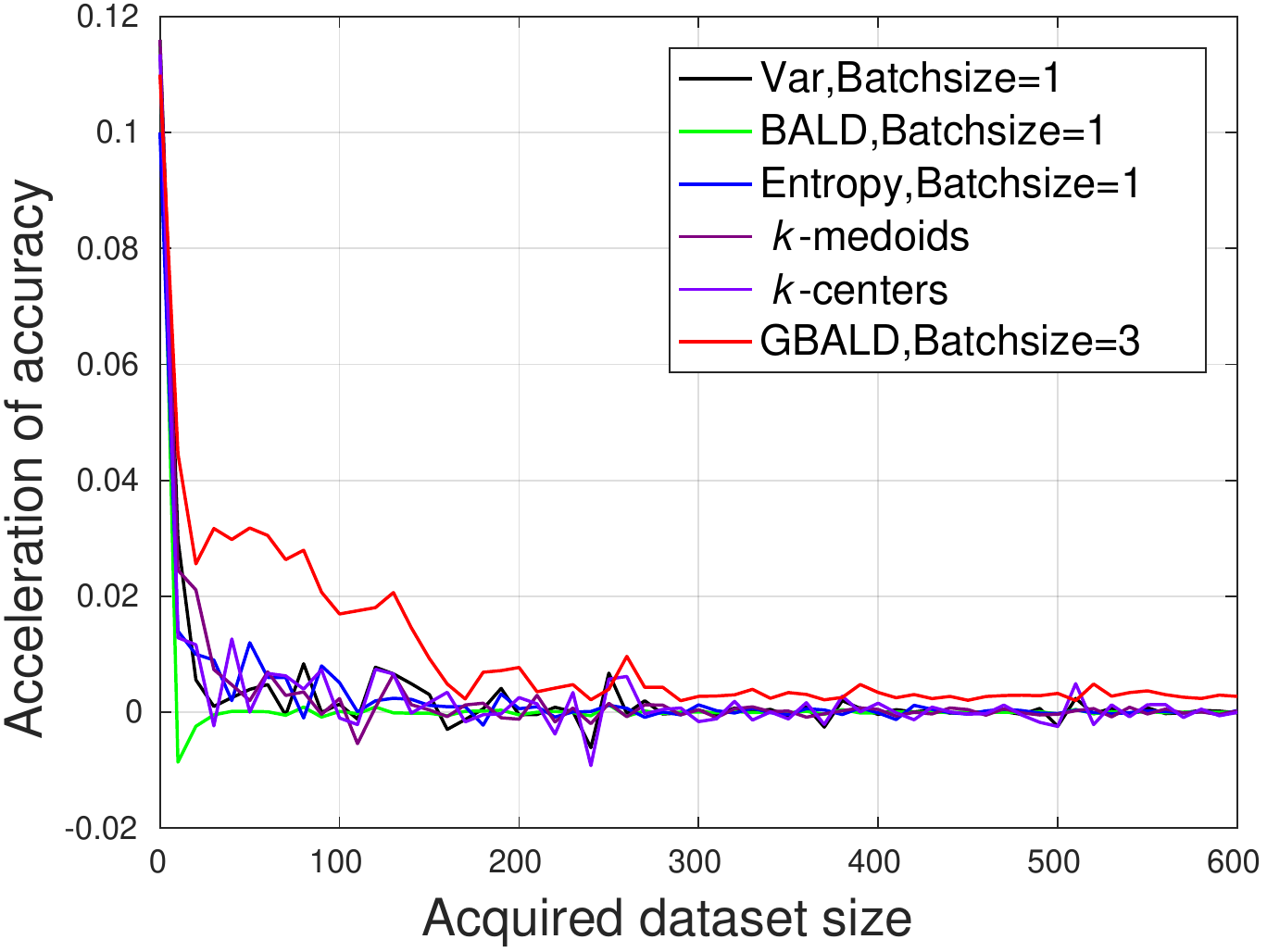}
\end{minipage}
}
\subfloat[SVHN]{
\label{fig:improved_subfig_b}
\begin{minipage}[t]{0.32\textwidth}
\centering
\includegraphics[width=2.38in,height=1.78in]{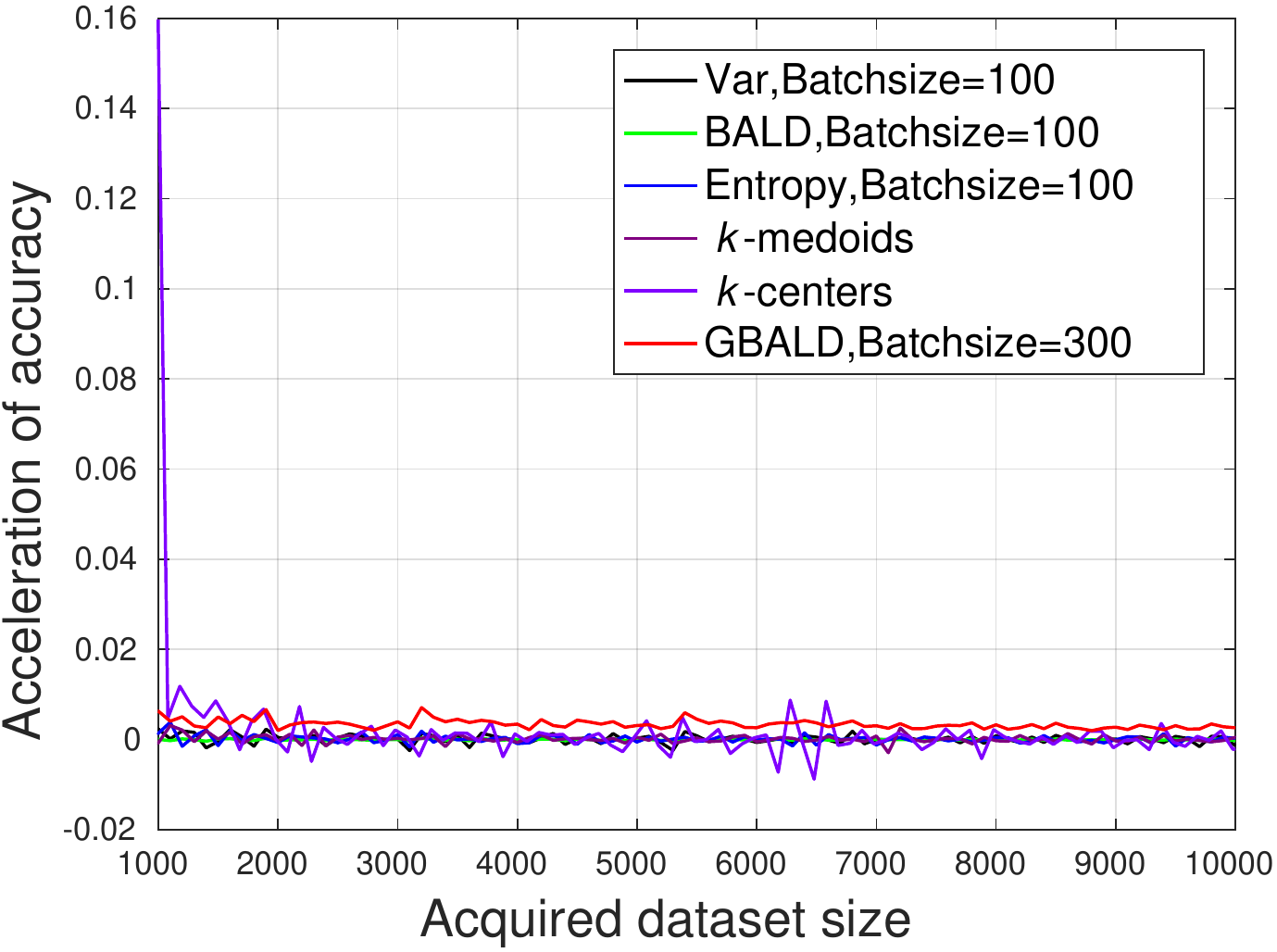}
\end{minipage}
}
\subfloat[CIFAR10]{
\label{fig:improved_subfig_b}
\begin{minipage}[t]{0.32\textwidth}
\centering
\includegraphics[width=2.38in,height=1.78in]{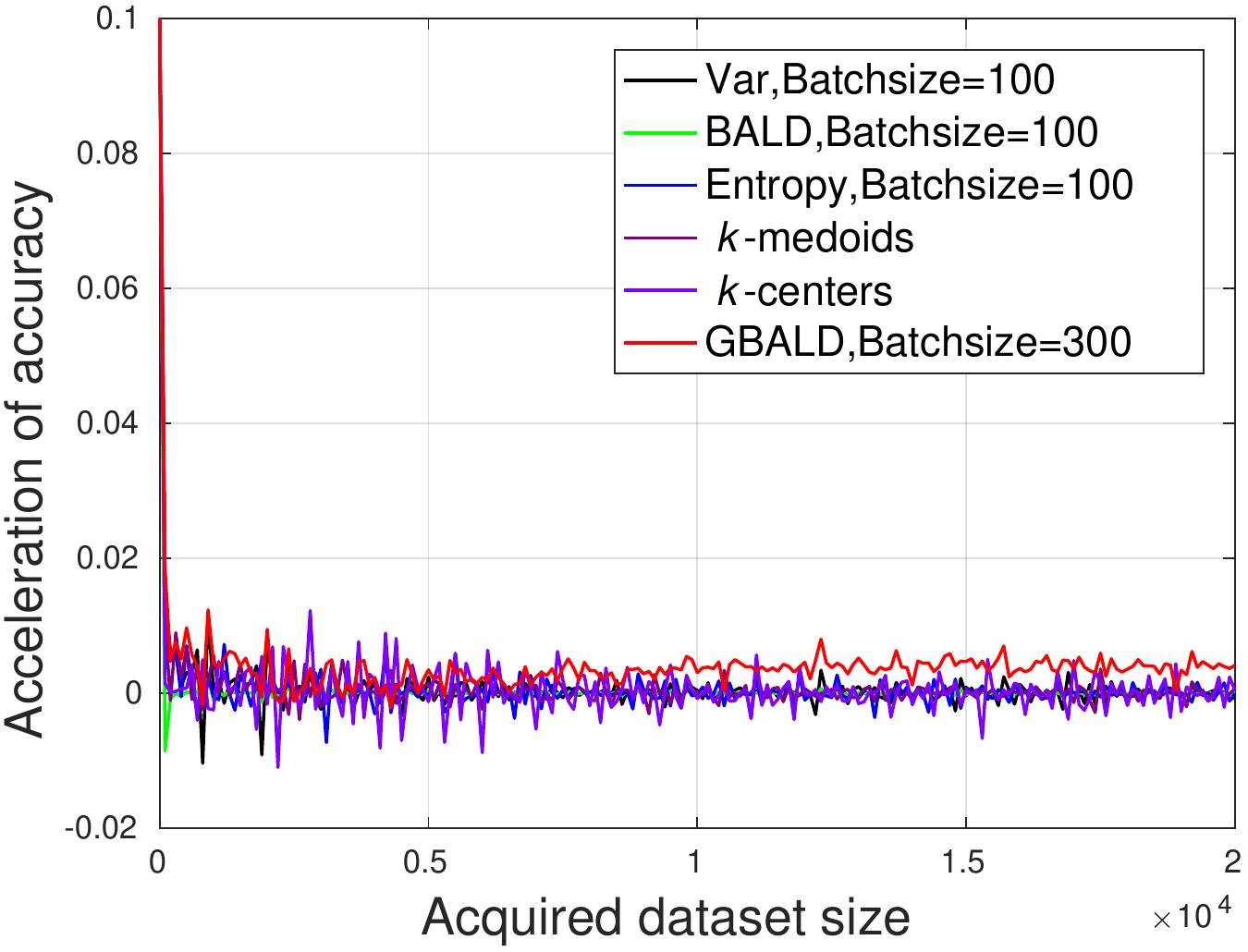}
\end{minipage}
}
\caption{Accelerations of accuracy of different baselines  on MNIST, SVHN, and CIFAR10 datasets. 
 } 
\end{figure*}

 \begin{figure*}[!htbp]  
\subfloat[Var]{
\label{fig:improved_subfig_b}
\begin{minipage}[t]{0.32\textwidth}
\centering
\includegraphics[width=2.38in,height=1.78in]{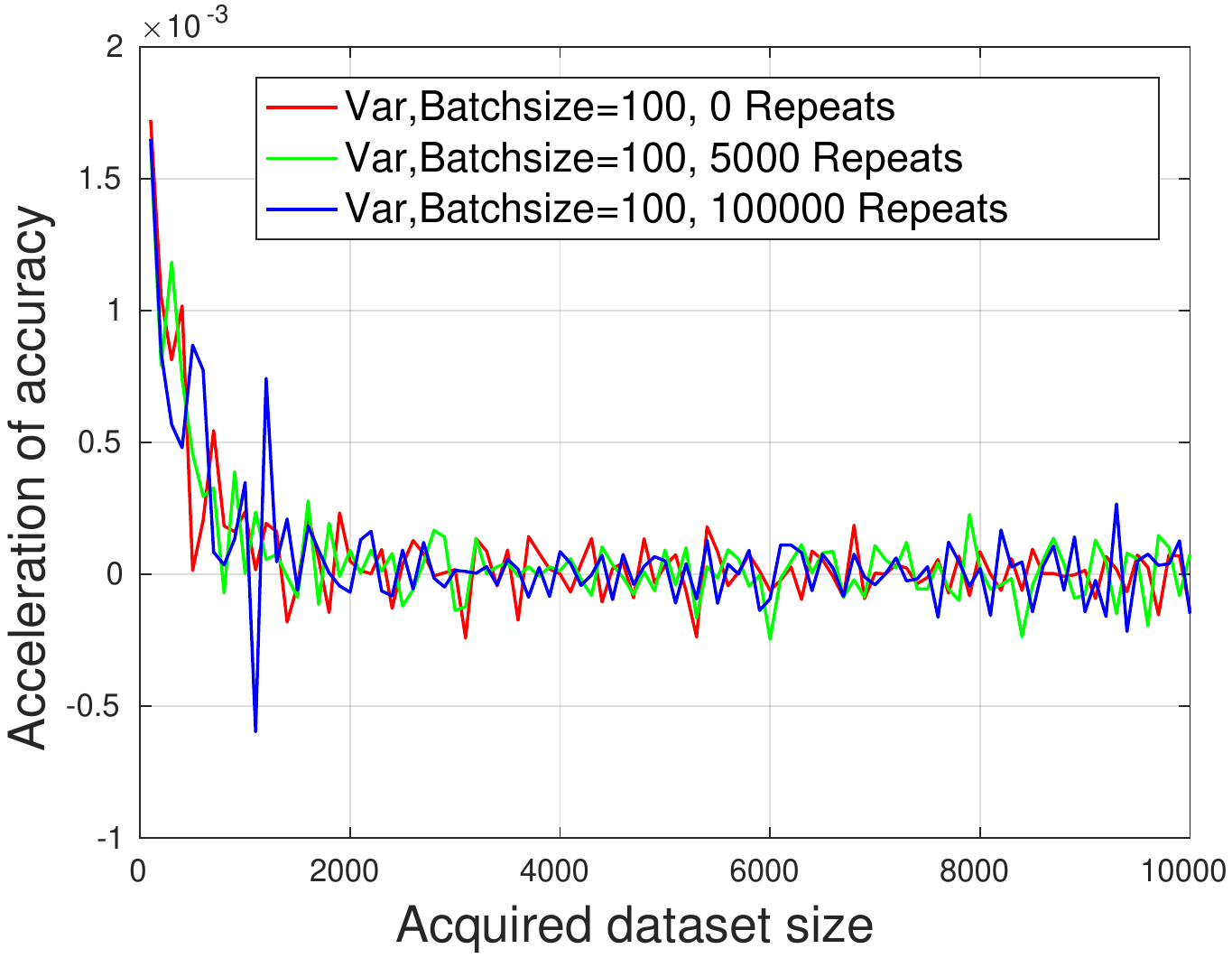}
\end{minipage}
}
\subfloat[BALD]{
\label{fig:improved_subfig_b}
\begin{minipage}[t]{0.32\textwidth}
\centering
\includegraphics[width=2.38in,height=1.78in]{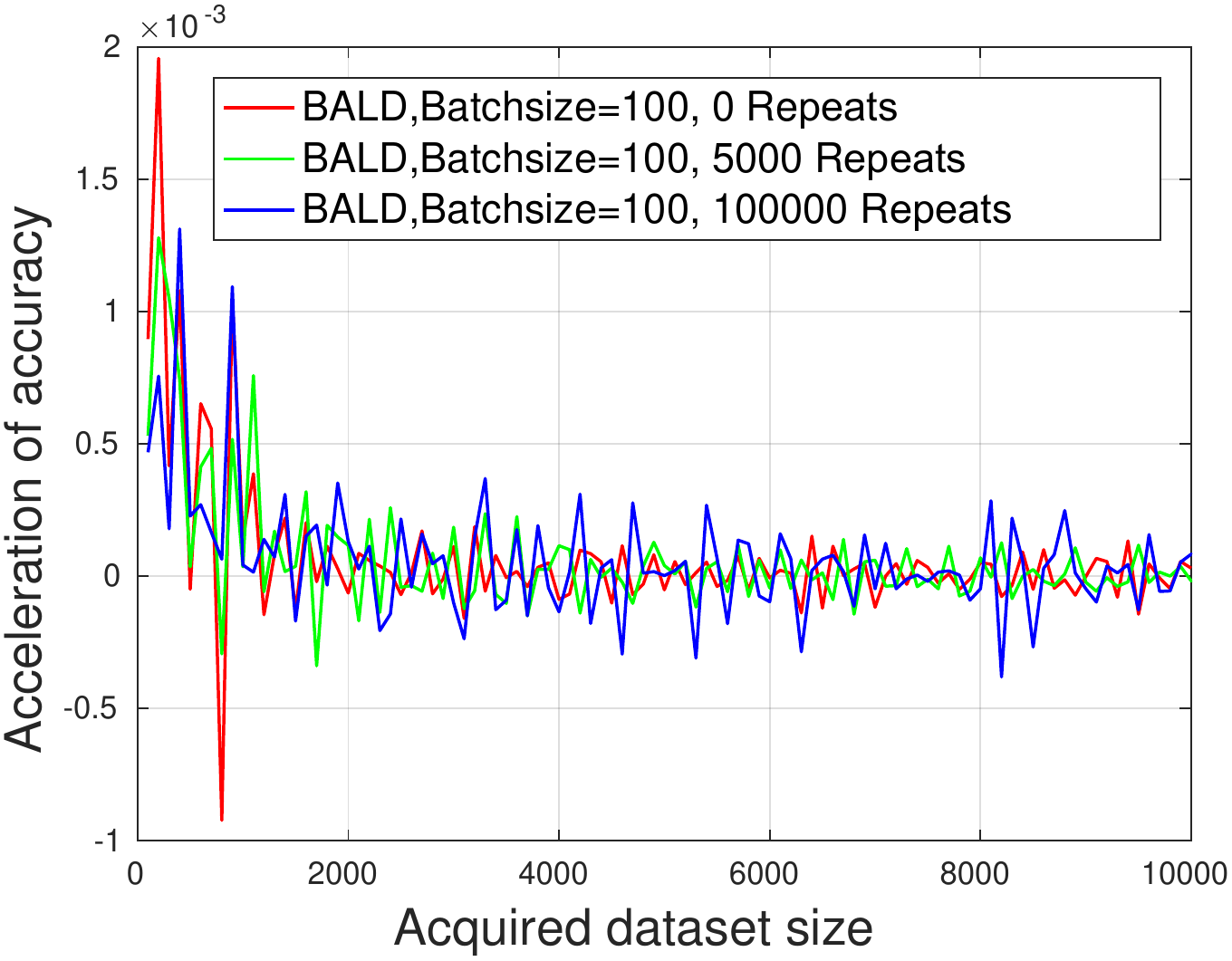}
\end{minipage}
}
\subfloat[GBALD ]{
\label{fig:improved_subfig_b}
\begin{minipage}[t]{0.32\textwidth}
\centering
\includegraphics[width=2.38in,height=1.78in]{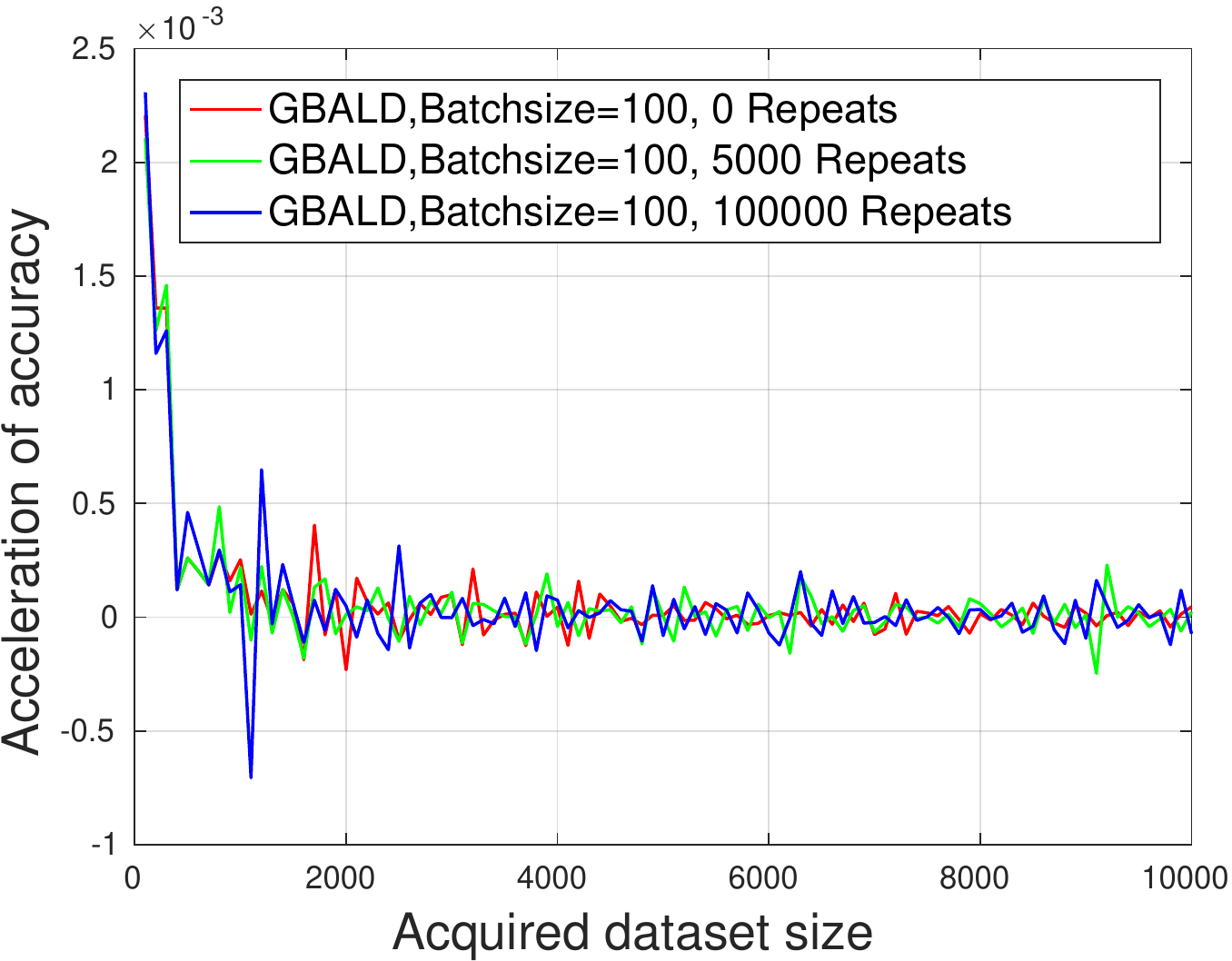}
\end{minipage}
}
\caption{Accelerations of accuracy of active  acquisitions on  SVHN with 5,000 and 10,000 repeated samples.}
 
\end{figure*}
  \begin{figure*}[!htbp] 
\subfloat[Var ]{
\label{fig:improved_subfig_b}
\begin{minipage}[t]{0.32\textwidth}
\centering
\includegraphics[width=2.38in,height=1.78in]{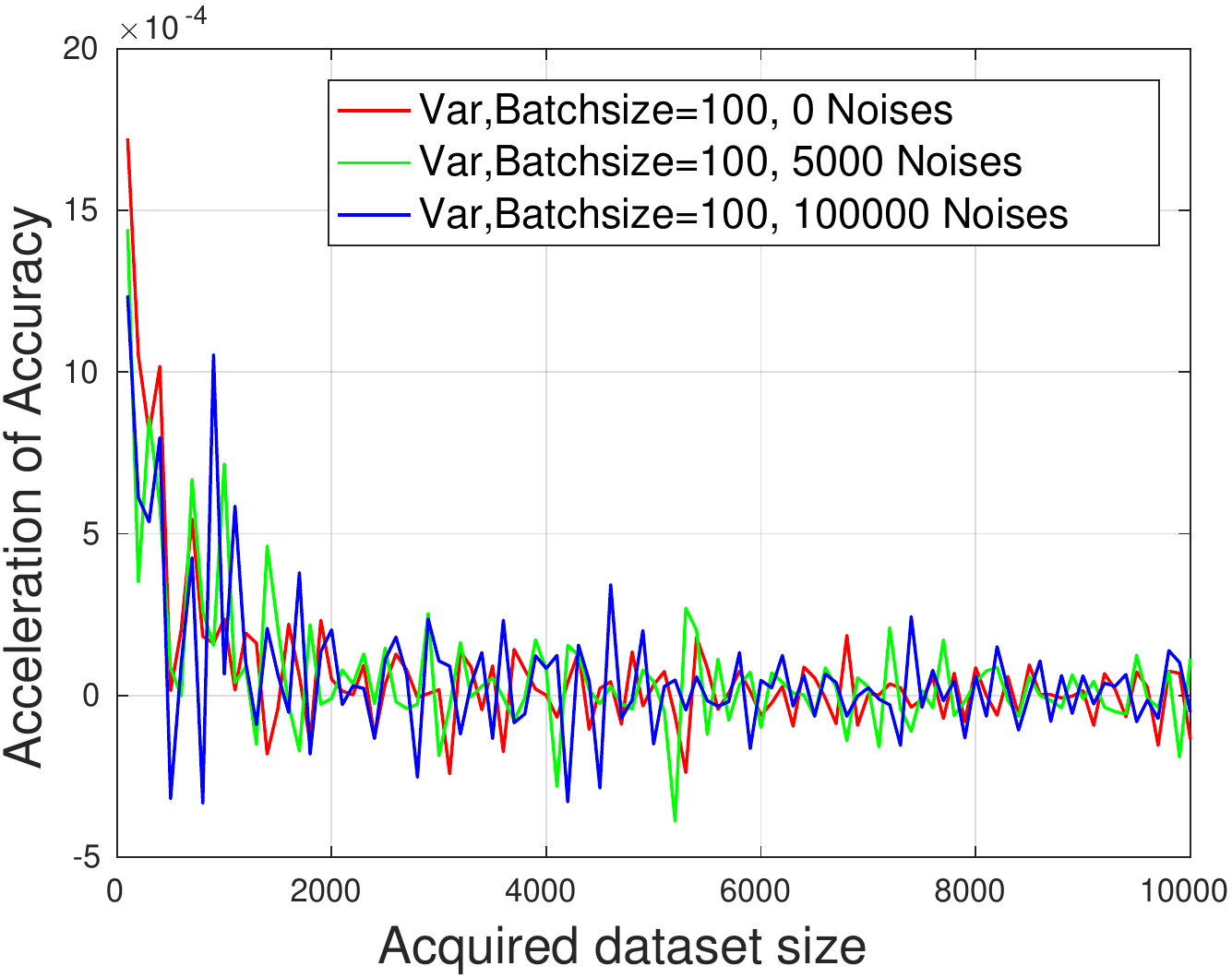}
\end{minipage}
}
\subfloat[BALD]{
\label{fig:improved_subfig_b}
\begin{minipage}[t]{0.32\textwidth}
\centering
\includegraphics[width=2.38in,height=1.78in]{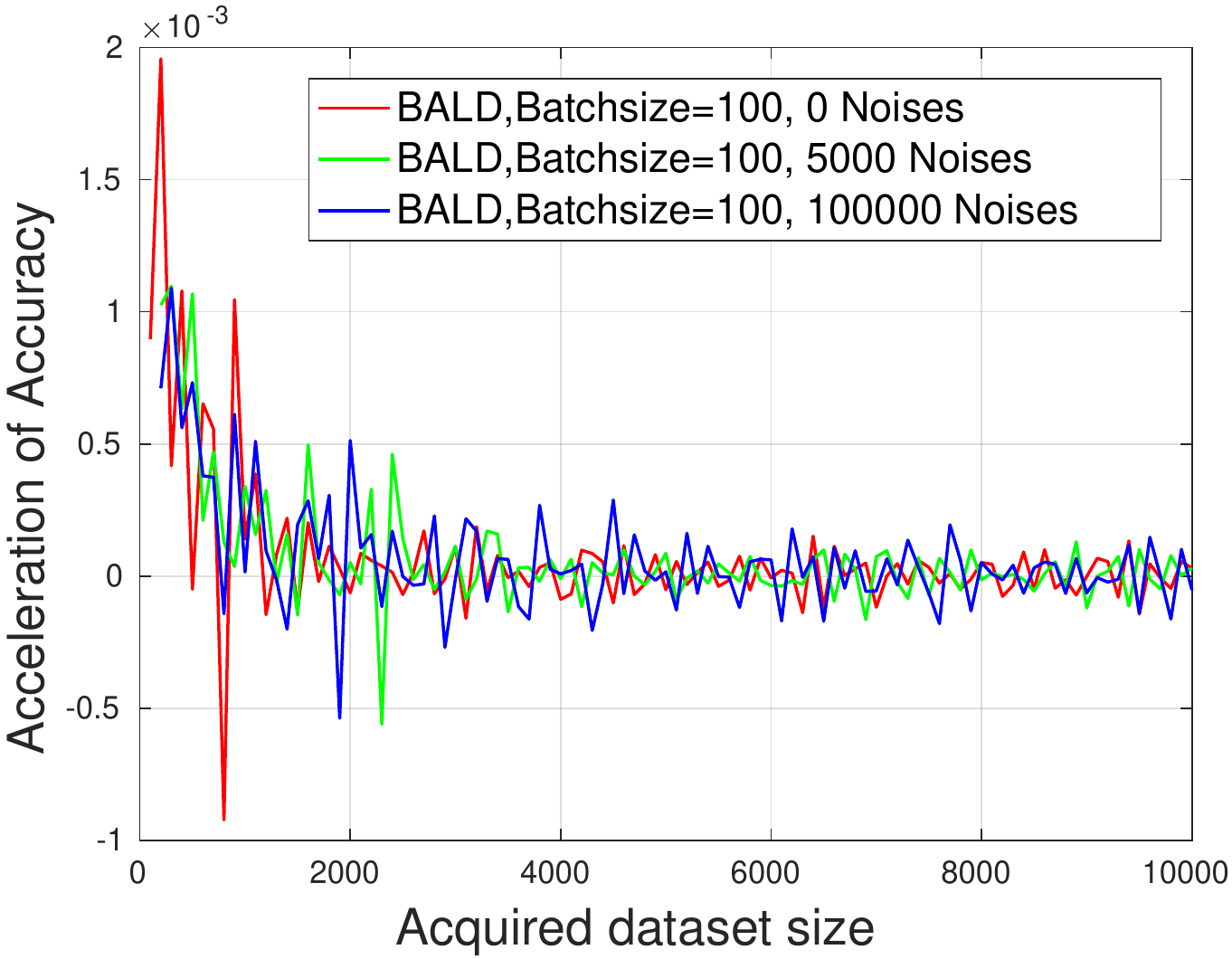}
\end{minipage}
}
\subfloat[GBALD]{
\label{fig:improved_subfig_b}
\begin{minipage}[t]{0.32\textwidth}
\centering
\includegraphics[width=2.38in,height=1.78in]{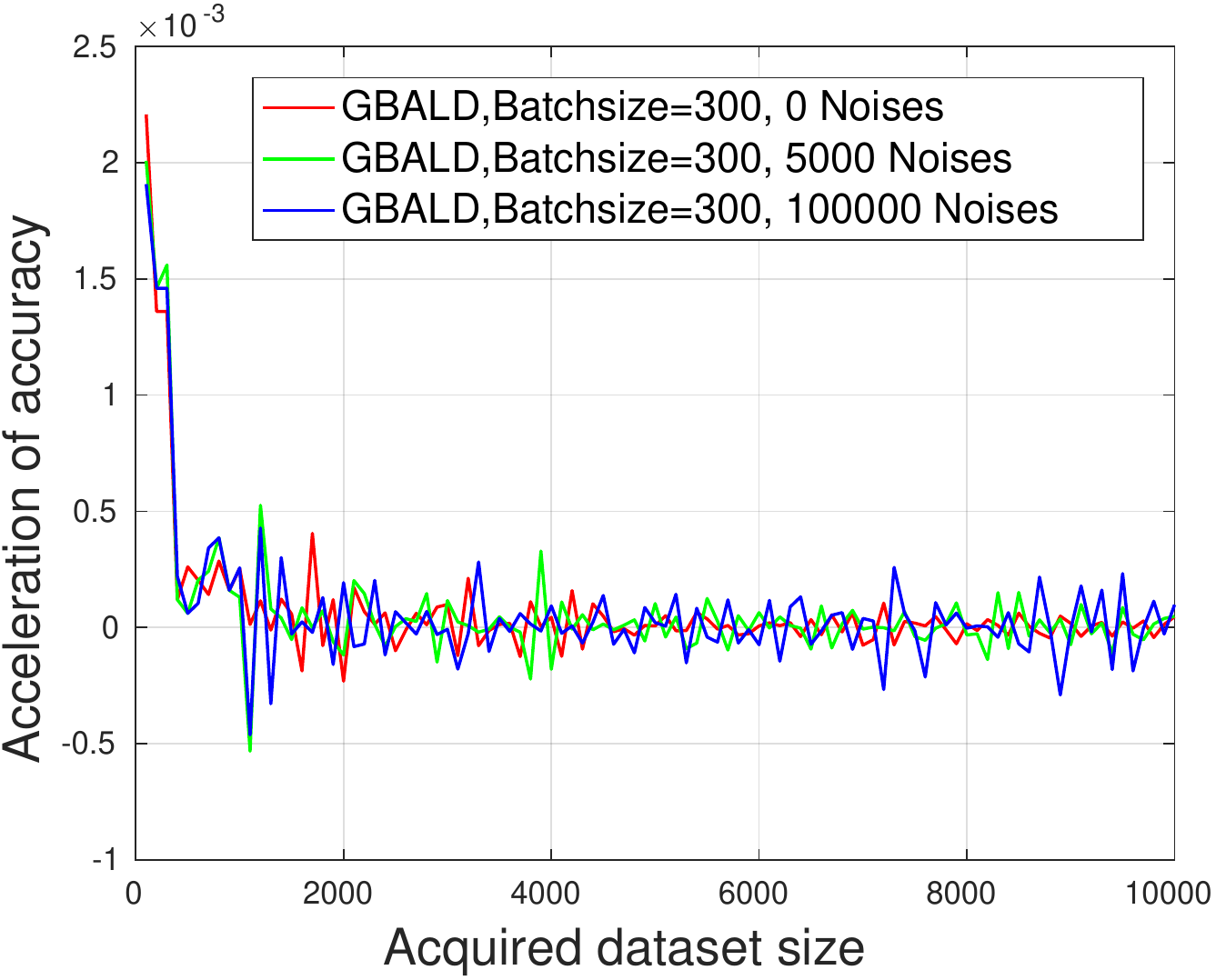}
\end{minipage}
}
\caption{Accelerations of accuracy of  active noisy acquisitions on   SVHN with  5,000 and 10,000 noisy labels.
 } 
\end{figure*}

\subsection{Hyperparameter settings}
What is the proper time to start active acquisitions using Eq.~(14) in GBALD framework? Does the  size  ratio of  the core-set and model uncertainty acquisitions affect the performance of GBALD?

\textbf{We discuss the key hyperparameter of GBALD here: the core-set size $N_\mathcal{M}$.} Table~6 presents the    relationship of    accuracies and    size of  core-set, where  start accuracy denotes the test accuracy over the initial core-set, and   ultimate accuracy denotes the test accuracy over up to $Q=20,000$ training data.   Let $b=1000, b'=500$  in GBALD,  $\mathcal{N}_\mathcal{M}$ be the number of the core-set size,   the iteration  budget  $\mathcal{A}$ of GBALD can then  be defined as   $\mathcal{A}=(Q-\mathcal{N}_\mathcal{M})/b'$. For example, if the number of the initial core-set labels are set as $\mathcal{N}_\mathcal{M}=1,000$, we have   $\mathcal{A}=(Q-\mathcal{N}_\mathcal{M})/b'\approx 38$;  if   $\mathcal{N}_\mathcal{M}=2,000$, then   $\mathcal{A}=(Q-\mathcal{N}_\mathcal{M})/b'\approx 36$. 

 From Table~6, the GBALD algorithm keeps   stable accuracies over the start,  ultimate, and  mean$\pm$std accuracies when there  inputs more than  1,000 core-set  labels. Therefore, drawing sufficient   core-set labels using Eq.~(10) to start the model uncertainty   of Eq.~(14) can maximize the performance of our GBALD framework. 

\begin{table}
 \caption{Relationship of  accuracies and   sizes of core-set  on SVHN. }
\setlength{\tabcolsep}{3.5pt}
\begin{center}
\scalebox{1.0}{
 \begin{tabular}{c| ccc } 
\hline
\multirow{2}{*}{ Size of core-set    }  &  & \multicolumn{1}{c}{Accuracies}   \\
                                         & Start  accuracy & Ultimate  accuracy&Mean$\pm$std accuracy\\
  \hline

 $N_\mathcal{M}=$ 1,000      &  0.8790&0.9344 &0.9134$\pm$0.0169 \\
   
 $N_\mathcal{M}=$ 2,000   & 0.8898   &0.9212 & 0.9151$\pm$0.0148     \\

$N_\mathcal{M}=$ 3,000    &0.8848&\textbf{0.9364}&0.9173$\pm$0.0138   \\

$N_\mathcal{M}=$ 4,000   &0.8811&0.9271&0.9146$\pm$0.0165    \\

$N_\mathcal{M}=$ 5,000    &\textbf{0.8959}&0.9342& \textbf{0.9197$\pm$0.0117}     \\
 
\hline
\end{tabular}}
\end{center} 
\end{table}

  \begin{figure*}[!htbp] 
\subfloat[MNIST ]{
\label{fig:improved_subfig_b}
\begin{minipage}[t]{0.32\textwidth}
\centering
\includegraphics[width=2.38in,height=1.78in]{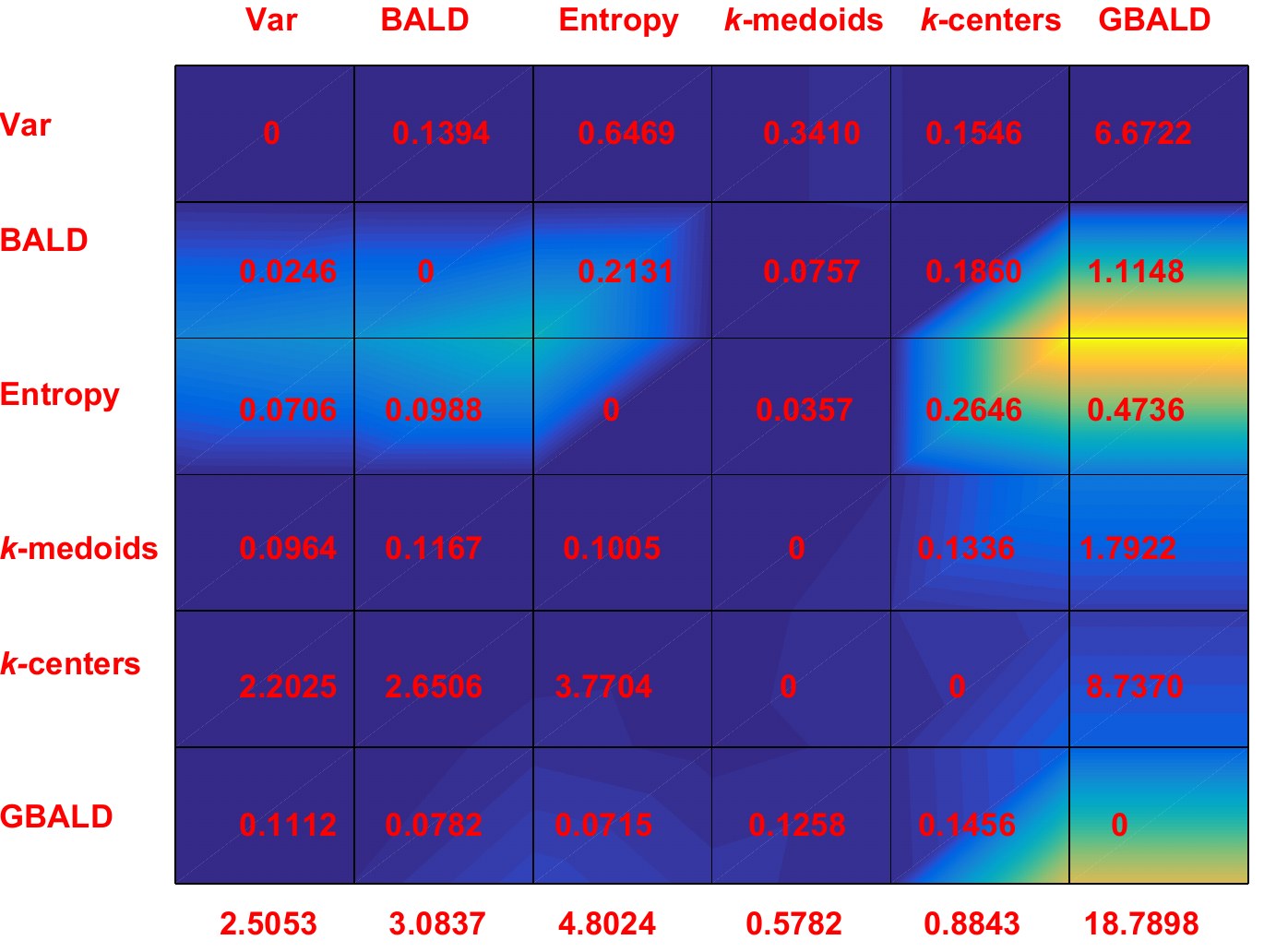}
\end{minipage}
}
\subfloat[SVHN]{
\label{fig:improved_subfig_b}
\begin{minipage}[t]{0.32\textwidth}
\centering
\includegraphics[width=2.38in,height=1.78in]{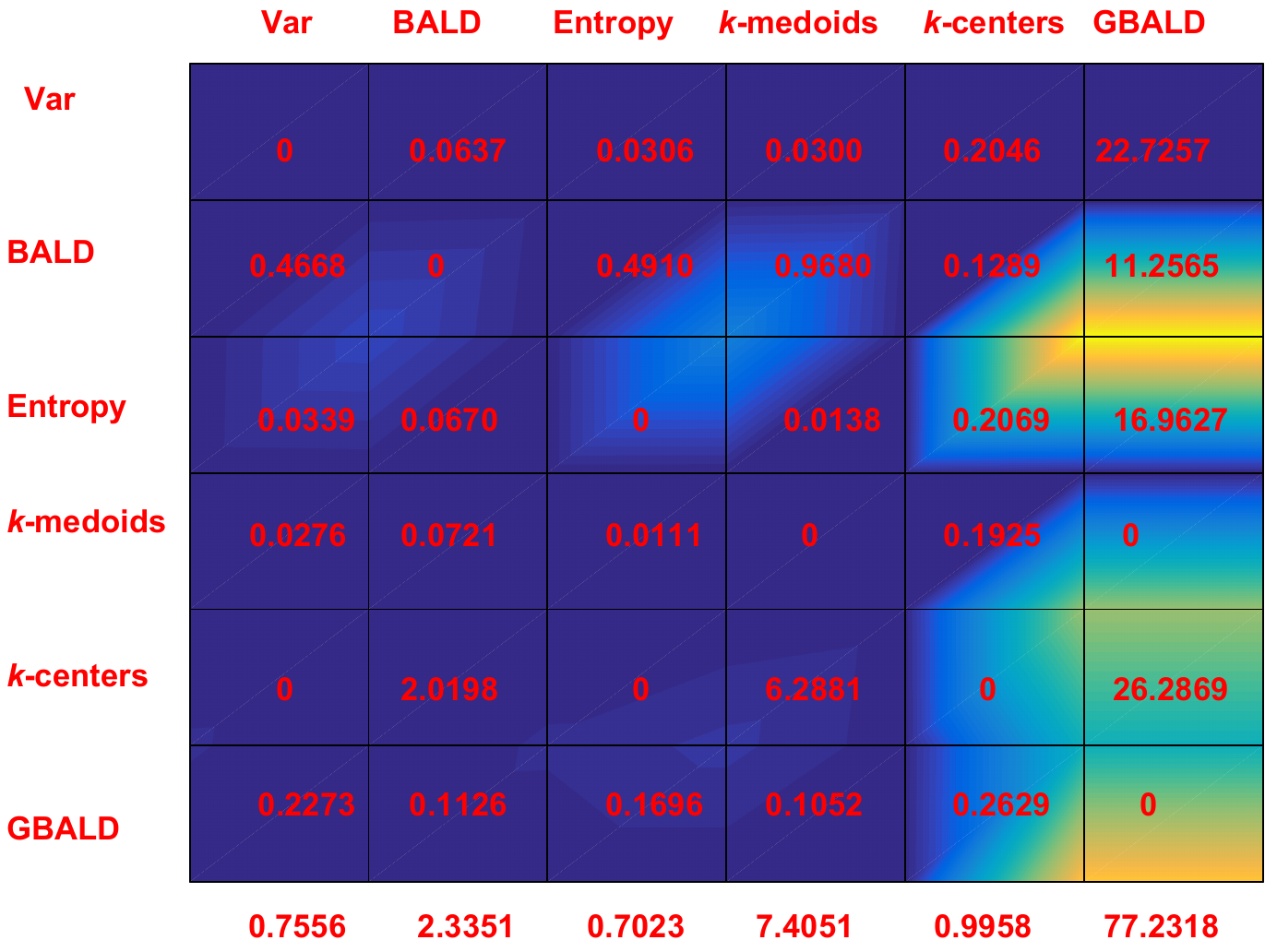}
\end{minipage}
}
\subfloat[CIFAR10]{
\label{fig:improved_subfig_b}
\begin{minipage}[t]{0.32\textwidth}
\centering
\includegraphics[width=2.38in,height=1.78in]{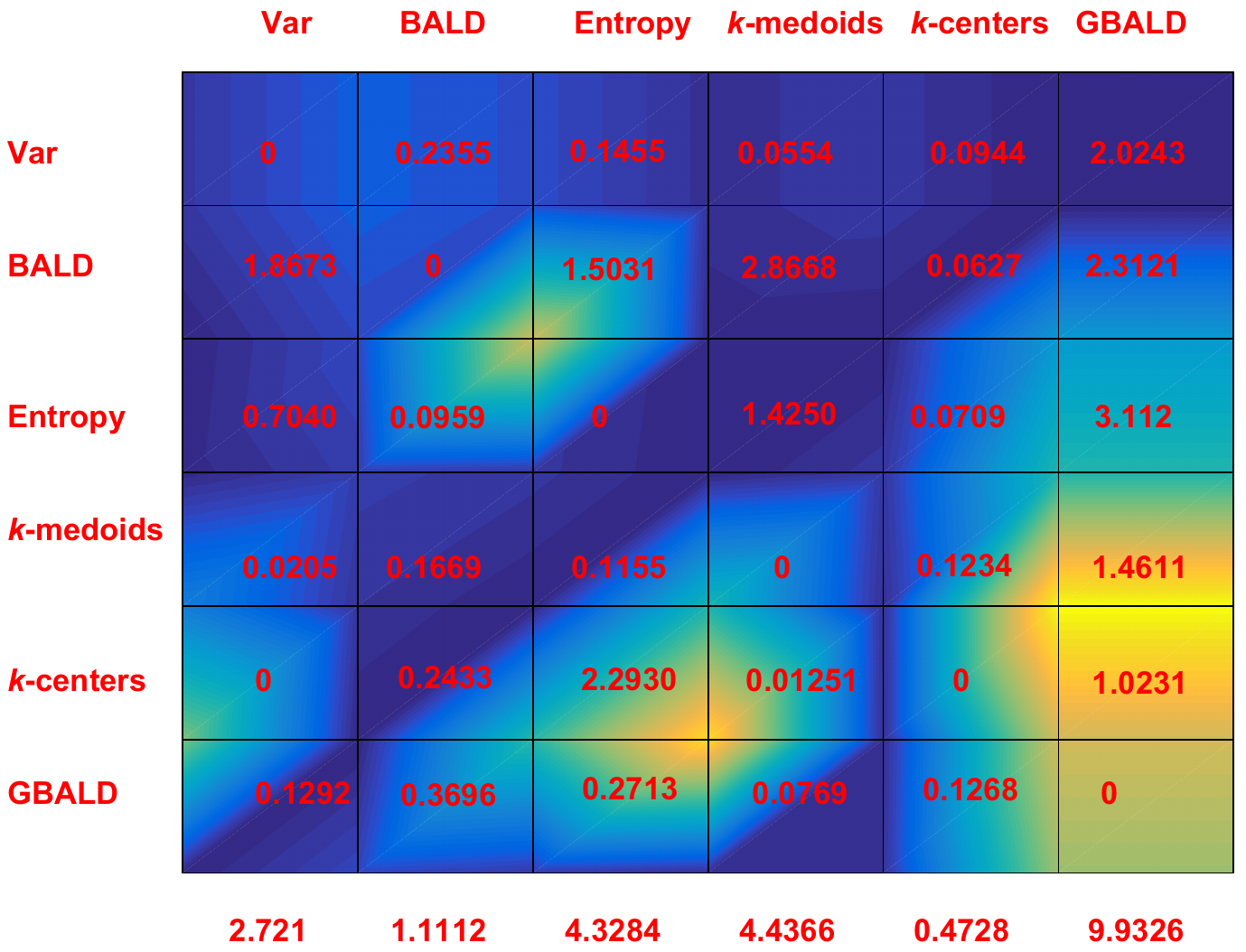}
\end{minipage}
}
\caption{A pairwise penalty matrix over active acquisitions on MNIST, SVHN, and CIFAR10.  Column-wise values 
at the bottom of each matrix show the overall performance of the compared baselines (larger value has more significant superior performance).
 } 
\end{figure*}
\textbf{Hyperparameter settings on batch returns $b$ and bath outputs $b'$.} Experiments of Sections~7.1 and 7.2 used different $b$ and $b'$ to observe the parameter perturbations. No matter what the settings of $b'$ and $b$  are, GBALD still outperforms BALD. For single acquisition of GBALD, we suggest $b=3$ and $b'=1$. For bath  acquisitions, the settings on $b'$ and $b$  are user-defined according the time cost and hardware resources.

\textbf{Hyperparameter setting on  iteration budget $\mathcal{A}$.} Given the acquisition budget $Q$, let $b'$  be the  number of the output returns at each loop,  $\mathcal{N}_\mathcal{M}$ be the number of the core-set size,   the iteration  budget  $\mathcal{A}$ of GBLAD then can be defined as   $\mathcal{A}=(Q-\mathcal{N}_\mathcal{M})/b'$.

\textbf{Other hyperparameter settings.} Eq.~(5) has one parameter $R_0$ which describes the geometric  prior from probability. The default radius of the intern balls $R_0$ is used to legalize the prior and has no further influences on  Eq.~(10). It is set as $R_0=2.0e+3$ for those three image datasets. Ellipsoid geodesic is adjusted by  $\eta$ which controls how far of the updates of   core-set to the boundaries of  distributions. It is set as $\eta=0.9$ in this paper.

\subsection{Two-sided $t$-test}
 We present two-sided (two-tailed) $t$-test \cite{hodges1956efficiency} \cite{donmez2007dual} for the learning curves of Figure~5. Different to the mean$\pm$ std of Table~1,
 $t$-test can enlarge the significant difference of those baselines.  In the typical $t$-test, the two groups of observations usually require a degree of freedom  smaller than 30. However, the numbers of  the breakpoints of MNIST, SVHN, and CIFAR10 are 61, 101, and 201, respectively, thereby holding a degree of freedom of 60, 100, 200, respectively. It is thus we introduce $t$-test score to directly compare the significant difference of the pairwise baselines.

 $t$-test score between any pair group    of breakpoints  are defined as follows.
Let $B_1=\{\alpha_1,\alpha_2, ..., \alpha_n\}$ and $B_2=\{\beta_1, \beta_2, ..., \beta_n\}$, there exists $t$-score  of
\[t-{\rm score}=\sqrt{n}\frac{\mu}{\sigma},\]
where $\mu=\frac{1}{n}\sum_{i=1}^n (\alpha_i-\beta_i)$, and $\sigma=\sqrt{\frac{1}{n-1}\sum_{i=1}^n (\alpha_i-\beta_i-\mu)^2}$.

\par In two-sided $t$-test,  $B_1$ beats $B_2$ on breakpoints $\alpha_i$ and $\beta_i$ satisfying a condition of  $t-{\rm score}>\nu$; $B_2$ beats $B_1$ on breakpoints $\alpha_i$ and $\beta_i$ satisfying  a condition of  $t-{\rm score}<-\nu$, where $\nu$ denotes the hypothesized criterion with a given confidence risk. Following \cite{ash2019deep}, we add a penalty of $\frac{1}{e}$  to each pair of breakpoints, which further enlarges their differences in the aggregated penalty matrix, where $e$ denotes the number of  $B_1$ beats $B_2$ on all breakpoints.    All penalty values finally calculate their  $L_1$ expressions.

Figure~12 presents the penalty matrix over the learning curves of Figure~5.  Column-wise values 
at the bottom of each matrix show the overall performance of the compared baselines. As the shown results, GBALD has significant performance than that of the other baselines over the three datasets. Especially for SVHN, it has superior performance.

\section{Conclusion}
 We have introduced a novel Bayesian AL framework termed GBALD from the  perspective of geometry, which seamlessly incorporates the representative (core-set)  and informative (model uncertainty estimation) acquisitions to accelerate the training of a DNN model. Our GBALD yields significant improvements over BALD, flexibly resolving the limitations of an  uninformative   prior and the redundant information by optimizing the acquisition on an ellipsoid. Generalization analysis has asserted that, geodesic search with ellipsoid  has 
  tighter lower error bound and higher probability to  achieve a zero error, than that of geodesic search with  sphere.  
  Compared to the representative or informative acquisition algorithms,   experiments  show that our GBALD spends much fewer acquisitions to accelerate the  convergence of training model. Moreover, it keeps slighter accuracy reduction than other baselines against repeated and noisy acquisitions. 
 Leveraging the acquisition sizes of the geometric  core-set decides how our framework   interacts with  the  representative and informative acquisitions. It will be  a future work of Auto-AL that derives  an advanced AL pipeline.

\bibliographystyle{IEEEtran}
\bibliography{References}

 \onecolumn
   \begin{figure*}
\label{fig:improved_subfig_b}
\centering
\includegraphics[width=3.8in,height=2.28in]{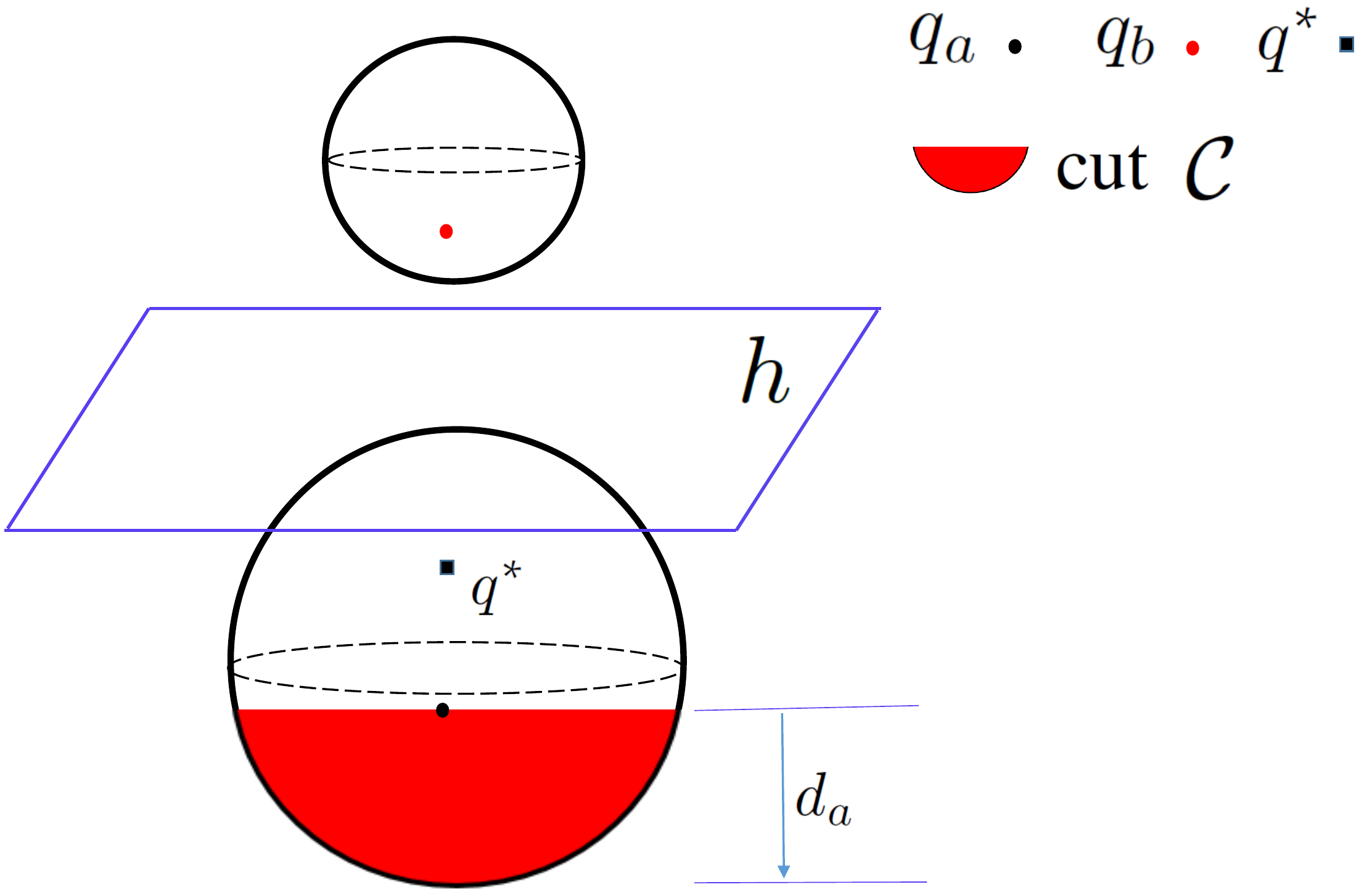}
\caption{Assumption of the generalization analysis. The ball above $h$ denotes $S_b$, the ball below $h$ denotes $S_a$, and $R_b<R_a$.  Spherical cap$^1$  of the half-sphere of $S_a$ is denoted as $\mathcal{D}$.
 } 

\end{figure*}

\subsection*{A.1 Case study of generalization analysis of $err(h,3)$ of geodesic search with sphere }
\begin{theorem}
Given a  perceptron  function $h={w_1 x_1+w_2x_2+w_3}$ that classifies $A$ and $B$,  and a sampling budget $k$.  By drawing core-set on $S_a$ and $S_b$, the minimum distance to the boundaries of  that core-set elements  of  $S_A$ and $S_B$,    are defined as  $d_a$ and  $d_b$, respectively.  Let  $err(h,k)$ be the classification error rate with respect to $h$ and $k$, given $\frac{\pi}{\varphi}={\rm arcsin} \frac{R_a-d_a}{R_a}$,  we have an inequality of error:

 \[  {\rm min} \Bigg \{ \frac{(2R_a+t)(R_a-t)^2}{4R_a^3+4R_b^3},   \frac{(2R_b+t')(R_b-t')^2}{4R_b^3+4R_a^3}\Bigg\}< err(h,3)<0.3334,\] 
where $t= \frac{R_a^2}{3}+\sqrt[3]{-\frac{\mu}{2 \pi}+\sqrt{\frac{\mu^2}{4 \pi^2}-\frac{ \pi^3  R_a^3}{27\pi^3}      }       }+\sqrt[3]{-\frac{\mu}{2 \pi}-\sqrt{\frac{\mu^2}{4 \pi^2}-\frac{ \pi^3  R_a^3}{27\pi^3}      }       }$, $\mu= (\frac{2}{9}-\frac{1}{\varphi}cos \frac{\pi}{\varphi}  )\pi R_a^3- \frac{4\pi R_b^3}{9}$, $t'= \frac{R_b^2}{3}+\sqrt[3]{-\frac{\mu'}{2 \pi}+\sqrt{\frac{\mu'^2}{4 \pi^2}-\frac{ \pi^3  R_b^3}{27\pi^3}      }       }+\sqrt[3]{-\frac{\mu'}{2 \pi}-\sqrt{\frac{\mu'^2}{4 \pi^2}-\frac{ \pi^3  R_b^3}{27\pi^3}      }       } $,  and  $\mu'= (\frac{2}{9}-\frac{1}{\varphi}cos \frac{\pi}{\varphi}  )\pi R_b^3- \frac{4\pi R_a^3}{9}$.
\end{theorem} 

%\section*{Proof of Theorem 5}

\begin{proof}
 Given the unseen acquisitions of $\{q_a,q_b,q^*\}$, where $q_a\in A$, $q_b\in B$, and $q^*\in A$ or $q^*\in B$ is uncertain. However, the position of $q^*$ largely decides $h$. Therefore,  the proof studies  the error bounds highly related to $q^*$ in terms of  two cases: $R_a\geq R_b$ and $R_a<R_b$.
 
 1) If $R_a\geq R_b$, $q^*\in  {A}$. Estimating the position of $q^*$ starts from the analysis on $q_a$.
  Given the volume function ${\rm Vol}(\cdot)$ over the 3-D geometry, we know: ${\rm Vol}({A})=\frac{4\pi}{3} R_a^3$ and ${\rm Vol}({B})=\frac{4\pi}{3} R_b^3$. Given $k=3$ over $S_{a}$ and $S_{b}$, we define the minimum distance of $q_a$ to the boundary of ${A}$ as $d_a$. Let $S_a$ be cut off by a cross section $h'$, where 
 $\mathcal{C}$ be the cut and  $\mathcal{D}$ be the  \emph{spherical cap}\footnote{https://en.wikipedia.org/wiki/Spherical\_cap} 
 of the half-sphere (see Figure~13) that satisfy
 \begin{equation}
{\rm Vol}(\mathcal{C})=\frac{2}{3}\pi R_a^3 -{\rm Vol}(\mathcal{D})=\frac{4\pi(R_a^3+R_b^3)}{9}, 
 \end{equation}
and the volume of $\mathcal{D}$ is 
\begin{equation}
 \begin{split}
{\rm Vol}(\mathcal{D})&% =Vol(Cone)+Vol(Mount Frustum) \\
                                   =\pi ( \sqrt{R_a^2-(R_a-d_a)^2} )^2(R_a-d_a)+\int_0^{2\pi \sqrt{R_a^2-(R_a-d_a)^2}  }  \frac{{\rm arcsin} \frac{R_a-d_a}{R_a} }{2\pi}\pi R_a^2 \ d x \\
                               &=   \pi ( {R_a^2-(R_a-d_a)^2} )(R_a-d_a)+\frac{{\rm arcsin} \frac{R_a-d_a}{R_a} }{2\pi}\pi R_a^2(2\pi \sqrt{R_a^2-(R_a-d_a)^2})\\
                                &= \pi ({R_a^2-(R_a-d_a)^2} )(R_a-d_a)+\pi R_a^2 {\rm arcsin} \frac{R_a-d_a}{R_a} \sqrt{R_a^2-(R_a-d_a)^2}.\\
\end{split}
 \end{equation}
 Let $\frac{\pi}{\varphi}={\rm arcsin} \frac{R_a-d_a}{R_a}$, Eq.~(20) can be written as 
\begin{equation}
 \begin{split}
{\rm Vol}(\mathcal{D})=  \pi ({R_a^2-(R_a-d_a)^2} )(R_a-d_a)+\frac{\pi  R_a^3}{\varphi }cos \frac{\pi}{\varphi}.\\
\end{split}
 \end{equation}
 Introducing Eq.~(21) to Eq.~(19), we have 
 
 \begin{equation}
 \begin{split}
(\frac{2}{3}- \frac{1}{\varphi} cos \frac{\pi}{\varphi}     )\pi R_a^3- \pi ({R_a^2-(R_a-d_a)^2} )(R_a-d_a) =\frac{4\pi(R_a^3+R_b^3)}{9}.
\end{split}
\end{equation}
 
Let $t=R_a-d_a$,  Eq.~(22) can be rewritten as 
 \begin{equation}
 \begin{split}
 \pi t^3- \pi  R_a^2t +(\frac{2}{9}-\frac{1}{\varphi}cos \frac{\pi}{\varphi}  )\pi R_a^3- \frac{4\pi R_b^3}{9}=0.
\end{split}
\end{equation}

To simplify Eq.~(23), let $\mu= (\frac{2}{9}-\frac{1}{\varphi}cos \frac{\pi}{\varphi}  )\pi R_a^3- \frac{4\pi R_b^3}{9}$,  Eq.~(23) then can be written as
 \begin{equation}
 \begin{split}
 \pi t^3- \pi  R_a^2t +\mu=0.
\end{split}
\end{equation}

The positive solution  of  $t$ can be
 \begin{equation}
 \begin{split}
&t= \frac{R_a^2}{3}+\sqrt[3]{-\frac{\mu}{2 \pi}+\sqrt{\frac{\mu^2}{4 \pi^2}-\frac{ \pi^3  R_a^3}{27\pi^3}      }       }+\sqrt[3]{-\frac{\mu}{2 \pi}-\sqrt{\frac{\mu^2}{4 \pi^2}-\frac{ \pi^3  R_a^3}{27\pi^3}      }       }.
\end{split}
\end{equation}

 Based on Eq.~(19),  we know 
  \begin{equation}
 \begin{split}
{\rm Vol}(\mathcal{D})&=\frac{2}{3}\pi R_a^3-\frac{4\pi(R_a^3+R_b^3)}{9}\\
         &=\frac{2}{9}\pi R_a^3- \frac{4}{9}\pi R_b^3>0.      \\               
\end{split}
\end{equation}
Thus, $ \sqrt[3]{2}R_b< R_a$. We next prove $q^*\in {A}$.  Based on Eq.~(26), we know
  \begin{equation}
 \begin{split}
\pi R_b^3  <           \frac{1}{2}\pi R_a^3.
\end{split}
\end{equation}
Then,   the following inequalities hold: 1)$2\pi R_b^3  <            \pi R_a^3$, 2) $\frac{2}{3}\pi R_b^3  <     \frac{1}{3}       \pi R_a^3$, and 3) $\frac{2}{3}\pi R_b^3+ \frac{1}{3}\pi R_b^3  <    \frac{1}{3}       \pi R_a^3+\frac{1}{3}\pi R_b^3$. Finally, we have
  \begin{equation}
 \begin{split}
  \pi R_b^3  <          \frac{1}{3}(\pi R_a^3+\pi R_b^3).
\end{split}
\end{equation}
Therefore, ${\rm Vol}({B})<\frac{1}{3}({\rm Vol}( {A})+{\rm Vol}({B}))$. We thus know: 1) $q_a\in  {A}$ and it is with a minimum distance $d_a$ to the boundary of $ S_a$, 2) $q_b\in {B}$, and 3) $q^*\in {A}$. Therefore, class $ {B}$ can be deemed as having a very high probability to achieve a  nearly zero generalization error  and the position of $q^*$ largely decides the upper bound of the  generalization error of $h$. 
\par In $S_a$ that covers class $A$,  the nearly optimal  error region can be bounded as the  spherical cap  of $S_a$ with a volume constraint of 
${\rm Vol}({A})-{\rm Vol}(\mathcal{C})$. We thence have an inequality of
  \begin{equation}
 \begin{split}
\frac{{\rm Vol}( {A})-{\rm Vol}(\mathcal{C})}{{\rm Vol}( {A})+{\rm Vol}( {B})}<err(h,3)<\frac{1}{3}.
\end{split}
\end{equation} 
We next calculate the volume of the  spherical cap:
\begin{equation}
 \begin{split}
 {\rm Vol}(A)-{\rm Vol}(\mathcal{C})&=\int_{-R_a}^{R_a-d_a} \pi x^2 d\ y  \\
                                                                  & = \pi\int_{R_a-d_a}^{R_a}  R_a^2-y^2 d\ y  \\
                                                                  & = \frac{4}{3}\pi R_a^3- {\pi}d_a^2(R_a-\frac{d_a}{3}).      \\
\end{split}
\end{equation}
 Eq.~(29) then is rewritten as
  \begin{equation}
 \begin{split}
\frac{\frac{4}{3}\pi R_a^3- \frac{\pi}{3}(3R_a-d_a)d_a^2}{\frac{4}{3}\pi R_a^3+\frac{4}{3}\pi R_b^3}<err(h,3)<0.3334,
\end{split}
\end{equation} 
Then, we have the error bound of 
 \begin{equation}
 \begin{split}
 \frac{4R_a^3-(3R_a-d_a)d_a^2}{4R_a^3+4R_b^3}< err(h,3)<0.3334.    
\end{split}
\end{equation} 
Introducing $d_a=R_a-t$, Eq.~(32) is written as 
 \begin{equation}
 \begin{split}
 \frac{4R_a^3-(2R_a+t)(R_a-t)^2}{4R_a^3+4R_b^3}< err(h,3)<0.3334.    
\end{split}
\end{equation}

2) With another assumption of $R_a<R_b$, we follow the same proof skills  of  $R_a\geq R_b$ and know $ \frac{4R_b^3-(3R_b-d_b)d_b^2}{4R_b^3+4R_a^3}< err(h,3)<0.3334$, i.e.  where $d_b=R_b-t'$ and $t'= \frac{R_b^2}{3}+\sqrt[3]{-\frac{\mu'}{2 \pi}+\sqrt{\frac{\mu'^2}{4 \pi^2}-\frac{ \pi^3  R_b^3}{27\pi^3}      }       }+\sqrt[3]{-\frac{\mu'}{2 \pi}-\sqrt{\frac{\mu'^2}{4 \pi^2}-\frac{ \pi^3  R_b^3}{27\pi^3}      }       } $,  and  $\mu'= (\frac{2}{9}-\frac{1}{\varphi}cos \frac{\pi}{\varphi}  )\pi R_b^3- \frac{4\pi R_a^3}{9}$.

We thus conclude that ${\rm min} \Bigg \{ \frac{4R_a^3-(2R_a+t)(R_a-t)^2}{4R_a^3+4R_b^3},   \frac{4R_b^3-(2R_b+t')(R_b-t')^2}{4R_b^3+4R_a^3}\Bigg\}< err(h,3)<0.3334$.
\end{proof}

%%%%%%%%%%%%%%%%%%%%%%%%%%%%%%%%%%%%%%%%%%%%%%%%%%%%%%%%%
%%%%%%%%%%%%%%%%%%%%%%%%%%%%%%%%%%%%%%%%%%%%%%%%%%%%%%%%%

\subsection*{A.2 Specification of Assumption~1 }
 
In clustering stability, $\gamma$-tube structure that surrounds the cluster boundary, largely decides the performance of a learning algorithm. Definition of  $\gamma$-tube is as follows.

\begin{definition}
$\gamma$-tube  ${\rm Tube}_\gamma(f)$   is a set of points distributed in the boundary of the cluster.
\begin{equation}
\begin{split}
{\rm Tube}_\gamma(f): = \{ x\in {X} | \ell(x,B(f)) \leq \gamma \},
\end{split} 
\end{equation}
where  ${X}$ is a noise-free  cluster  with  $n$ samples,   $B(f):=\{x\in {X}, f  \ \mbox{ is discontinuous  at   \  x } \}$,  $f$ is a clustering function, and $\ell(\cdot,\cdot)$ denotes the distance function.
\end{definition}

Following this conclusion, representation data can achieve the optimal generalization error if they are spread over the tube structure. Let $\gamma=d_a$, the probability of achieving a nearly zero generalization   error  can be expressed as the   volume ration of $\gamma$-tube and $S_a$:
 \begin{equation}
\begin{split}
{\rm Pr}[err(h,k)=0]_{\rm Sphere}&=\frac{{\rm Vol}({\rm Tube_\gamma}) }{{\rm Vol}{(S_a)} }\\
                              &=\frac{\frac{4}{3}\pi R_a^3-\frac{4}{3}(R_a-d_a)^3}{\frac{4}{3}\pi R_a^3} \\
                               &=1-\frac{t_k^3}{R_a^3},\\
\end{split} 
\end{equation}
where $t_k$ keeps consistent with  Eq.~(40). With the initial sampling from the tube structure of  class $A$, the subsequent acquisitions of AL would be updated from the tube structure of class $B$. If the initial sampling comes from the tube structure of $B$, the next acquisition must be updated from the tube structure of $A$.  With the updated acquisitions spread over the tube structures of both classes, $h$ is easy to achieve a nearly zero error.

\subsection*{A.3  Specification of Assumption~2 }
 
Following the specification of Assumption~1, volume of the tube is redefined as ${\rm Vol(Tube_\gamma)}=\frac{4}{3}\pi R_{a_1}R_{a_2}R_{a_3}$. Then, we know
 \begin{equation}
\begin{split}
{\rm Pr}[err(h,k)=0]_{\rm Ellipsoid}&=\frac{{\rm Vol}({\rm Tube_\gamma}) }{{\rm Vol}{(E_a)} }\\
                              &=\frac{\frac{4}{3}\pi R_{a_1}R_{a_2}R_{a_3}-\frac{4}{3}(R_{a_1}-d_a)(R_{a_2}-d_a)(R_{a_3}-d_a) }{\frac{4}{3}\pi R_{a_1}R_{a_2}R_{a_3}} \\
                               &=1-\frac{\lambda_{k_1}\lambda_{k_2}\lambda_{k_3}}{R_{a_1}R_{a_2}R_{a_3}},\\
\end{split} 
\end{equation}
where $\lambda_{k_i}= \frac{R_{a_i}^2}{3}+\sqrt[3]{-\frac{\sigma_{k_i}}{2 \pi}+\sqrt{\frac{\sigma_{k_i}^2}{4 \pi^2}-\frac{ \pi^3  R_{a_i}^3}{27\pi^3}      }       }+\sqrt[3]{-\frac{\sigma_{k_i}}{2 \pi}-\sqrt{\frac{\sigma_{k_i}^2}{4 \pi^2}-\frac{ \pi^3  R_{a_i}^3}{27\pi^3}      }       }$, and $\sigma_{k_i}= (\frac{2k-4}{3k}-\frac{\pi R_{a_i}  }{2\varphi}  )\pi R_{a_1}R_{a_2}R_{a_3}- \frac{4\pi R_{b_1}R_{b_2}R_{b_3}}{3k}$, $i=1,2,3$.

\subsection*{A.4  Proof of Theorem 1}
We next present the generalization errors  against an agnostic sampling budget $k$ following the above proof technique.
\begin{proof}
 The proof studies two cases: $R_a\geq R_b$ and $R_a<R_b$. 1) If $R_a\geq R_b$,  we estimate the optimal position of $q_a$ that satisfies $q_a \in A$. Given the volume function ${\rm Vol}(\cdot)$ over the 3-D geometry, we know: ${\rm Vol}({A})=\frac{4\pi}{3} R_a^3$ and ${\rm Vol}({B})=\frac{4\pi}{3} R_b^3$. 
Assume $q_a$ be the nearest representative data to the boundary of $S_a$,  $q_b$ be the nearest representative data to the boundary of $S_b$, and $q^*$ be the nearest  representative data to $h$ either in  $S_a$ or $S_b$. 
Given   the minimum distance of $q_a$ to the boundary of $\mathcal{A}$ as $d_a$. Let $S_{a}$ be cut off by a cross section $h'$, where 
 $\mathcal{C}$ be the cut and  $\mathcal{D}$ be the  spherical cap  of the half-sphere that satisfy
 \begin{equation}
{\rm Vol}(\mathcal{C})=\frac{2}{3}\pi R_a^3 -{\rm Vol}(\mathcal{D})=\frac{4\pi(R_a^3+R_b^3)}{3k}, 
 \end{equation}
and the volume of $\mathcal{D}$ is 
\begin{equation}
 \begin{split}
{\rm Vol}(\mathcal{D})&% =Vol(Cone)+Vol(Mount Frustum) \\
                                   =\pi ( \sqrt{R_a^2-(R_a-d_a)^2} )^2(R_a-d_a)+\int_0^{2\pi \sqrt{R_a^2-(R_a-d_a)^2}  }  \frac{{\rm arcsin} \frac{R_a-d_a}{R_a} }{2\pi}\pi R_a^2 \ d x \\
                               &=   \pi ( {R_a^2-(R_a-d_a)^2} )(R_a-d_a)+\frac{{\rm arcsin} \frac{R_a-d_a}{R_a} }{2\pi}\pi R_a^2(2\pi \sqrt{R_a^2-(R_a-d_a)^2})\\
                                &= \pi ({R_a^2-(R_a-d_a)^2} )(R_a-d_a)+\pi R_a^2  {\rm arcsin} \frac{R_a-d_a}{R_a} \sqrt{R_a^2-(R_a-d_a)^2}.\\
\end{split}
 \end{equation}
 Let $\frac{\pi}{\varphi}={\rm arcsin} \frac{R_a-d_a}{R_a}$, Eq.~(36) can be written as 
\begin{equation}
 \begin{split}
{\rm Vol}(\mathcal{D})=  \pi ({R_a^2-(R_a-d_a)^2} )(R_a-d_a)+\frac{\pi  R_a^3}{\varphi }cos \frac{\pi}{\varphi}.\\
\end{split}
 \end{equation}
 Introducing Eq.~(36) to Eq.~(34), we have 
 \begin{equation}
 \begin{split}
(\frac{2}{3}-\frac{1}{\varphi}cos \frac{\pi}{\varphi}     )\pi R_a^3- \pi ({R_a^2-(R_a-d_a)^2} )(R_a-d_a) =\frac{4\pi(R_a^3+R_b^3)}{3k}.
\end{split}
\end{equation}
 
Let $t_k=R_a-d_a$,  we know
 \begin{equation}
 \begin{split}
 \pi t^3- \pi  R_a^2t +(\frac{2k-4}{3k}-\frac{1}{\varphi}cos \frac{\pi}{\varphi}  )\pi R_a^3- \frac{4\pi R_b^3}{3k}=0.
\end{split}
\end{equation}

To simplify Eq.~(38), let $\mu_k= (\frac{2k-4}{3k}-\frac{1}{\varphi}cos \frac{\pi}{\varphi}  )\pi R_a^3- \frac{4\pi R_b^3}{3k}$,  Eq.~(38) then can be written as
 \begin{equation}
 \begin{split}
 \pi t^3- \pi  R_a^2t +\mu_k=0.
\end{split}
\end{equation}

The positive solution  of  $t_k$ can be
 \begin{equation}
 \begin{split}
&t_k= \frac{R_a^2}{3}+\sqrt[3]{-\frac{\mu_k}{2 \pi}+\sqrt{\frac{\mu_k^2}{4 \pi^2}-\frac{ \pi^3  R_a^3}{27\pi^3}      }       }+\sqrt[3]{-\frac{\mu_k}{2 \pi}-\sqrt{\frac{\mu_k^2}{4 \pi^2}-\frac{ \pi^3  R_a^3}{27\pi^3}      }       }.
\end{split}
\end{equation}

 Based on Eq.~(35),  we know 
  \begin{equation}
 \begin{split}
{\rm Vol}(\mathcal{D})&=\frac{2}{3}\pi R_a^3-\frac{4\pi(R_a^3+R_b^3)}{3k}\\
         &=\frac{2k-4}{3k}\pi R_a^3- \frac{4}{3k}\pi R_b^3>0.      \\               
\end{split}
\end{equation}
Thus, $ \sqrt[3]{\frac{2}{k-2}}R_b< R_a$. We next prove $q^*\in {A}$.  According to  Eq.~(41), we know
  \begin{equation}
 \begin{split}
\pi R_b^3  <       {\frac{k-2}{2}}\pi R_a^3.
\end{split}
\end{equation}

Then,   the following inequalities hold: 1)$\frac{2}{k-2} \pi R_b^3 <            \pi R_a^3$, 2) $\frac{ 2   }{(k-2)k}\pi R_b^3  <       \frac{1}{k}       \pi R_a^3$, and 3) $\frac{ 2   }{(k-2)k}\pi R_b^3+ \frac{k^2-2k-2}{(k-2)k}\pi R_b^3 <       \frac{1}{k}       \pi R_a^3+\frac{k^2-2k-2}{(k-2)k}\pi R_b^3$. Finally, we have:
  \begin{equation}
 \begin{split}
&  \pi R_b^3  \\
&  <          \frac{1}{k}       \pi R_a^3+\frac{k^2-2k-2}{(k-2)k}\pi R_b^3\\
 &   =\frac{1}{ k}\pi (R_a^3+R_b^2)-\frac{2}{(k-2)k}\pi R_a^3\\
 &< \frac{1}{ k}\pi (R_a^3+R_b^3).          \\
\end{split}
\end{equation}

Therefore, ${\rm Vol}({B})<\frac{1}{k}({\rm Vol}( {A})+{\rm Vol}( {B}))$. We thus know: 1) $q_a\in  {A}$ and it is with a minimum distance $d_a$ to the boundary of $S_a$, 2) $q_b\in {B}$, and 3) $q^*\in  {A}$. Therefore, class   $ {B}$ can be deemed as having a very high probability to achieve a zero generalization error  and the position of $q^*$ largely decides the upper bound of the  generalization error of $h$.  
\par In $S_a$ that covers class $A$,  the nearly optimal error region can be bounded as ${\rm Vol}({A})-{\rm Vol}({C})$. We then have the inequality of
  \begin{equation}
 \begin{split}
 \frac{{\rm Vol}( {A})-{\rm Vol}(\mathcal{C})}{{\rm Vol}( {A})+{\rm Vol}( {B})}<err(h,k)<\frac{1}{k}.
\end{split}
\end{equation} 
 Based on the volume equation of the spherical cap in Eq.~(30), we  have
  \begin{equation}
 \begin{split}
\frac{\frac{4}{3}\pi R_a^3- \frac{\pi}{3}(3R_a-d_a)d_a^2}{\frac{4}{3}\pi R_a^3+\frac{4}{3}\pi R_b^3}<err(h,k)<\frac{1}{k}.
\end{split}
\end{equation} 
Then, we have the error bound of 
 \begin{equation}
 \begin{split}
 \frac{4R_a^3-(3R_a-d_a)d_a^2}{4R_a^3+4R_b^3}< err(h,k)<\frac{1}{k}.    
\end{split}
\end{equation} 
Introducing $d_a=R_a-t_k$, Eq.~(46) is written as 
 \begin{equation}
 \begin{split}
 \frac{4R_a^3-(2R_a+t)(R_a-t_k)^2}{4R_a^3+4R_b^3}< err(h,k)< \frac{1}{k}.    
\end{split}
\end{equation}

2) With another assumption of $R_a<R_b$, we follow the same proof skills  of  $R_a\geq R_b$ and know $ \frac{4R_b^3-(3R_b-d_b)d_b^2}{4R_b^3+4R_a^3}< err(h,k)< \frac{1}{k}$,  where $d_b=R_b-t_k'$ and $t_k'= \frac{R_b^2}{3}+\sqrt[3]{-\frac{\mu_k'}{2 \pi}+\sqrt{\frac{\mu_k'^2}{4 \pi^2}-\frac{ \pi^3  R_b^3}{27\pi^3}      }       }+\sqrt[3]{-\frac{\mu_k'}{2 \pi}-\sqrt{\frac{\mu_k'^2}{4 \pi^2}-\frac{ \pi^3  R_b^3}{27\pi^3}      }       } $,  and  $\mu_k'=  (\frac{2k-4}{3k}-\frac{1}{\varphi}cos \frac{\pi}{\varphi}  )\pi R_b^3- \frac{4\pi R_a^3}{3k}$.

We thus conclude that ${\rm min} \Bigg \{ \frac{4R_a^3-(2R_a+t_k)(R_a-t_k)^2}{4R_a^3+4R_b^3},   \frac{4R_b^3(2R_b+t_k')(R_b-t_k')^2}{4R_b^3+4R_a^3}\Bigg\}< err(h,k)<\frac{1}{k}$.
\end{proof}

%%%%%%%%%%%%%%%%%%%%%%%%%%%%%%%%%%%%%%%%%%%%%%%%%%%%%%%%%
%%%%%%%%%%%%%%%%%%%%%%%%%%%%%%%%%%%%%%%%%%%%%%%%%%%%%%%%%
\subsection*{A.5 Proof of Theorem 2}
\begin{proof}
Given   class  $A$ and $B$   are tightly covered by   ellipsoid   $E_a$ and $E_b$ in a three-dimensional geometry. Let  $R_{a_1}$ be the polar radius of   $E_a$, 
$\{R_{a_2}, R_{a_3}\}$ be the equatorial radii of $E_a$,  $R_{b_1}$ be polar radius of   $E_b$, 
and $\{R_{b_2}, R_{b_3}\}$ be the equatorial radii of $E_b$. Based on Eq.~(10), we know $R_{a_i}< R_a,  R_{b_i}< R_b, \forall i$, where $R_a$ and $R_b$ are the radii of the spheres over the  class $A$ and $B$, respectively.  We follow the same proof technique of Theorem 1 to present the generalization errors of AL  with ellipsoid.

The proof studies two cases: $R_{a_1}\geq R_b$ and $R_{a_1}<R_{b_1}$. 1) If $R_{a_1}\geq R_{b_1}$, $q^*\in  {A}$. Given the volume function ${\rm Vol}(\cdot)$ over the 3-D geometry, we know: ${\rm Vol}( {A})=\frac{4\pi}{3} R_{a_1}R_{a_2}R_{a_3}$ and ${\rm Vol}( {B})=\frac{4\pi}{3} R_{b_1}R_{b_2}R_{b_3}$. Given   the minimum distance of $q_a$ to the boundary of ${A}$ as $d_a$. Let $E_{a}$ by cut off by a cross section $h'$, where 
 $\mathcal{C}$ be the cut and  $\mathcal{D}$ be the ellipsoid cap  of the half-ellipsoid that satisfy
 \begin{equation}
{\rm Vol}(\mathcal{C})=\frac{2}{3}\pi R_a^1R_a^2R_a^3 -{\rm Vol}(\mathcal{D})=\frac{4\pi(R_{a_1}R_{a_2}R_{a_3}+R_{b_1}R_{b_2}R_{b_3})}{3k}, 
 \end{equation}
and the volume of $\mathcal{D}$ is approximated as
\begin{equation}
 \begin{split}
{\rm Vol}(\mathcal{D})&% =Vol(Cone)+Vol(Mount Frustum) \\
                                   \approx \pi ( \sqrt{R_{a_1}^2-(R_{a_1}-d_a)^2} )^2(R_{a_1}-d_a)+\int_0^{\pi  R_{a_2} R_{a_3}  }  \frac{arcsin \frac{R_{a_1}-d_a}{R_{a_1}} }{2\pi}\pi R_{a_1}^2 \ d x \\
                               &=   \pi ( {R_{a_1}^2-(R_{a_1}-d_a)^2} )(R_{a_1}-d_a)+\frac{arcsin \frac{R_{a_1}-d_a}{R_a} }{2\pi}\pi R_{a_1}^2(\pi  R_{a_2} R_{a_3} )\\
                                &= \pi ({R_{a_1}^2-(R_{a_1}-d_a)^2} )(R_{a_1}-d_a)+\frac{1}{2}\pi R_{a_1}^2R_{a_2}R_{a_3}  arcsin \frac{R_{a_1}-d_a}{R_{a_1}}.\\
\end{split}
 \end{equation}
 Let $\frac{\pi}{\varphi}={\rm arcsin} \frac{R_{a_1}-d_a}{R_{a_1}}$, Eq.~(49) can be written as 
\begin{equation}
 \begin{split}
{\rm Vol}(\mathcal{D})=  \pi ({R_{a_1}^2-(R_{a_1}-d_a)^2} )(R_{a_1}-d_a)+ \frac{\pi^2}{2\varphi}  R_{a_1}^2R_{a_2}R_{a_3}   .\\
\end{split}
 \end{equation}
 Introducing Eq.~(50) to Eq.~(48), we have 
  \begin{equation}
 \begin{split}
(\frac{2}{3}-\frac{\pi R_{a_1}  }{2\varphi} )\pi R_{a_1}R_{a_2}R_{a_3}- \pi ({R_{a_1}^2-(R_{a_1}-d_a)^2} )(R_{a_1}-d_a) =\frac{4\pi(R_{a_1}R_{a_2}R_{a_3}+R_{b_1}R_{b_2}R_{b_3})}{3k}.
\end{split}
\end{equation}
 
Let $\lambda_k=R_{a_1}-d_a$,  we know
 \begin{equation}
 \begin{split}
 \pi \lambda_k^3- \pi  R_a^2\lambda_k +(\frac{2k-4}{3k}-\frac{\pi R_{a_1}  }{2\varphi}  )\pi R_{a_1}R_{a_2}R_{a_3}- \frac{4\pi R_{b_1}R_{b_2}R_{b_3}}{3k}=0.
\end{split}
\end{equation}

To simplify Eq.~(52), let $\sigma_k= (\frac{2k-4}{3k}-\frac{\pi R_{a_1}  }{2\varphi}  )\pi R_{a_1}R_{a_2}R_{a_3}- \frac{4\pi R_{b_1}R_{b_2}R_{b_3}}{3k}$,  Eq.~(52) then can be written as
 \begin{equation}
 \begin{split}
 \pi \lambda_k^3- \pi  R_{a_1}^2\lambda_k +\sigma_k=0.
\end{split}
\end{equation}

The positive solution  of  $\lambda_k$ can be
 \begin{equation}
 \begin{split}
&\lambda_k= \frac{R_{a_1}^2}{3}+\sqrt[3]{-\frac{\sigma_k}{2 \pi}+\sqrt{\frac{\sigma_k^2}{4 \pi^2}-\frac{ \pi^3  R_{a_1}^3}{27\pi^3}      }       }+\sqrt[3]{-\frac{\sigma_k}{2 \pi}-\sqrt{\frac{\sigma_k^2}{4 \pi^2}-\frac{ \pi^3  R_{a_1}^3}{27\pi^3}      }       }.
\end{split}
\end{equation}

The remaining proof process follows Eq.~(40) to Eq.~(46) of Theorem 1. 
We thus conclude that
\begin{equation}
 \begin{split}
{\rm min} \Bigg \{ \frac{4\prod_i R_{a_i}-(2R_{a_1}+\lambda_k)(R_{a_1}-\lambda_k)^2}{4\prod_i R_{a_i}+4\prod_i R_{b_i}},   \frac{4\prod_i R_{b_i}-(2R_{b_1}+\lambda_k')(R_{b_1}-\lambda_k')^2}{4\prod_i R_{b_i}+4\prod_i R_{a_i}}\Bigg\}< err(h,k)<\frac{1}{k},
\end{split}
\end{equation}
where $i=1,2,3$, $\lambda_k'= \frac{R_{b_1}^2}{3}+\sqrt[3]{-\frac{\sigma_k}{2 \pi}+\sqrt{\frac{\sigma_k'^2}{4 \pi^2}-\frac{ \pi^3  R_{b_1}^3}{27\pi^3}      }       }+\sqrt[3]{-\frac{\sigma_k}{2 \pi}-\sqrt{\frac{\sigma_k^2}{4 \pi^2}-\frac{ \pi^3  R_{a_1}^3}{27\pi^3}      }       }$, and $\sigma_k'= (\frac{2k-4}{3k}-\frac{\pi R_{b_1}  }{2\varphi}  )\pi R_{b_1}R_{b_2}R_{b_3}- \frac{4\pi R_{a_1}R_{a_2}R_{a_3}}{3k}$. In a simple way, $R_{a_1}R_{a_2}R_{a_3}$ and $R_{b_1}R_{b_2}R_{b_3}$ can be written as $\prod_i R_{a_i}$ and $\prod_i R_{b_i}$, i=$1,2,3$, respectively.

\end{proof}

\subsection*{A.6  Proof of Proposition 1}
\begin{proof}
Let ${\rm Cube_a}$ tightly covers $S_a$ with a side length of $2R_a$, and ${\rm Cube}_a'$ tightly covers the cut $\mathcal{C}$,   following theorem 1, we know
\begin{equation}
\begin{split}
            err(h,k)> \frac{{\rm Vol}( {A})-{\rm Vol}(\mathcal{C})}{{\rm Vol}( {A})}> \frac{{\rm Cube_a-Cube_a'}}{{\rm Cube_a}}.
\end{split} 
\end{equation}
Then, we know
\begin{equation}
\begin{split}
            err(h,k)&> \frac{\pi R_a^3-\pi R_a^2d_a}{\pi R_a^3}\\
                      &=1-\frac{d_a}{R_a}.\\
\end{split} 
\end{equation}
 Meanwhile, let  ${\rm Cube_e}$ tightly covers $E_a$ with a side length of $2R_{a_1}$, ${\rm Cube_e}'$ tightly covers $\mathcal{C}$,   following theorem 1, we know
\begin{equation}
\begin{split}
            err(h,k)> \frac{{\rm Vol}( {A})-{\rm Vol}(\mathcal{C})}{{\rm Vol}( {A})}> \frac{{\rm Cube_e-Cube_e'}}{{\rm Cube_e}}.
\end{split} 
\end{equation}
Then, we know
\begin{equation}
\begin{split}
            err(h,k)&> \frac{\pi R_{a_1}R_{a_2}R_{a_3} -\pi  d_aR_{a_2}R_{a_3} }{\pi R_{a_1}R_{a_2}R_{a_3}}\\
                      &=1-\frac{d_a}{R_{a_1}}.\\
\end{split} 
\end{equation}
Since ${R_{a_1}}<R_a$, we know $1-\frac{d_a}{R_a}>1-\frac{d_a}{R_{a_1}}$. It is thus the lower bound of AL  with ellipsoid is tighter than AL  with sphere. Then, Proposition 1 holds.
\end{proof}

\subsection*{A.7 Proof of Proposition 2}
\begin{proof}
%It is clearly that $t_k>\lambda_k$. Then, $t_k>\lambda_{k_i}$, i=1,2,3. 
Following the proofs of Theorem 3:
\begin{equation}
\begin{split}
            {\rm Pr}[err(h,k)=0]_{{\rm Sphere}}&=1-\frac{t_k^3}{R_a^3}\\
                                          &=   1-\frac{(R_a-d_a)^3}{R_a^3}  \\
                                           &=1- \Bigg(1-\frac{d_a}{R_a}\Bigg)^3.\\
\end{split} 
\end{equation}
Following the proofs Theorem 4:
 \begin{equation}
\begin{split}
            {\rm Pr}[err(h,k)=0]_{\rm Ellipsoid}&=1-\frac{\lambda_{k_1}\lambda_{k_2}\lambda_{k_3}}{R_{a_1}^3}\\
                                          &=   1-\frac{(R_{a_1}-d_a)(R_{a_2}-d_a)(R_{a_3}-d_a)}{R_{a_1}^3}  \\
                                           &>1- \Bigg(1-\frac{d_a}{R_{a_1}}\Bigg)^3.\\
\end{split} 
\end{equation}

Based on Proposition 1,  $1-\frac{d_a}{R_a}>1-\frac{d_a}{R_{a_1}}$, therefore  ${\rm Pr}[err(h,k)=0]_{{\rm Sphere}}<   {\rm Pr}[err(h,k)=0]_{\rm Ellipsoid}$. Then, Proposition 2 is as stated.
\end{proof}

\subsection*{A.8 Proof of Theorems~3 and 4}
 Proof of Theorems~3.
\begin{proof}

Given $S_a$ over class $A$ is defined with ${\bf x_1^2+x_2^2+x_3^2}=R_a^2$. Let ${\bf x_1^2+x_2^2}=r_a^2$ be its 2-D generalization of $S_a$, 
assume that $x_2$ be a variable parameter in this 2-D generalization formula, the     ``volume'' (2-D volume is the area of the geometry object) of it   can be expressed as 
   \begin{equation}
\begin{split}
 {\rm  Vol}_{\bm 2}(S_a)=\int_{-r_a}^{r_a} 2\sqrt{r_a^2-{\bf x_2^2}} d {\bf x_2}.
 \end{split}
\end{equation}
Let $\vartheta$ be an angle variable that satisfies  $x_2=r_a  sin(\vartheta)$, we know $d{\bf x_2}=r_a cos(\vartheta)d \vartheta$. Then, Eq.~(65) is rewritten  as
\begin{equation}
\begin{split}
  {\rm  Vol}_{\bm 2}(S_a)&=\int_{-\pi/2}^{\pi/2} 2r_a^2cos^2(\vartheta)  d\vartheta\\
                          &= \int_{0}^{\pi/2} 4r_a^2cos^2(\vartheta)  d\vartheta.                 \\
 \end{split}
\end{equation}

For a $3$-D geometry, for the variable ${\bf x_3}$, it is over a cross-section which is a ${\bm 2}$-dimensional ball (circle), where the radius of the ball can be expressed as $r_acos(\vartheta)$, s.t. $\vartheta \in [0,\pi]$. Particularly, let  ${\rm  Vol}_{\bm 3}(S_a)$ be the volume of $S_a$ with  $3$ dimensions, the volume of this 3-dimensional sphere  then can be written  as
\begin{equation}
\begin{split}
  {\rm  Vol}_{\bm 3}(S_a)&=\int_{0}^{\pi/2} 2     {\rm Vol}_{\bm 2}(r_acos(\vartheta)) r_a(cos(\vartheta))  d\vartheta.\\
 \end{split}
\end{equation}

With Eq.~(67), volume of a ${\bm d}$-dimensional geometry  can be expressed as the integral over the $({\bf d}\text{-}1)$-dimensional cross-section of $S_a$
\begin{equation}
\begin{split}
  {\rm  Vol}_{\bm m}(S_a)&=\int_{0}^{\pi/2} 2     {\rm Vol}_{\bm m\text{-}1}(r_acos(\vartheta)) r_a(cos(\vartheta))  d\vartheta,\\
 \end{split}
\end{equation}
where  ${\rm Vol}_{\bm {m\text{-}1}}$ denotes the volume of ${\bm{ (m\text{-}1)}}$-dimensional  generalization geometry of $S_a$. 

 Based on Eq.~(68), we know ${\rm  Vol}_{\bm 3}$ can be written as 
\begin{equation}
\begin{split}
  {\rm  Vol}_{\bm 3}(S_a)&=\int_{0}^{\pi/2} 2     {\rm  Vol}_{\bm{2}}(R_acos(\vartheta)) r_a(cos(\vartheta'))  d\vartheta
 \end{split}
\end{equation}
Introducing Eq.~(66) into Eq.~(69), we have
\begin{equation}
\begin{split}
  {\rm  Vol}_{\bm 3}(S_a)&=\int_{0}^{\pi/2} 2     {\rm  Vol}_{\bm{2}}(R_acos(\vartheta')) r_a(cos(\vartheta'))  d\vartheta'\\
                &=\int_{0}^{\pi/2}  \int_{0}^{\pi/2}    
                8(r_acos(\vartheta')^2cos^2(\vartheta')
                 r_a(cos(\vartheta))  d\vartheta'    d\vartheta\\      &=\frac{4}{3}\pi R_a^2 \ \ \ \ \rm{s.t.} R_a=r_a.            \\
 \end{split}
\end{equation}

 Therefore, the generalization analysis results  of the 3-D geometry still can hold in the high dimensional geometry. Then,  
\end{proof}

 Proof of Theorems~4.
\begin{proof}
The integral of Eq.~(67) also can be adopted into the volume of ${\rm Vol}_{\bm 3}(E_a)$ by transforming the area i.e. ${\rm Vol}_{\bm 2}(S_a)$ into  ${\rm Vol}_{\bm 2}(E_a)$. Then, Eq.~(68) follows this transform.  
\end{proof}
 
\begin{IEEEbiography}
[{\includegraphics[width=1in,height=1.25in,clip,keepaspectratio]{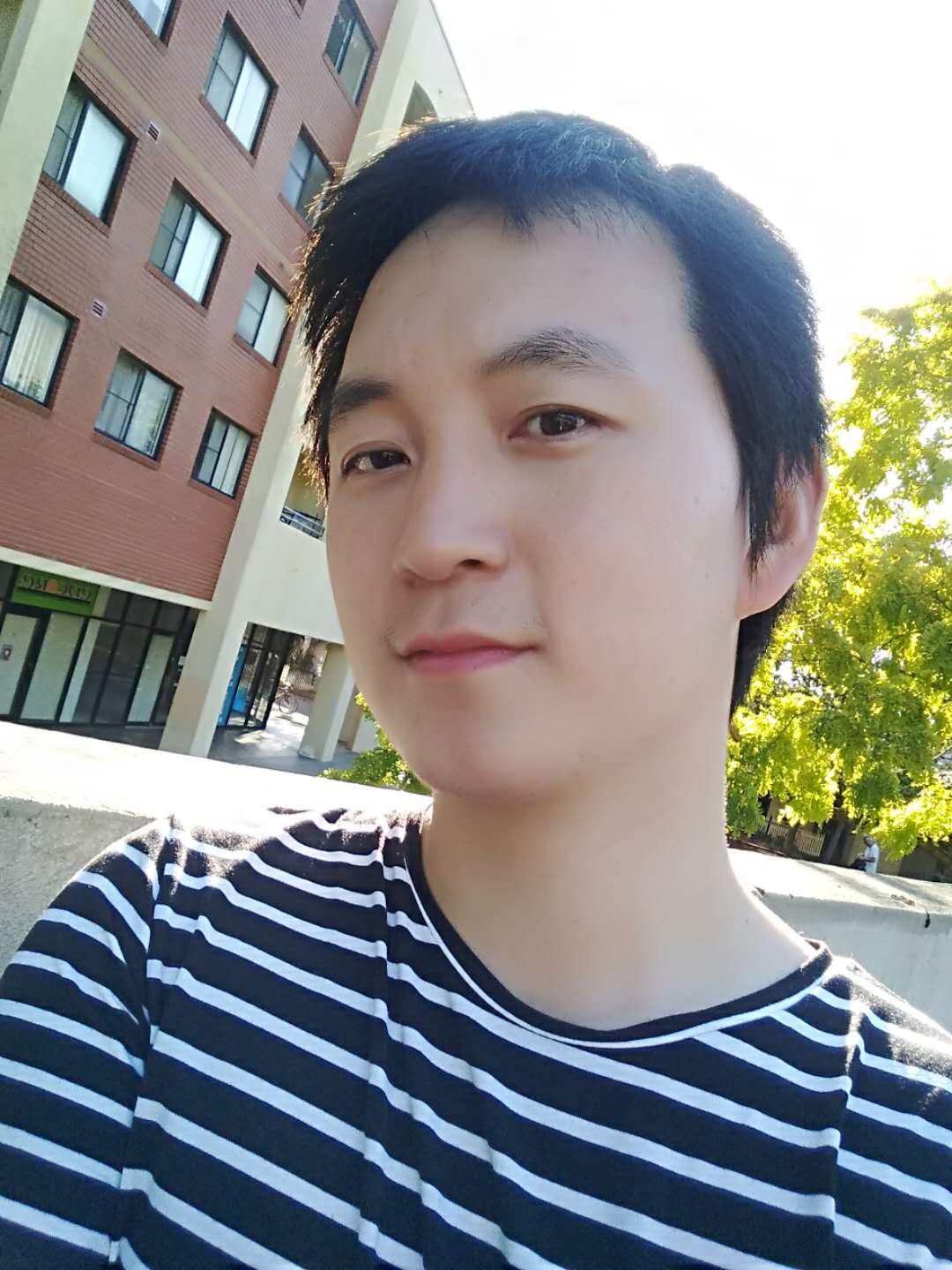}}]{Xiaofeng Cao} completed his PhD study at Australian Artificial Intelligence Institute (AAII), University of Technology Sydney. He is working as a Research Assistant at AAII.  His  research interests include PAC learning theory, agnostic learning algorithm,  generalization analysis, and hyperbolic geometry. 
\end{IEEEbiography}
%\vspace*{-9.6cm}
 
\begin{IEEEbiography}
[{\includegraphics[width=1in,height=1.25in,clip,keepaspectratio]{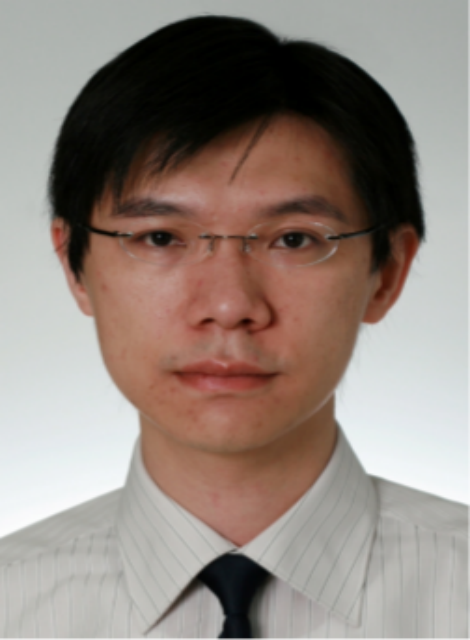}}]{Ivor W. Tsang} is Professor of Artificial Intelligence, at University of Technology Sydney. He is also the Research Director of the Australian Artificial Intelligence Institute.  In 2019, his paper titled ``Towards ultrahigh dimensional feature selection for big data" received the International Consortium of Chinese Mathematicians Best Paper Award. In 2020, Prof Tsang was recognized as the AI 2000 AAAI/IJCAI Most Influential Scholar in Australia for his outstanding contributions to the field of Artificial Intelligence between 2009 and 2019. His works on transfer learning granted him the Best Student Paper Award at International Conference on Computer Vision and Pattern Recognition 2010 and the 2014 IEEE Transactions on Multimedia Prize Paper Award. In addition, he had received the prestigious IEEE Transactions on Neural Networks Outstanding 2004 Paper Award in 2007.  
\par Prof. Tsang serves as a Senior Area Chair for Neural Information Processing Systems and Area Chair for International Conference on Machine Learning, and the Editorial Board for  Journal Machine Learning Research, Machine Learning, 
Journal of Artificial Intelligence Research, and IEEE Transactions on Pattern Analysis and Machine Intelligence. 
\end{IEEEbiography}

\end{document}